\def\isarxiv{1} 

\ifdefined\isarxiv
\documentclass[11pt]{article}

\usepackage[numbers]{natbib}
\usepackage{subfig}
\usepackage{graphicx}

\else
\documentclass[twoside]{article}

\usepackage[accepted]{aistats2025}
\usepackage[round]{natbib}

\usepackage{subfig}
\fi

\usepackage{amsmath}
\usepackage{amsthm}
\usepackage{amssymb}
\usepackage{algorithm}
\usepackage{algpseudocode}
\usepackage{grffile}
\usepackage{wrapfig,epsfig}
\usepackage{url}
\usepackage{xcolor}
\usepackage{epstopdf}
\usepackage{bbm}
\usepackage{dsfont}

\allowdisplaybreaks

\ifdefined\isarxiv

\usepackage{tikz}
\usepackage{hyperref}  
\hypersetup{colorlinks=true,citecolor=blue,linkcolor=blue} 
\usetikzlibrary{arrows}
\usepackage[margin=1in]{geometry}

\else
\usepackage{hyperref}       

\fi
\graphicspath{{./figs/}}

\theoremstyle{plain}
\newtheorem{theorem}{Theorem}[section]
\newtheorem{lemma}[theorem]{Lemma}
\newtheorem{definition}[theorem]{Definition}

\newtheorem{condition}[theorem]{Condition}

\newcommand{\wh}{\widehat}
\newcommand{\wt}{\widetilde}
\newcommand{\ov}{\overline}

\newcommand{\R}{\mathbb{R}}

\renewcommand{\i}{\mathbf{i}}
\renewcommand{\tilde}{\wt}
\renewcommand{\hat}{\wh}

\DeclareMathOperator*{\E}{{\mathbb{E}}}

\DeclareMathOperator{\softmax}{softmax}

\makeatletter
\newcommand*{\RN}[1]{\expandafter\@slowromancap\romannumeral #1@}
\makeatother

\usepackage{lineno}

\begin{document}

\ifdefined\isarxiv

\date{}

\title{Fourier Circuits in Neural Networks and Transformers: A Case Study of Modular Arithmetic with Multiple Inputs
}
\author{ 
Chenyang Li\thanks{\texttt{
lchenyang550@gmail.com}. Fuzhou University.}
\and
Yingyu Liang\thanks{\texttt{
yingyul@hku.hk}. The University of Hong Kong. \texttt{
yliang@cs.wisc.edu}. University of Wisconsin-Madison.} 
\and
Zhenmei Shi\thanks{\texttt{
zhmeishi@cs.wisc.edu}. University of Wisconsin-Madison.}
\and 
Zhao Song\thanks{\texttt{ magic.linuxkde@gmail.com}. The Simons Institute for the Theory of Computing at the University of California, Berkeley.}
\and 
Tianyi Zhou\thanks{\texttt{
tzhou029@usc.edu}. University of Southern California.}
}

\else
%
\runningtitle{Fourier Circuits in Neural Networks and Transformers}

%

\twocolumn[

\aistatstitle{Fourier Circuits in Neural Networks and Transformers: A Case Study of Modular Arithmetic with Multiple Inputs}

\aistatsauthor{ 
Chenyang Li$^{1}$
\And 
Yingyu Liang$^{2,3}$
\And  
Zhenmei Shi$^{3}$
\And 
Zhao Song$^{4}$
\And 
Tianyi Zhou$^{5}$
} 

\aistatsaddress{ 
$^1$Fuzhou University. \qquad
$^2$The University of Hong Kong. \qquad
$^3$University of Wisconsin-Madison. 
\qquad
\\
$^4$The Simons Institute for the Theory of Computing at the UC, Berkeley. \qquad
$^5$University of Southern California.
} 
]

\fi

\ifdefined\isarxiv
\begin{titlepage}
  \maketitle
  \begin{abstract}
In the evolving landscape of machine learning, a pivotal challenge lies in deciphering the internal representations harnessed by neural networks and Transformers. Building on recent progress toward comprehending how networks execute distinct target functions, our study embarks on an exploration of the underlying reasons behind networks adopting specific computational strategies. We direct our focus to the complex algebraic learning task of modular addition involving $k$ inputs. Our research presents a thorough analytical characterization of the features learned by stylized one-hidden layer neural networks and one-layer Transformers in addressing this task. A cornerstone of our theoretical framework is the elucidation of how the principle of margin maximization shapes the features adopted by one-hidden layer neural networks. Let $p$ denote the modulus, $D_p$ denote the dataset of modular arithmetic with $k$ inputs and $m$ denote the network width. We demonstrate that a neuron count of $ m \geq 2^{2k-2} \cdot (p-1) $, these networks attain a maximum $ L_{2,k+1} $-margin on the dataset $ D_p $. Furthermore, we establish that each hidden-layer neuron aligns with a specific Fourier spectrum, integral to solving modular addition problems. By correlating our findings with the empirical observations of similar studies, we contribute to a deeper comprehension of the intrinsic computational mechanisms of neural networks. Furthermore, we observe similar computational mechanisms in attention matrices of one-layer Transformers. Our work stands as a significant stride in unraveling their operation complexities, particularly in the realm of complex algebraic tasks.

  \end{abstract}
  \thispagestyle{empty}
\end{titlepage}

{\hypersetup{linkcolor=black}
\tableofcontents
}
\newpage

\else

\begin{abstract}
In the evolving landscape of machine learning, a pivotal challenge lies in deciphering the internal representations harnessed by neural networks and Transformers. Building on recent progress toward comprehending how networks execute distinct target functions, our study embarks on an exploration of the underlying reasons behind networks adopting specific computational strategies. We direct our focus to the complex algebraic learning task of modular addition involving $k$ inputs. Our research presents a thorough analytical characterization of the features learned by stylized one-hidden layer neural networks and one-layer Transformers in addressing this task. A cornerstone of our theoretical framework is the elucidation of how the principle of margin maximization shapes the features adopted by one-hidden layer neural networks. Let $p$ denote the modulus, $D_p$ denote the dataset of modular arithmetic with $k$ inputs and $m$ denote the network width. We demonstrate that a neuron count of $ m \geq 2^{2k-2} \cdot (p-1) $, these networks attain a maximum $ L_{2,k+1} $-margin on the dataset $ D_p $. Furthermore, we establish that each hidden-layer neuron aligns with a specific Fourier spectrum, integral to solving modular addition problems. By correlating our findings with the empirical observations of similar studies, we contribute to a deeper comprehension of the intrinsic computational mechanisms of neural networks. Furthermore, we observe similar computational mechanisms in attention matrices of one-layer Transformers. Our work stands as a significant stride in unraveling their operation complexities, particularly in the realm of complex algebraic tasks.
\end{abstract}

\fi

\section{INTRODUCTION}\label{sec:intro}
The field of artificial intelligence has experienced a significant transformation with the development of large language models (LLMs), particularly through the introduction of the Transformer architecture~\citep{vsp+17}. This advancement has revolutionized approaches to challenging tasks in natural language processing, notably in machine translation \citep{pcr19, ghg+20} and text generation \citep{lsx+22}. Consequently, models 
e.g.,  Mistral~\citep{mistral}, Llama~\citep{llama3}, Gemini~\citep{gemini}, Gemma~\citep{gemma},  Claude3~\citep{claude}, GPT4~\citep{openai23} and so on, have become predominant in NLP.

Central to this study is the question of how these advanced models transcend mere pattern recognition to engage in what appears to be logical reasoning and problem-solving. This inquiry is not purely academic; it probes the core of ``understanding'' in artificial intelligence. While LLMs, such as Claude3 and GPT4, demonstrate remarkable proficiency in human-like text generation, their capability to comprehend and process mathematical logic is a topic of considerable debate.
This line of investigation is crucial, given AI's potential to extend beyond text generation into deeper comprehension of complex subjects. Mathematics, often seen as the universal language, presents a uniquely challenging domain for these models \citep{yc23}. 
Our research aims to determine whether Transformers with attention, noted for their NLP efficiency, can also demonstrate an intrinsic understanding of mathematical operations and reasoning.

In a recent surprising study of mathematical operations learning, {\cite{pbe+22}} train Transformers on small algorithmic datasets, e.g., $a_1+a_2 \mod p$ and we let $p$ be a prime number,
and show the ``grokking'' phenomenon, where models abruptly transition from bad generalization to perfect generalization after a large number of training steps.  
Nascent studies, such as those by {\cite{ncl+23}}, empirically reveal that Transformers can solve modular addition using Fourier-based circuits. 
They found that the Transformers trained by Stochastic Gradient Descent (SGD) not only reliably compute $a_1+a_2 \mod p$, but also that the networks consistently employ a specific geometric algorithm. This algorithm, which involves composing integer rotations around a circle, indicates an inherent comprehension of modular arithmetic within the network's architecture. The algorithm relies on this identity: for any $a_1, a_2$ and $\zeta \in \mathbb{Z}_p \backslash \{ 0\}$, the following two quantities are equivalent
\begin{align*}
    (a_1 + a_2) \bmod p = \underset{c \in \mathbb{Z}_p}{\arg \max } \{\cos ({2 \pi \zeta( a_1 + a_2 -c)}/{p} ) \}.
\end{align*}
{\cite{ncl+23}} further 
show that the attention and MLP module in the Transformer imbues the neurons with Fourier circuit-like properties. 
To study why networks arrive at Fourier-based circuits computational strategies, {\cite{meo+23}} theoretically study one-hidden layer neural network learning on two inputs modular addition task and certify that the trained networks will execute modular addition by employing Fourier features aligning closely with the previous empirical observations.
However, the question remains whether neural networks can solve more complicated mathematical problems.

Inspired by recent developments in mechanistic interpretability \citep{ocs+20, eno+21, eho22} and the study of inductive biases \citep{shn18, v23} in neural networks, we extend our research to modular addition with more ($k$) inputs.
\begin{align}\label{eq:data}
    (a_1+a_2+\dots+a_k) \bmod p.
\end{align}
This approach offers insights into why certain representations and solutions emerge from neural network training. By integrating these insights with our empirical findings, we aim to provide a comprehensive understanding of neural networks' learning mechanisms, especially in solving the modular addition problem. We also determine the necessary number of neurons for the network to learn this Fourier method for modular addition. 
Our paper's contributions are summarized as follows:
\begin{itemize}
\item \textbf{Expansion of Input for Modular Addition Problem:} We extend the input parameter range for the modular addition problem from a binary set to $k$-element sets.
\item \textbf{Network's Maximum Margin:} 
For $p$-modular addition of $k$ inputs, we give the closed form of the maximum margin of a network (Lemma~\ref{lem:margin_soln-k:informal}):
    \begin{align*}
        \gamma^*=\frac{2(k!)}{(2k+2)^{(k+1)/2}(p-1) p^{(k-1)/2}}.
    \end{align*} 
\item \textbf{Neuron Count in One-Hidden-Layer Networks:} We propose that in a general case, a one-hidden-layer network having $m \geq 2^{2k-2} \cdot (p-1)$ neurons can achieve the maximum $L_{2,k+1}$-margin solution, each hidden neuron aligning with a specific Fourier spectrum.
This ensures the network's capability to effectively solve the modular addition in a Fourier-based method (Theorem~\ref{thm:main_k:informal}). 
\item \textbf{Empirical Validation of Theoretical Findings:} We validate our theoretical finding that: when $m \geq 2^{2k-2} \cdot (p-1)$, for each spectrum $\zeta \in\{1, \ldots, \frac{p-1}{2}\}$, there exists a hidden-neuron utilizes this spectrum. It strongly supports our analysis. (Figure~\ref{fig:nn_w_k4} and Figure~\ref{fig:nn_freq_k4}).
\item \textbf{Similar Findings in Transformer:} We have a similar observation in one-layer Transformer learning modular addition involving $k$ inputs. For the $2$-dimensional matrix $W_K W_Q$, where $W_K,W_Q$ denotes the key and query matrix, it shows the superposition of two cosine waveforms in each dimension, each characterized by distinct frequencies (Figure~\ref{fig:s_k4}).
\item \textbf{Grokking under Different $k$}: We observe that as $k$ increases, the grokking phenomenon becomes weaker, as predicted by our analysis (Figure~\ref{fig:grok}).
\end{itemize}

{\bf Detailed comparison with \cite{meo+23}.} 
Theoretically, we generalize beyond the results of \cite{meo+23} to $k$ inputs. There are unique technique challenges for our setting and not presented in previous settings: (1) While constructing a general $k$ version of the max-margin solution, we need our unique sum-to-product Identities for $k$ inputs, which was proved by our Lemma~\ref{lem:sum2product}; (2) To calculate the number of neurons, we need our unique Lemma \ref{lem:construct-k:informal} and Eq~\eqref{eq:k_cos_sum} to handle cosine operation on $k$ inputs; (3) We need our Lemma~\ref{lemma:fourier_space-k} and Lemma~\ref{lem:margin_soln-k} to handle multiple variable Fourier transform, where we also introduce multiple inequalities for $k$ inputs version. Empirically, our experiments verified that the theoretical insight obtained can be carried over to practical transformers. Furthermore, the study of grokking is beyond \cite{meo+23}, as the study of grokking is only possible when there are $k \ge 2$. Our general $k$ version is necessary for studying some key properties of network learning, like grokking over tasks of increasing complexity or generalization ability over tasks of increasing complexity. 
\section{RELATED WORK}

{\bf Max Margin Solutions in Neural Networks.} {\cite{bbg22}} demonstrated that neurons in a one-hidden-layer ReLU network align with clauses in max margin solutions for read-once DNFs, employing a unique proof technique involving the construction of perturbed networks. {\cite{meo+23}} utilize max-min duality to certify maximum-margin solutions. Further, extensive research in the domain of margin maximization in neural networks, including works by {\cite{glss18,shn18,gunasekar2018characterizing,wllm19,ll19,ji2019implicit,moroshko2020implicit,chizat2020implicit,jt20,lyu2021gradient,frei2022implicit,frei2023benign,smf+24,llss24b}} and more, has highlighted the implicit bias towards margin maximization inherent in neural network optimization. They provide a foundational understanding of the dynamics of neural networks and their inclination towards maximizing margins under various conditions and architectures.

{\bf Algebraic Tasks Learning Mechanism Interpretability.}
 The study of neural networks trained on algebraic tasks has been pivotal in shedding light on their training dynamics and inductive biases. Notable contributions include the work of {\cite{pbe+22,g23,qb23}} on modular addition and subsequent follow-up studies, investigations into learning parities \citep{dm20,beg+22,swl22,swxl24,swl23,scl+23,ztb+23,xsl24}, and research into algorithmic reasoning capabilities \citep{sgh+19,hbk+21,lad+22,mbab22,dls22,ccn23,syfb23,nlw23,zlta23,tho+23,hlv23}. The field of mechanistic interpretability, focusing on the analysis of internal representations in neural networks, has also seen significant advancements through the works of {\cite{cgc+20,oen+22,mts23,rsr23,vsk+23,dhdg24,smn+24,kll+24,cll+24,lss+24_relu,sawl24,kls+25,cll+25_icl,lss+25_relu,lll+25_loop}} and others. 

{\bf Grokking and Emergent Ability.}
The phenomenon known as ``grokking'' was initially identified by {\cite{pbe+22}} and is believed to be a way of studying the emerging abilities of LLM~\citep{wei2022emergent}. This research observed a unique trend in two-layer transformer models engaged in algorithmic tasks, where there was a significant increase in test accuracy, surprisingly occurring well after these models had reached perfect accuracy in their training phase. In {\cite{m22}}, it was hypothesized that this might be the result of the SGD process that resembles a random path along what is termed the optimal manifold. Adding to this, {\cite{ncl+23}} aligns with the findings of {\cite{b22}}, indicating a steady advancement of networks towards algorithms that are better at generalization. {\cite{lkn+22,xu2023benign,lyu2023dichotomy}} developed smaller-scale examples of grokking and utilized these to map out phase diagrams, delineating multiple distinct learning stages. Furthermore, {\cite{tlz+22,murty2023grokking}} suggested the possibility of grokking occurring naturally, even in the absence of explicit regularization. They attributed this to an optimization quirk they termed the slingshot mechanism, which might inadvertently act as a regularizing factor.

{\bf Theoretical Work About Fourier Transform.}
To calculate Fourier transform there are two main methodologies: one uses carefully chosen samples through hashing functions (referenced in works like \cite{ikp14,ik14,k16,k17}) to achieve sublinear sample complexity and running time, while the other uses random samples (as discussed in \cite{bou14,hr16,nsw19}) with sublinear sample complexity but nearly linear running time. There are many other works studying Fourier transform \citep{s19,jls23,gss22,lsz19,cls20,sswz22,ckps16,sswz23,css+23,syyz23,lls+24}.
\section{PROBLEM SETUP}\label{sec:problem}

\subsection{Data and Network Setup}\label{sec:problem:network_setup}

{\bf Data.} Following {\cite{meo+23}}, let $\mathbb{Z}_p = [p]$ denote the modular group on $p$ integers, where $p > 2$ is a given prime number.  
The input space is $\mathcal{X} := \mathbb{Z}_p^k$ for some integer $k$, and the output space is $\mathcal{Y} := \mathbb{Z}_p$. Then an input data point is $a = (a_1, \dots, a_k )$ with $a_i\in \mathbb{Z}_p$. When clear from context, we also let $x_i \in \{0,1\}^p$ be the one-hot encoding of $a_i$, and let $x = (x_1, \dots, x_k)$ denote the input point.

{\bf Network.} We consider single-hidden layer neural networks with polynomial activation functions: 
\begin{align}  \label{eq:nn}
f(\theta, x) :=& ~ \sum_{i=1}^m \phi(\theta_i, x), \\
 \phi(\theta_i, x) := & ~ (u_{i,1}^{\top} x_1 + \dots + u_{i,k}^{\top} x_k )^k w_i, \notag
\end{align}
where $\theta := \{\theta_1, \ldots, \theta_m \} \in \R^{(k+1) \times p}$, ~$\phi(\theta_i, x)$ is one neuron, and $\theta_i := \{u_{i,1}, \dots, u_{i,k}, w_i\}$ are the parameters of the neuron with $u_{i,1}, \dots, u_{i,k}, w_i \in \mathbb{R}^p$. We use polynomial activation functions due to the homogeneous requirement in Lemma~\ref{lem:homo} and easy sum-to-product identities calculation in Fourier analysis.
Using the notation $a$ instead of the one-hot encodings $x$, we can also write:
\begin{align*}
f(\theta, a)  :=& ~ \sum_{i=1}^m \phi(\theta_i, a),\\ \phi(\theta_i, a) := & ~ (u_{i,1}(a_1) + \dots + u_{i,k}(a_k) )^k w_i,
\end{align*}
where with $u_{i,j}(a_j)$ being the $a_j$-th component of $u_{i,j}$. We consider the parameter set:  
\begin{align*}
 \Theta:= &  \{\|\theta\|_{2,k+1}  \leq 1\}, \\
\text{~where~}  \|\theta\|_{2,k+1} := & (\sum_{i=1}^m \|\theta_i\|_2^{k+1})^\frac{1}{k+1}, \\ \|\theta_i\|_2 := & (\sum_{j=1}^k \|u_{i,j}\|^2_2 + \|w_i\|_2^2)^\frac{1}{2}.
\end{align*}

Here $\|\theta\|_{2,k+1} $ is the $L_{2,k+1}$ matrix norm of $\theta$ (Definition~\ref{def:norm}), and $\|\theta_i\|_2$ is the $L_2$ vector norm of the concatenated vector of the parameters in $\theta_i$. The training objective over $\Theta$ is then as follows.
\begin{definition}\label{def:reg_obj}
Given a dataset $D_p$ and the cross-entropy loss $l$, the regularized training objective is: 
\begin{align*}
\mathcal{L}_\lambda(\theta):=\frac{1}{|D_p|} \sum_{(x, y) \in D_p} l(f(\theta, x), y)+\lambda\|\theta\|_{2,k+1}.
\end{align*}
\end{definition}

\subsection{Margins of the Neural Networks}\label{sec:problem:definition}
Now, we define the margin for a data point and the margin for a whole dataset. 

\begin{definition}\label{def:g}
We denote  $g: \mathbb{R}^U \times \mathcal{X} \times \mathcal{Y} \rightarrow \mathbb{R}$ as the margin function, where for given  $(x, y) \in D_p$,
\begin{align*}
g(\theta, x, y):=f(\theta, x)[y]-\max _{y^{\prime} \in \mathcal{Y} \backslash\{y\}} f(\theta, x)[y^{\prime}] .
\end{align*}
\end{definition}

\begin{definition}\label{def:h}
The margin for a given dataset $D_p$ is denoted as $h: \mathbb{R}^U \rightarrow \mathbb{R}$ where
\begin{align*}
h(\theta):=\min _{(x, y) \in D_p} g(\theta, x, y).
\end{align*}
\end{definition}

For parameter $\theta$, its normalized margin  is denoted as $h(\theta /\|\theta\|_{2,k+1})$.
For simplicity, we define $\gamma^*$ to be the maximum normalized margin as the following:
\begin{definition}\label{def:gamma}
The minimum of the regularized objective is denoted as $\theta_\lambda \in \arg \min _{\theta \in \mathbb{R}^U} \mathcal{L}_\lambda(\theta)$. We define the normalized margin of $\theta_\lambda$ as $\gamma_\lambda:=h (\theta_\lambda /\|\theta_\lambda\|_{2,k+1} )$ and the maximum normalized margin as $\gamma^*:=\max _{\theta \in \Theta} h(\theta)$, where $\Theta =\{\|\theta\|_{2,k+1} \leq 1\}$. 
\end{definition}
Let $\mathcal{P}(D_p)$ denote a set containing all distributions over $D_p$. 
Then $\gamma^*$ can be rewritten as
\begin{align}
    \gamma^* 
    = & ~ \max_{\theta \in \Theta} h( \theta)  = ~ \max _{\theta \in \Theta} \min _{(x, y) \in D_p} g(\theta, x, y) \notag \\
    = & ~  \max _{\theta \in \Theta} \min _{q \in \mathcal{P}(D_p)} \underset{(x, y) \sim q}{\mathbb{E}}[g(\theta, x, y)], 
\end{align}
where the first step is from Definition~\ref{def:gamma}, the second step is from Definition~\ref{def:h}, and the last step is from the linearity of the expectation.
Now, we introduce an important concept of a duality stationary pair $(\theta^*, q^* )$. 

\begin{definition}\label{def:q_star_theta_star}
We define a stationary pair $(\theta^*, q^* )$ when satisfying
 \begin{align}
        q^* & \in \underset{q \in \mathcal{P}(D_p)}{\arg \min } \underset{(x, y) \sim q}{\mathbb{E}}[g(\theta^*, x, y)] \label{eq:definition_q_star}, \\
        \theta^* & \in \underset{\theta \in \Theta}{\arg \min } \underset{(x, y) \sim q^*}{\mathbb{E}}[g(\theta, x, y)]. \notag
    \end{align} 
\end{definition}

This means that $q^*$ is a distribution that minimizes the expected margin based on $\theta^*$, and simultaneously, $\theta^*$ is a solution that maximizes the expected margin relative to $q^*$. 
The max-min inequality~\citep{boyd2004convex} indicates that presenting such a duality adequately proves $\theta^*$ to be a maximum margin solution.
Recall that there is a ``max'' operation in Definition~\ref{def:g}, which makes the swapping of expectation and summation infeasible, meaning that the expected network margin cannot be broken down into the expected margins of individual neurons.
To tackle this problem, the class-weighted margin is proposed, whose intuition is similar to label smoothing. 
Let $\tau: D_p \rightarrow \Delta(\mathcal{Y})$ allocate weights to incorrect labels for every data point. Given $(x, y)$ in $D_p$ and for any $y' \in \mathcal{Y}$, we have $\tau(x, y)[ y' ] \geq 0$ and $\sum_{y' \in \mathcal{Y} \backslash\{y\}} \tau(x, y)[y' ]=1$. We denote a proxy $g'$ as the following to solve the issue.

\begin{definition}\label{def:g_prime}
     Draw $(x, y) \in D_p$. The class-weighted margin  $g^{\prime}$ is defined as
\begin{align*}
   & ~ g^{\prime}(\theta, x, y) 
   := ~ f(\theta, x)[y] - \sum_{y^{\prime} \in \mathcal{Y} \backslash\{y\}} \tau(x, y)[y^{\prime}] f(\theta, x)[y^{\prime}]. 
\end{align*}
\end{definition}
We have $g^{\prime}$ uses a weighted sum rather than max, so $ g(\theta, x, y) \le g^{\prime}(\theta, x, y)$. Following the linearity of expectation, we get the expected class-weighted margin as
\begin{align*}
    \underset{(x, y) }{\E}[g^{\prime}(\theta, x, y)]  = & ~ \sum_{i=1}^m \underset{(x, y) }{\E}\Big [\phi(\theta_i, x)[y] \\
    & ~ -\sum_{y^{\prime} \in \mathcal{Y} \backslash\{y\}} \tau(x, y)[y^{\prime}] \phi(\theta_i, x)[y^{\prime}]\Big],
\end{align*}
where we can move the summation $\sum_{i=1}^m$ out of the expectation $\E []$. 

\subsection{Connection between Training and the Maximum Margin Solutions}\label{sec:problem:preliminary}

\begin{figure}[!ht]
    \centering
\includegraphics[width=1\linewidth]{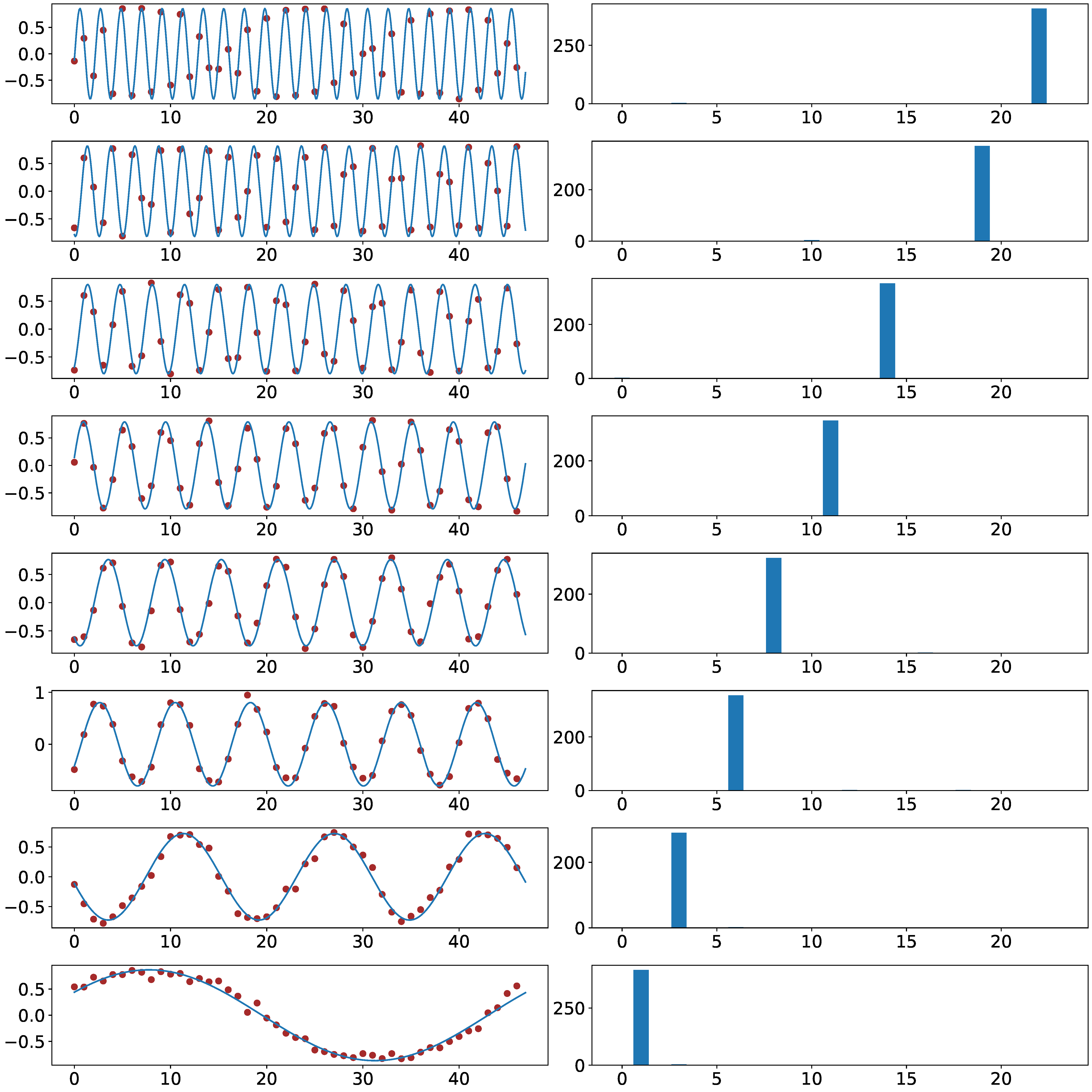}
    \caption{Cosine shape of the trained embeddings (hidden layer weights) and corresponding power of  Fourier spectrum. The two-layer network with $m=2944$ neurons is trained on $k=4$-sum mod-$p=47$ addition dataset. We even split the whole datasets ($p^k = 47^4$ data points) into the training and test datasets. Every row represents a random neuron from the network. The left figure shows the final trained embeddings, with red dots indicating the true weight values, and the pale blue interpolation is achieved by identifying the function that shares the same Fourier spectrum. The right figure shows their Fourier power spectrum. The results in these figures are consistent with our analysis statements in Lemma~\ref{lem:margin_soln-k:informal}. See Figure~\ref{fig:nn_w_k3},~\ref{fig:nn_w_k5} in Appendix~\ref{app:sec:exp_nn} for similar results when $k$ is 3 or 5.}
    \label{fig:nn_w_k4}
\end{figure}

We denote $\nu$ as the network's homogeneity constant, where the equation $f(\alpha \theta, x) = \alpha^\nu f(\theta, x)$ holds for any $x$ and any scalar $\alpha > 0$. Specifically, we focus on networks with homogeneous neurons that satisfy $\phi(\alpha \theta_i, x) = \alpha^\nu \phi(\theta_i, x)$ for any $\alpha > 0$. Note that our one-hidden layer networks (Eq.~\eqref{eq:nn}) are $k+1$ homogeneous.  As the following Lemma states, when $\lambda$ is small enough during training homogeneous functions, we have the $\mathcal{L}_\lambda$ global optimizers' normalized margin converges to $\gamma^*$. 

\begin{lemma}[\cite{wei2019regularization}, Theorem 4.1]\label{lem:homo} 
Let $f$ be a homogeneous function. For any norm $\| \cdot \|$, if $\gamma^* > 0$, we have $\lim_{\lambda\rightarrow 0} \gamma_\lambda = \gamma^* $.
\end{lemma}

Therefore, to comprehend the global minimize, we can explore the maximum-margin solution as a surrogate, enabling us to bypass complex analyses in non-convex optimization.
Furthermore, {\cite{meo+23}} states that under the following condition, the maximum-margin solutions and class-weighted maximum-margin ($g'$)  solutions are equivalent to each other. 
\begin{condition}[Condition C.1 in page 8 in \cite{meo+23}]\label{lem:translation_property}

    We have $g^{\prime} (\theta^*, x, y )=g (\theta^*, x, y )$ for all $(x, y) \in \operatorname{spt} (q^* )$, where $\operatorname{spt}$ is the support. It means: 
    \begin{align*}
        \{y' \in \mathcal{Y} \backslash\{y\}: \tau(x, y)[y']>0\} \subseteq \underset{y' \in \mathcal{Y} \backslash\{y\}}{\arg \max } f (\theta^*, x )[y'].
    \end{align*}
\end{condition}
Thus, under these conditions, we only need to focus on the class-weighted maximum-margin solutions in our following analysis.

\section{MAIN RESULT}

We characterize the Fourier features to perform modular addition with $k$ input in the one-hidden-layer neuron network.   
We show that every neuron only focuses on a distinct Fourier frequency.
Additionally, within the network, there is at least one neuron for each frequency. When we consider the uniform class weighting, where $\mathcal{L}_\lambda(\theta)$ is based on
\begin{align}
    \tau (a_1, \dots, a_k)[c'] := 1/(p - 1) ~~ \forall c' \neq a_1 + \dots + a_k,
\end{align}
we have the following main result:
\begin{theorem}[Main result, informal version of Theorem \ref{thm:main_k:formal}]\label{thm:main_k:informal} 
Let $f(\theta, x)$ be the one-hidden layer networks defined in Eq~\eqref{eq:nn}. 
If $m \geq 2^{2k-1} \cdot \frac{ p-1 }{ 2 }$, then the max $L_{2,k+1}$-margin network satisfies:
\begin{itemize}
    \item  The maximum $L_{2,k+1}$-margin for a dataset $D_p$ is:
    \begin{align*}
        \gamma^*=\frac{2(k!)}{(2k+2)^{(k+1)/2}(p-1) p^{(k-1)/2}}.
    \end{align*}
    \item For each neuron $\phi(\{u_1,\dots, u_k, w\} ; a_1, \dots, a_k)$, there is a constant scalar $\beta \in \R$ and a frequency $\zeta \in\{1, \ldots, \frac{p-1}{2}\}$ satisfying
    \begin{align*}
    u_i(a_i) = & ~ \beta \cdot \cos ( \theta_{u_i}^* + 2 \pi \zeta a_i /p ), ~~\forall i \in [k] \\
    w(c) = & ~ \beta \cdot \cos (\theta_w^*+2 \pi \zeta c / p),
    \end{align*}
    where $\theta_{u_1}^*, \dots, \theta_{u_k}^*, \theta_w^* \in \R$ are some phase offsets satisfying $\theta_{u_1}^*+\dots+\theta_{u_k}^*=\theta_w^*$.
    \item For each frequency $\zeta \in\{1, \ldots, \frac{p-1}{2}\}$, there exists one neuron using this frequency only.
\end{itemize}
\end{theorem}
\begin{proof}[Proof sketch of Theorem~\ref{thm:main_k:informal}] See formal proof in Appendix~\ref{app:main:main_result_k}.
By Lemma~\ref{lem:margin_soln-k:informal}, we get $\gamma^*$ and the single-neuron class-weighted maximum-margin solution set $\Omega_q^{'*}$. By satisfying Condition \ref{lem:translation_property}, we know it is used in the maximum-margin solution.  
By Lemma~\ref{lem:construct-k:informal}, we can construct the network $\theta^*$ that uses neurons in $\Omega_q^{'*}$. By Lemma~\ref{lemma:multi-class}, we know that it is the maximum-margin solution.  Finally, by Lemma~\ref{lem:frequency-k}, we know that all frequencies are covered. 
\end{proof}
Theorem~\ref{thm:main_k:informal} tells us when the number of neurons is large enough, e.g., $m \geq 2^{2k-1} \cdot \frac{ p-1 }{ 2 }$ (the lower bound of $m$ may not be the tightest in our analysis), the one hidden neural network will exactly learn all Fourier spectrum/basis to recover the modular addition operation. 
More specifically, each neuron will only focus on one Fourier frequency. 
Our analysis provides a comprehensive understanding of why neural networks trained by SGD prefer to learn Fourier-based circuits.

Our analysis essentially provides hints on how neural networks learn to perform well. Note that humans do modular calculations completely differently with Fourier circuits. Thus, for more general tasks, the model behavior may differ from that of human beings. On the other hand, the Fourier spectrum feature pattern could be useful for out-of-distribution (OOD), robust learning, or designing better learning algorithms, e.g., making the implicit regularization explicit.

\ifdefined\isarxiv
\newcommand{\linewithvalue}{0.34}
\else
\newcommand{\linewithvalue}{0.49}
\fi

\begin{figure*}[!ht]
    \centering
\subfloat[]{\includegraphics[width=0.50\linewidth]{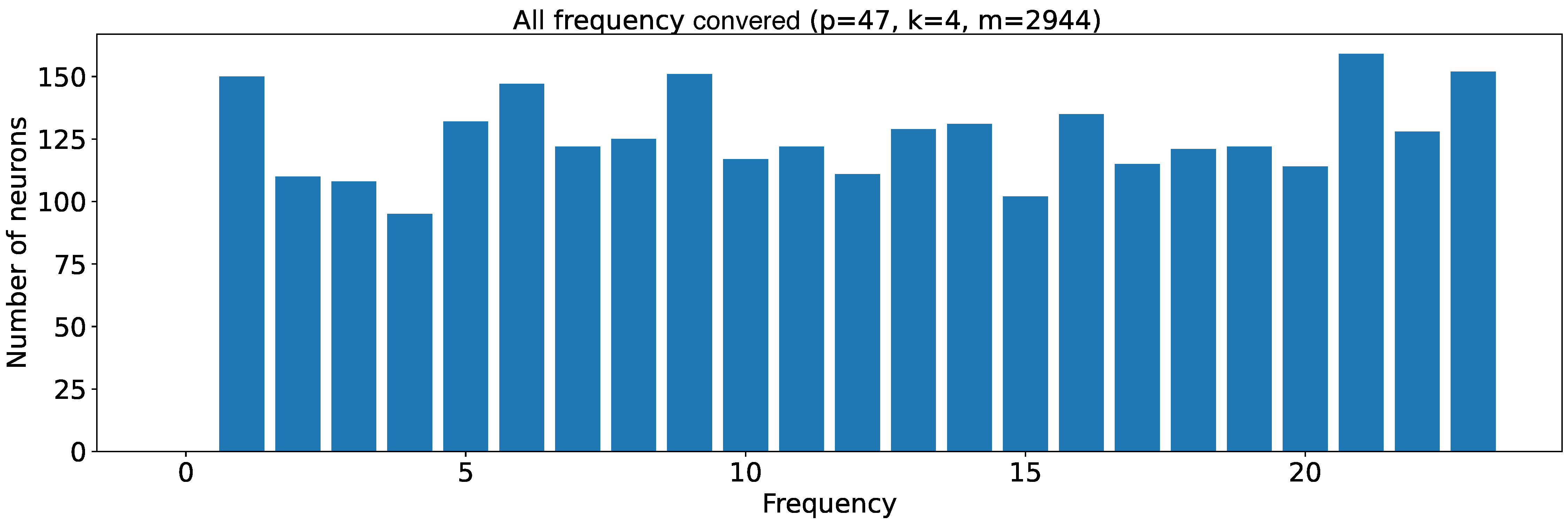}}
\subfloat[]{\includegraphics[width=0.24\linewidth]{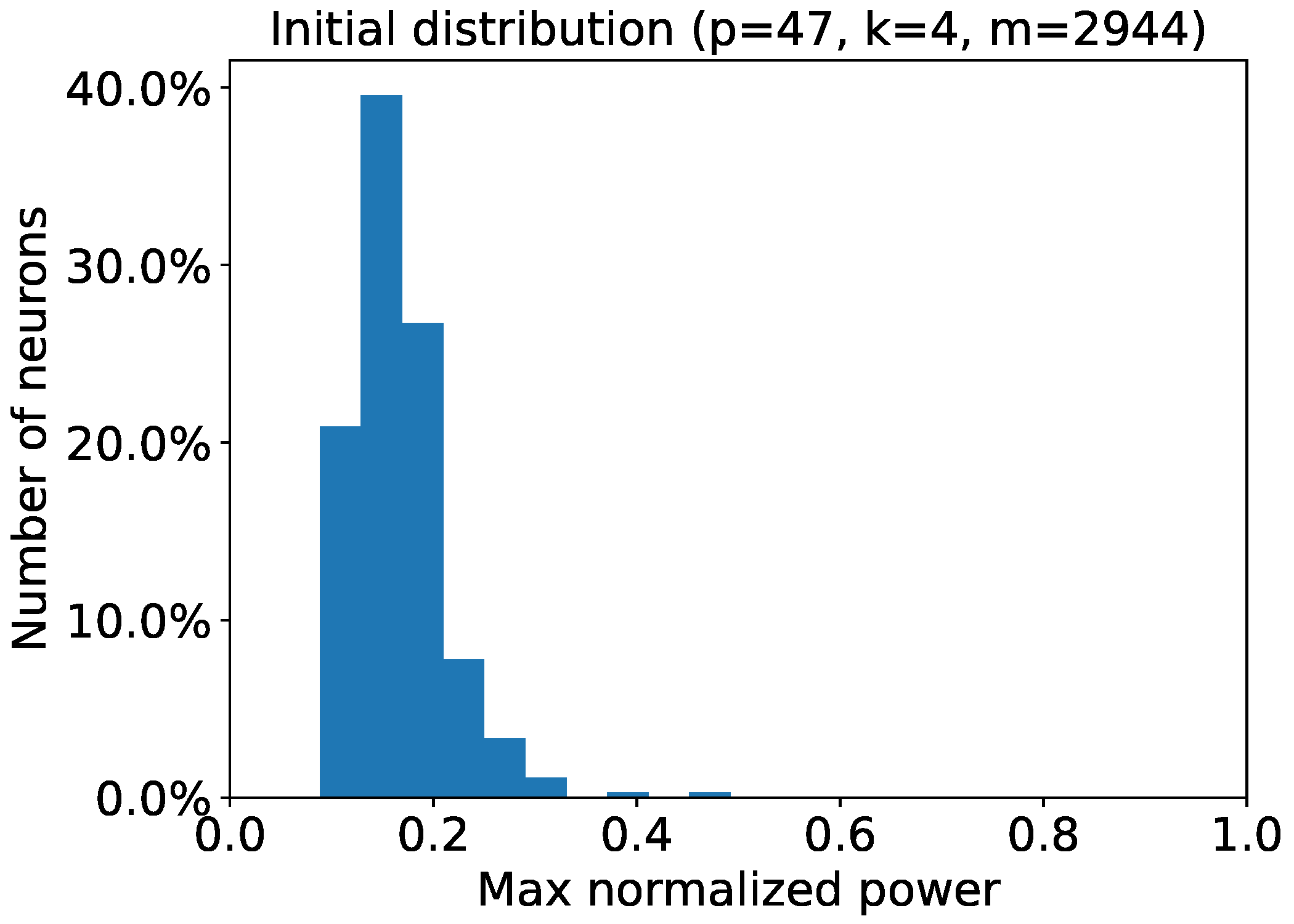}}
\subfloat[]{\includegraphics[width=0.24\linewidth]{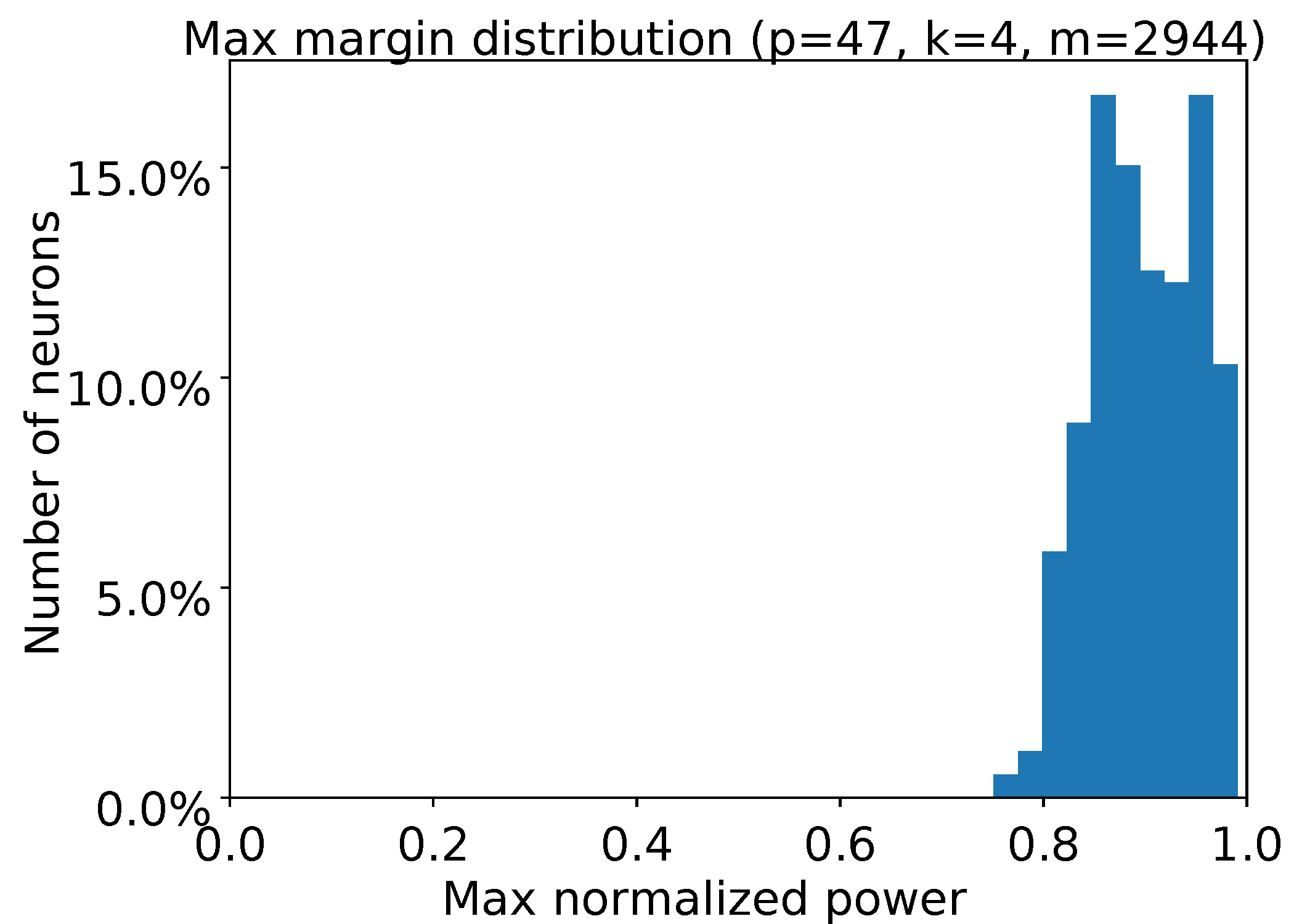}}
    \caption{All Fourier spectrum frequencies being covered and the maximum normalized power of the embeddings (hidden layer weights). The one-hidden layer network with $m=2944$ neurons is trained on $k=4$-sum mod-$p=47$ addition dataset. We denote $\hat{u}[i]$ as the Fourier transform of $u[i]$. Let $\max_i |\hat{u}[i]|^2 /( \sum|\hat{u}[j]|^2 )$ be the maximum normalized power. 
    Mapping each neuron to its maximum normalized power frequency, (a) shows the final frequency distribution of the embeddings. 
    Similar to our construction analysis in Lemma~\ref{lem:construct-k:informal}, we have an almost uniform distribution over all frequencies. 
    (b) shows the maximum normalized power of the neural network with random initialization. (c) shows, in frequency space, the embeddings of the final trained network are one-sparse, i.e., maximum normalized power being almost 1 for all neurons. This is consistent with our max-margin analysis results in Lemma~\ref{lem:construct-k:informal}. See Figure~\ref{fig:nn_freq_k3} and ~\ref{fig:nn_freq_k5} in Appendix~\ref{app:sec:exp_nn} for results when $k$ is 3 or 5.
    }
    \label{fig:nn_freq_k4}
\end{figure*}

\subsection{Technique Overview}
In this section, we propose techniques overview of the proof for our main result. 
We use $\i$ to denote $\sqrt{-1}$.
Let $f: \mathbb{Z}_p \rightarrow \mathbb{C}$. Then, for each frequency $j \in \mathbb{Z}_p$, we define $f$ discrete Fourier transform (DFT) as
$
\hat f(j) := \sum_{\zeta \in \mathbb{Z}_p} f(\zeta) \exp(-2\pi \i \cdot j\zeta /p).
$
Let $\Omega_q^{'*}$ be the single neuron class-weighted maximum-margin solution set (formally defined in Definition~\ref{def:margin_max_problem-k}).

First, we provide a brief informal high-level intuition of our proof. Based on Condition~\ref{lem:translation_property} (proved by Lemma~\ref{lem:construct-k:informal}), the maximum-margin solutions and class-weighted maximum-margin are equivalent to each other. Thus, in Lemma~\ref{lem:margin_soln-k:informal}, we can get the class-weighted maximum-margin solution as the problem has been reduced from a network max-margin problem to a single-neuron max-margin problem. Finally, combining all of these, we have our main Theorem~\ref{thm:main_k:informal}.

Now, we show how to get $\Omega'^*_q$. 

\begin{lemma}[Informal version of Lemma \ref{lem:margin_soln-k}]\label{lem:margin_soln-k:informal}
If for any $\zeta \in \{1, \ldots, \frac{p-1}{2}\}$, there exists a scaling constant $\beta \in \mathbb{R}$,  such that $u_i(a_i) = \beta \cdot \cos ( \theta_{u_i}^* + 2 \pi \zeta a_i /p )$ for any $i \in [k]$ and $w(c) =\beta \cdot \cos (\theta_w^*+2 \pi \zeta c / p)$,
where $\theta_{u_1}^*, \dots, \theta_{u_k}^*, \theta_w^* \in \R$ are some phase offsets satisfying $\theta_{u_1}^*+\dots+\theta_{u_k}^*=\theta_w^*$.
Then, we have $\Omega_q^{\prime *}  
    = \{(u_1, \dots, u_k, w)\}, $ and 
$ \gamma^* = \frac{2(k!)}{(2k+2)^{(k+1)/2}(p-1) p^{(k-1)/2}}$.
\end{lemma}

\begin{proof}[Proof sketch of Lemma~\ref{lem:margin_soln-k:informal}]
See formal proof in Appendix~\ref{app:single:get_solution_set_k}. 
The proof establishes the maximum-margin solution's sparsity in the Fourier domain through several key steps. Initially, by Lemma~\ref{lemma:fourier_space-k}, focus is directed to maximizing Eq.~\eqref{eq:discrete_fourier_transforms_uvw-k}. 
For odd \( p \), Eq.~\eqref{eq:discrete_fourier_transforms_uvw-k} can be reformulated with magnitudes and phases of $\hat{u}_i$ and $\hat{w}$ (discrete Fourier transform of $u_i$ and $w$), leading to an equation involving cosine of their phase differences.
Plancherel's theorem is then employed to translate the norm constraint to the Fourier domain. This allows for the optimization of the cosine term in the sum, effectively reducing the problem to maximizing the product of magnitudes of $\hat{u}_i$ and $\hat{w}$ (Eq.~\eqref{eq:reduce_optimize_discrete_fourier_transforms_uvw-k}).
By applying the inequality of arithmetic and geometric means, we have an upper bound for the optimization problem. 
To achieve the upper bound, equal magnitudes are required for all $\hat{u}_i$ and $\hat{w}$ at a single frequency, leading to Eq.~\eqref{eq:hat_u_same-k}. The neurons are finally expressed in the time domain, demonstrating that they assume a specific cosine form with phase offsets satisfying certain conditions.  
\end{proof}

Next, we show the number of neurons required to solve the problem and the properties of these neurons. We demonstrate how to use these neurons to construct the network $\theta^*$.

\begin{lemma}[Informal version of Lemma \ref{lem:construct-k}]\label{lem:construct-k:informal}
Let $\cos_\zeta(x)$ denote $\cos (2 \pi \zeta x / p)$. 
Then, we have 
the maximum $L_{2,k+1}$-margin solution $\theta^*$ will consist of $2^{2k-1} \cdot {p-1 \over 2}$ neurons $\theta^*_i \in \Omega_q^{'*}$ to simulate $\frac{p-1}{2}$ type of cosine computation, where each cosine computation is uniquely determined a $\zeta \in \{1, \ldots, \frac{p-1}{2}\}$. In particular, for each $\zeta$ the cosine computation is   $\cos_\zeta(a_1+\dots+a_k-c), \forall a_1,\dots,a_k,c \in \mathbb{Z}_p$.
\end{lemma}
\begin{proof}[Proof sketch of Lemma~\ref{lem:construct-k:informal}]
See formal proof in Appendix~\ref{app:construct:constructions_for_theta_k}. 
Our goal is to show that $2^{2k-1} \cdot {p-1 \over 2}$ neurons $\theta_i^* \in \Omega_q^{'*}$ are able to simulate $\frac{p-1}{2}$ type of $\cos$ computation.
We have the following expansion function of $\cos_\zeta(x)$, which denotes $\cos(2\pi \zeta x /p)$.
\begin{align*}
    \cos_\zeta(\sum_{i = 1}^k a_i) = &  ~ \sum_{b \in \{0,1\}^{k}} \prod_{i=1}^{k} \cos^{1-b_i}(a_i) \cdot \sin^{b_i}(a_i)  \\ & ~ \cdot {\bf 1}[ \sum_{i=1}^{k} b_i \%2 = 0 ] \cdot (-1)^{ {\bf 1 }[ \sum_{i=1}^{k} b_i \% 4 = 2 ] }.
\end{align*} 
The above equation can decompose a $\cos(\sum )$ to some basic elements. We have $2^k$ terms in the above equation. By using the following fact in Lemma~\ref{lem:sum2product},
\begin{align*}
    2^k \cdot k! \cdot  \prod_{i=1}^k a_i = \sum_{c \in \{-1,+1\}^k } (-1)^{(k-\sum_{i=1}^k c_i)/2} ( \sum_{j=1}^k c_j a_j)^k,
\end{align*}
where each term can be constructed by $2^{k-1}$ neurons. Therefore, we need $2^{k-1} 2^k$ total neurons. To simulate $\frac{p-1}{2}$ type of simulation, we need $2^{2k-1} \frac{p-1}{2}$ neurons.
Then, using the Lemma~\ref{lem:combine}, we construct the network $\theta^*$. By using the Lemma~\ref{lemma:multi-class} from {\cite{meo+23}}, we get it is the maximum-margin solution.
\end{proof}
\section{EXPERIMENTS}\label{sec:exp}
First, we conduct simulation experiments to verify our analysis for $k=3,4,5$.  
Then, we show that the one-layer transformer learns 2-dimensional cosine functions in their attention weights. 
Finally, we show the grokking phenomenon under different $k$. Please refer to Appendix~\ref{app:sec:implement} for details about implementation.

\subsection{One-hidden Layer Neural Network}\label{sec:exp_nn}
We conduct simulation experiments to verify our analysis. In Figure~\ref{fig:nn_w_k4} and Figure~\ref{fig:nn_freq_k4}, we use SGD to train a two-layer network with $m=2944=2^{2k-2} \cdot (p-1)$ neurons, i.e., Eq.~\eqref{eq:nn}, on $k=4$-sum mod-$p=47$ addition dataset, i.e., Eq.~\eqref{eq:data}. Figure~\ref{fig:nn_w_k4} shows that the networks trained with SGD have single-frequency hidden neurons, which support our analysis in Lemma~\ref{lem:margin_soln-k:informal}. Furthermore, Figure~\ref{fig:nn_freq_k4} demonstrates that the network will learn all frequencies in the Fourier spectrum, which is consistent with our analysis in Lemma~\ref{lem:construct-k:informal}. Together, they verify our main results in Theorem~\ref{thm:main_k:informal} and show that the network trained by SGD prefers to learn Fourier-based circuits. There are more similar results when $k$ is 3 or 5 in Appendix~\ref{app:sec:exp_nn}.

\begin{figure}[!ht]
  \centering
\includegraphics[width=\linewithvalue\linewidth]{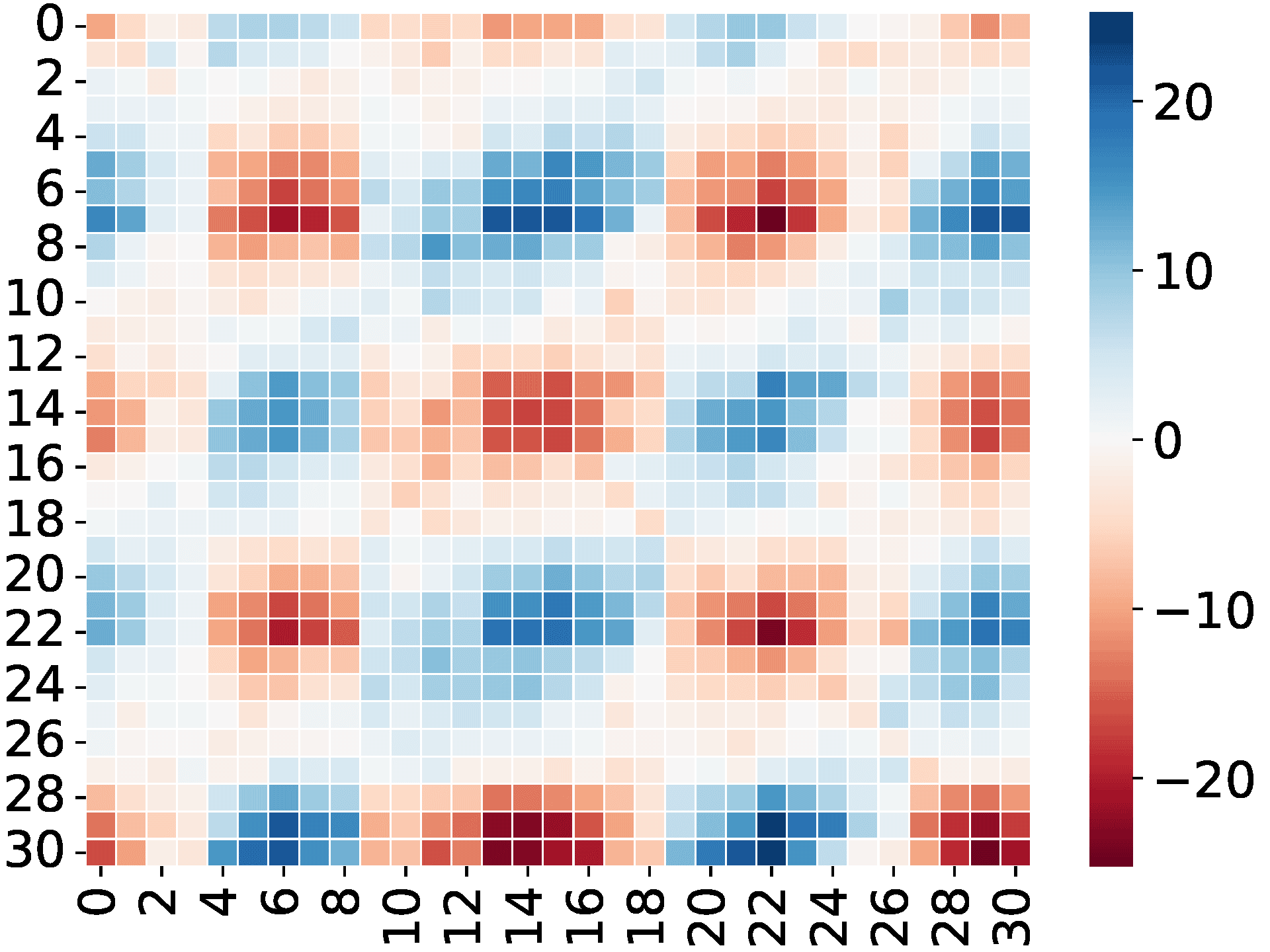}
\includegraphics[width=\linewithvalue\linewidth]{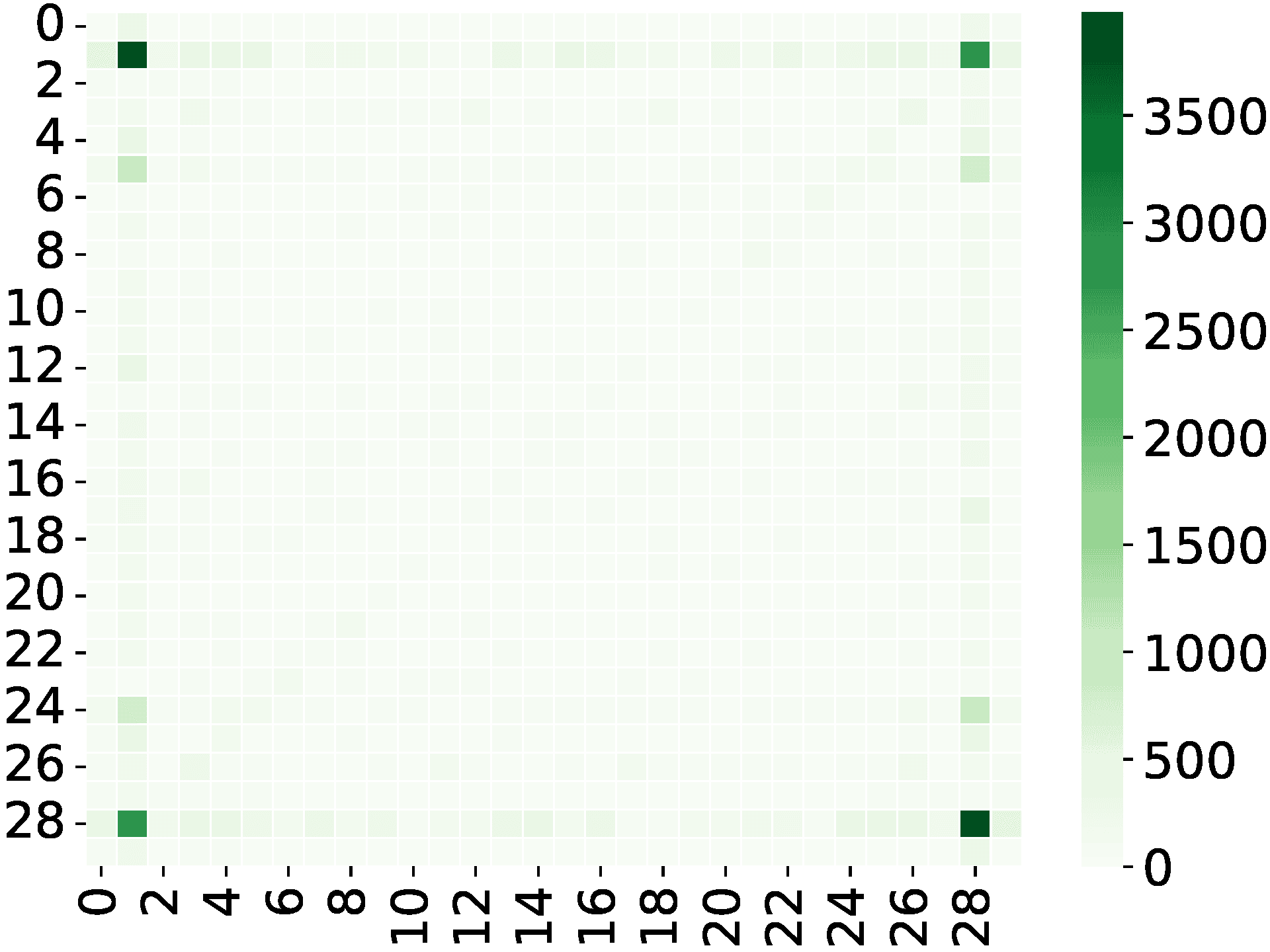}
\includegraphics[width=\linewithvalue\linewidth]{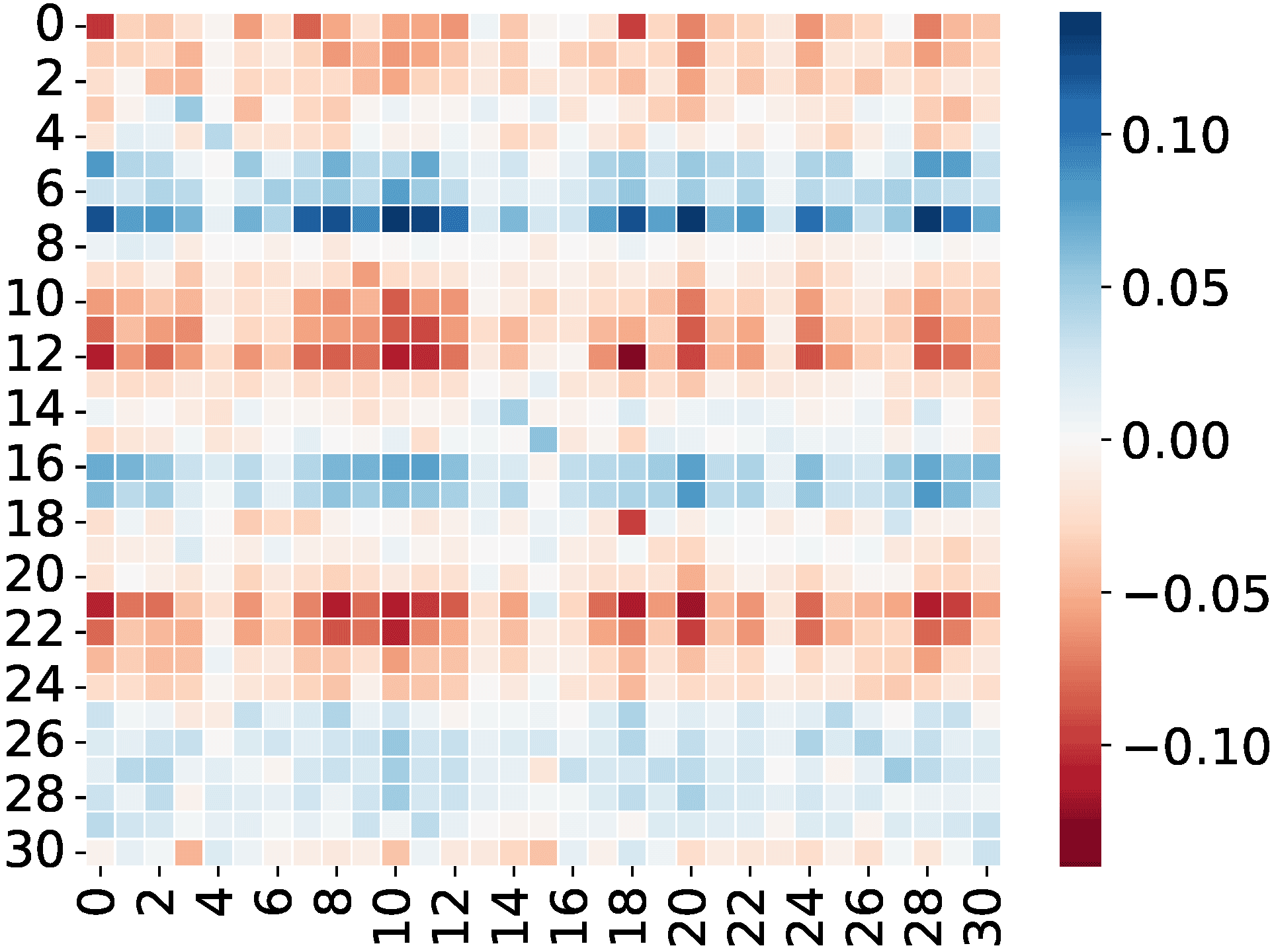}
\includegraphics[width=\linewithvalue\linewidth]{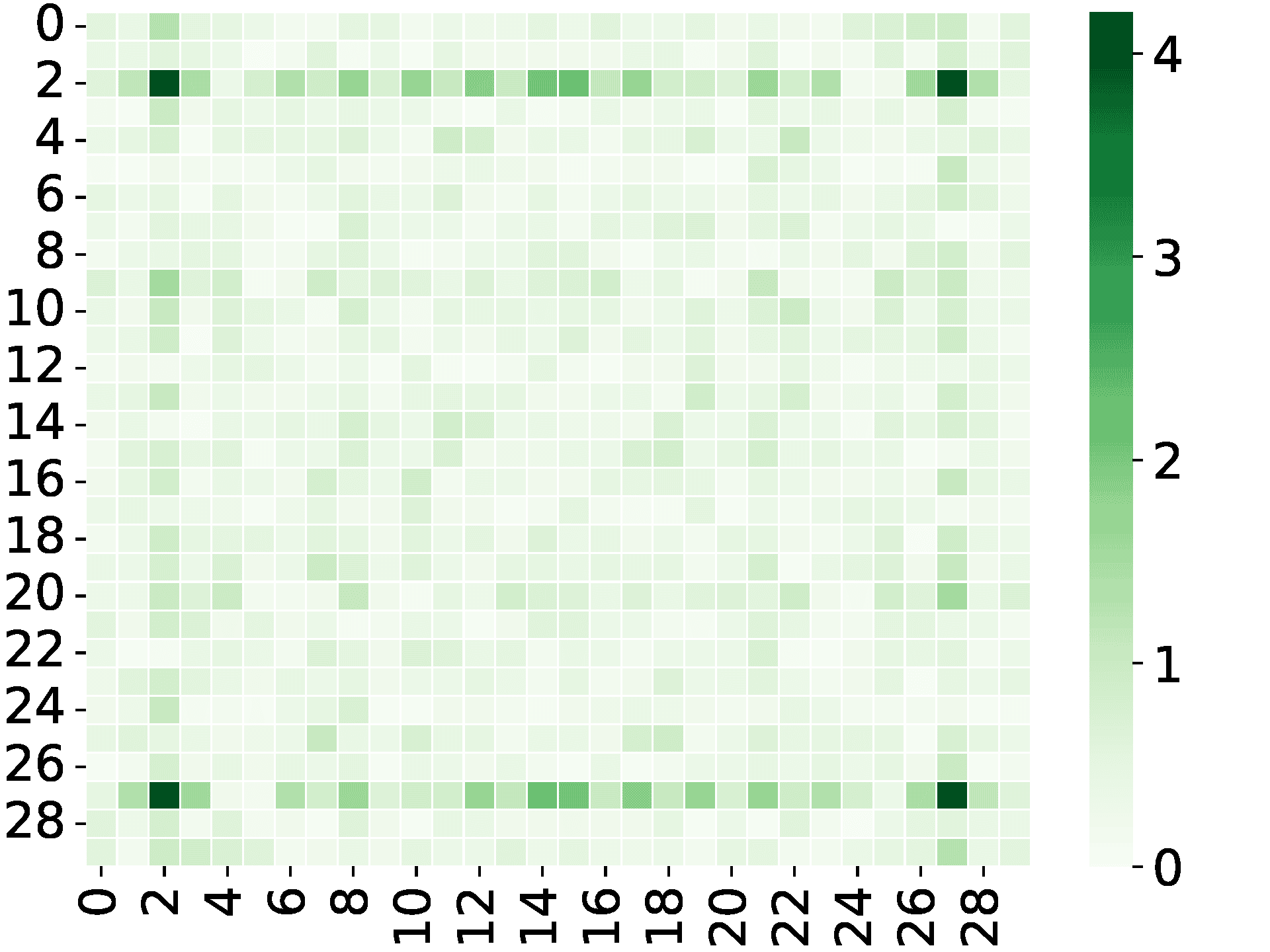}
\includegraphics[width=\linewithvalue\linewidth]{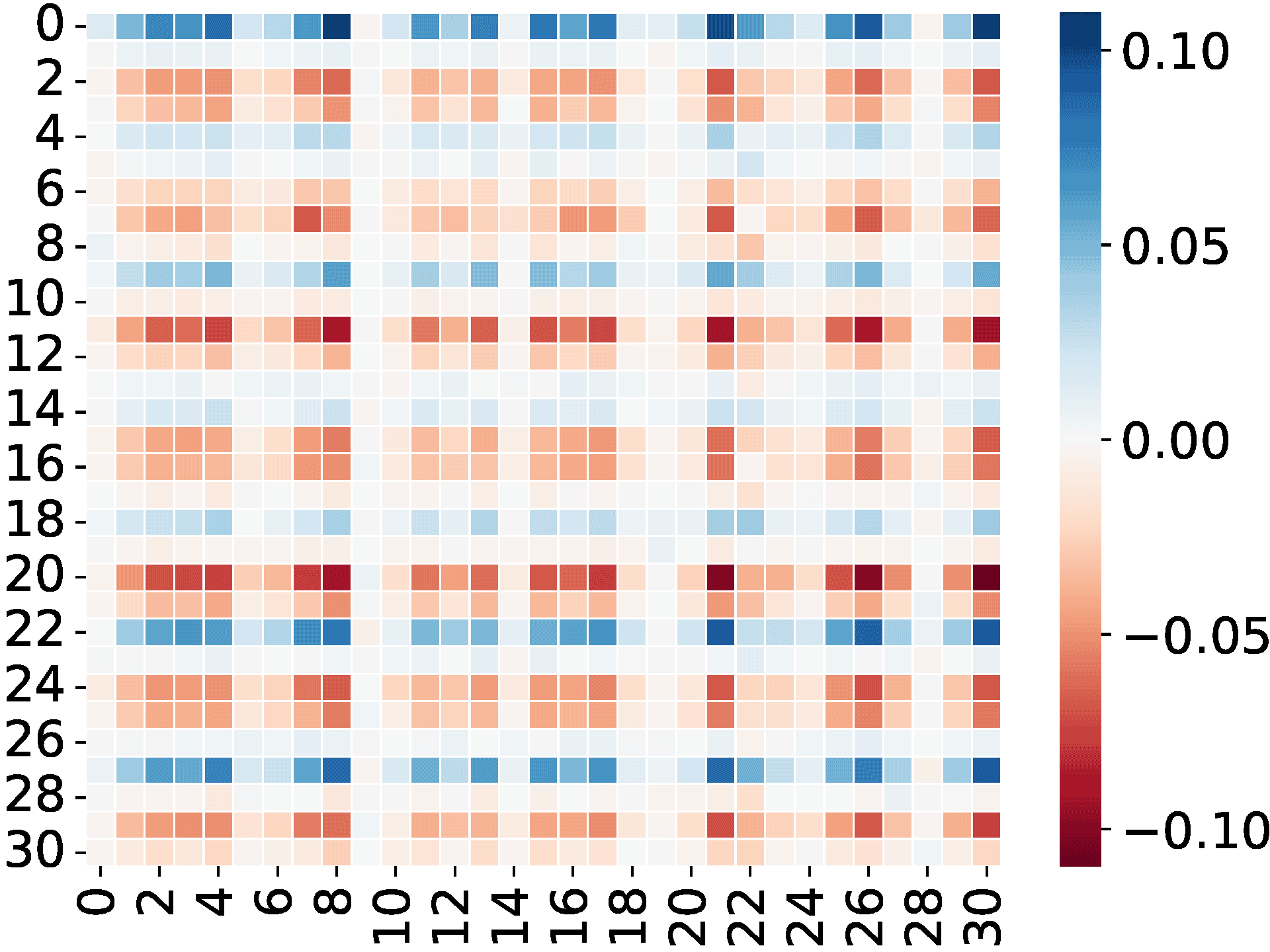}
\includegraphics[width=\linewithvalue\linewidth]{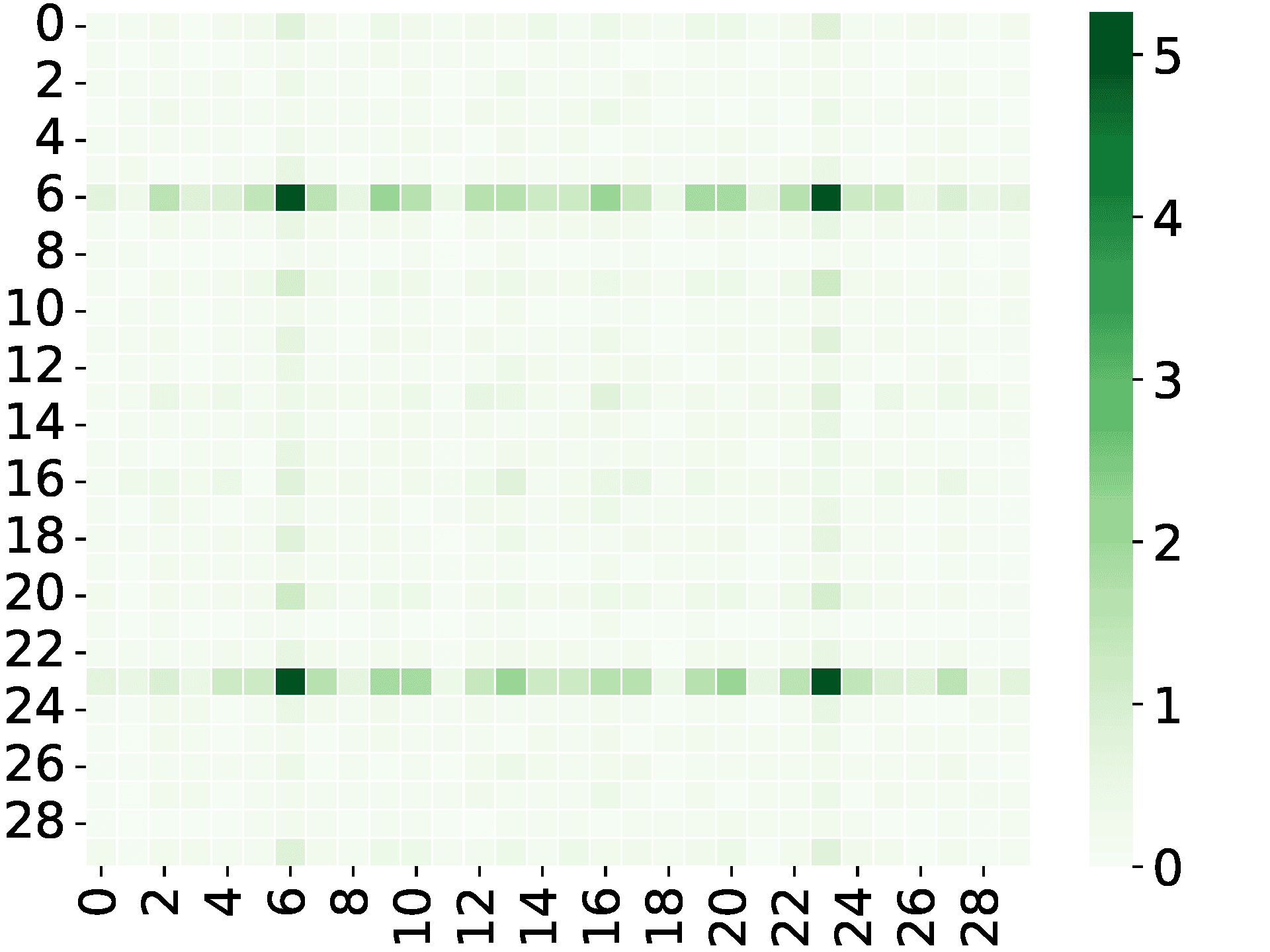}
\includegraphics[width=\linewithvalue\linewidth]{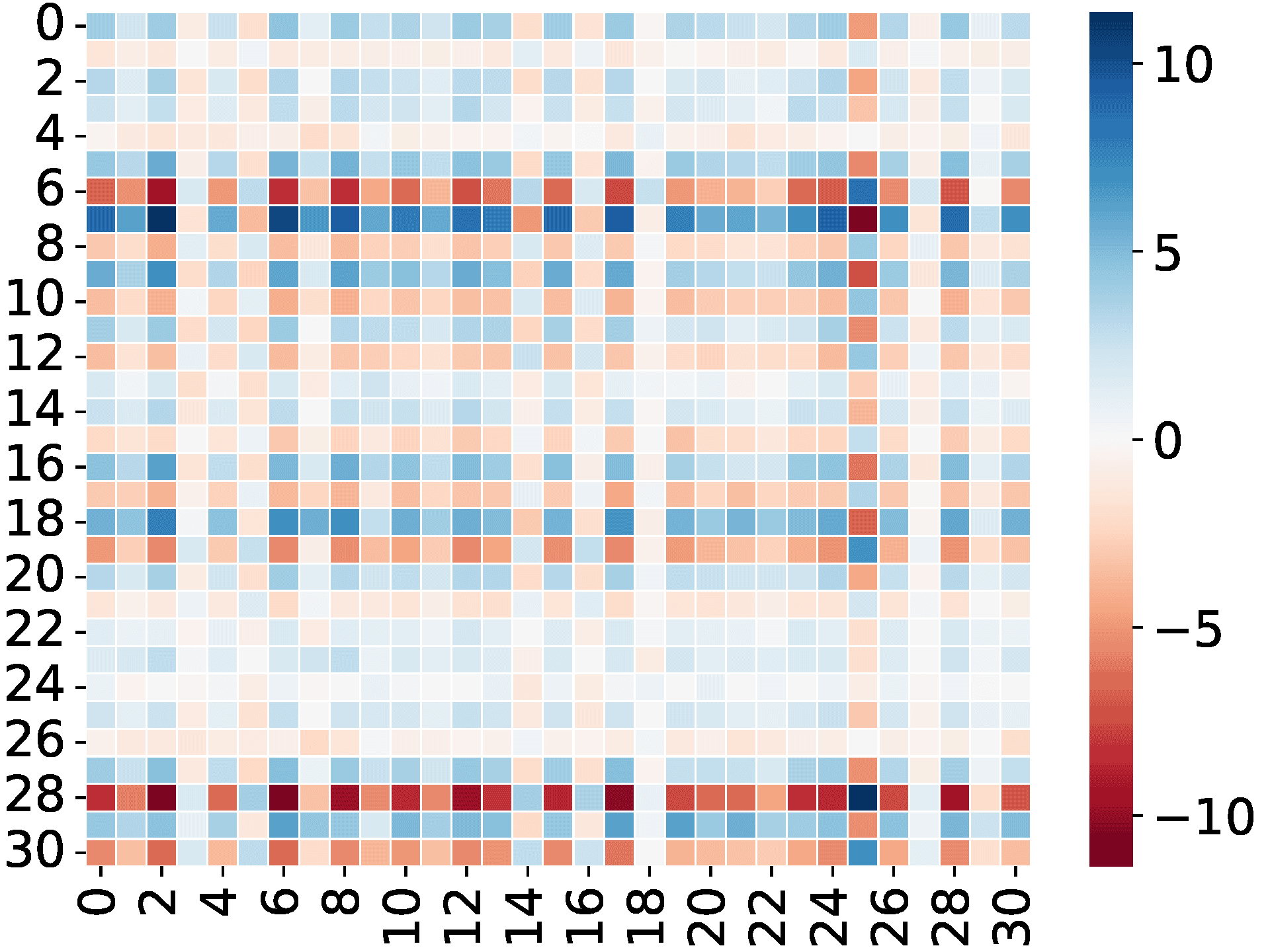}
\includegraphics[width=\linewithvalue\linewidth]{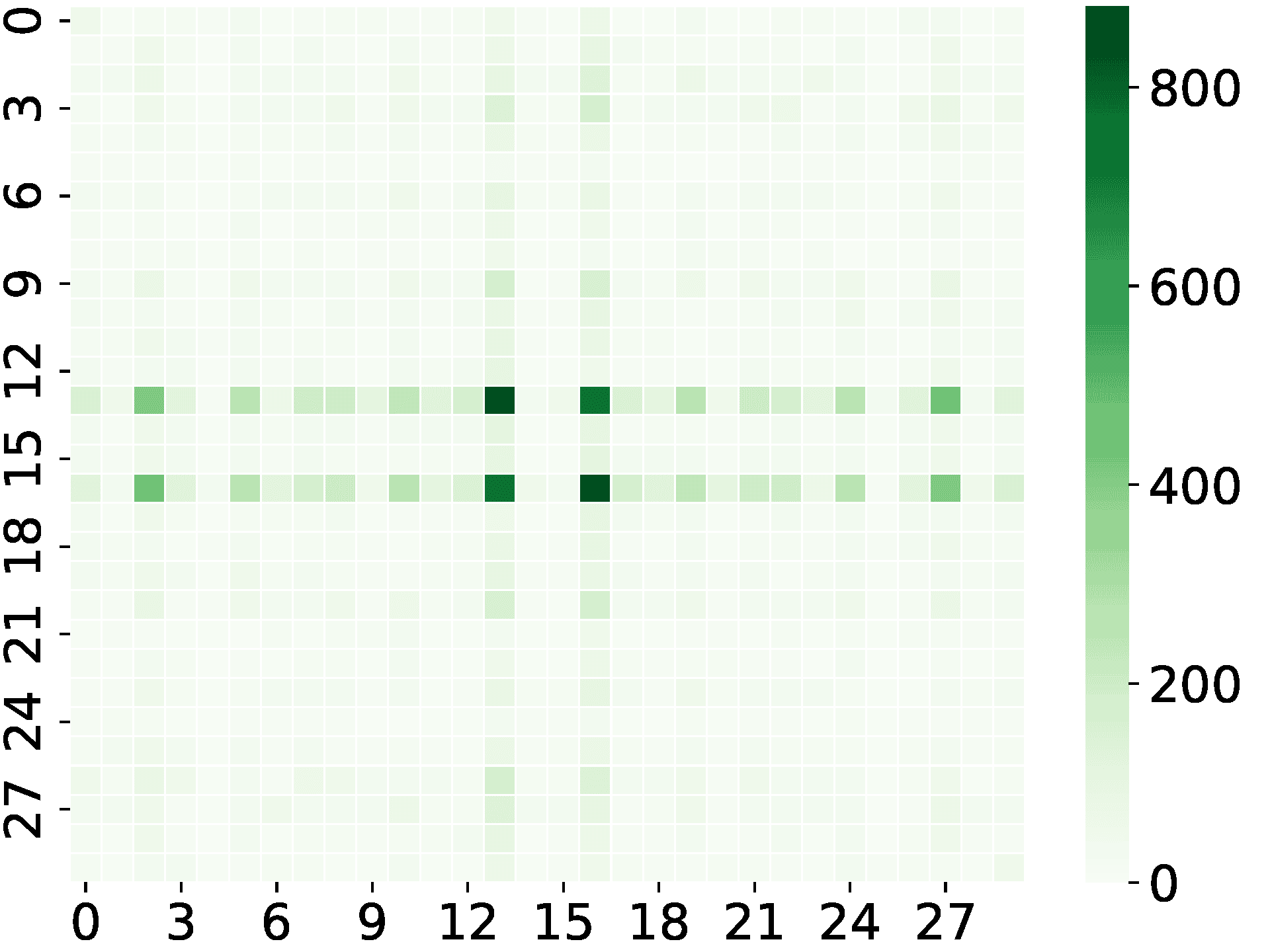}
    \caption{2-dimension cosine shape of the trained $W^{KQ}$ (attention weights) and their Fourier power spectrum. The one-layer transformer with attention heads $m=160$ is trained on $k=4$-sum mod-$p=31$ addition dataset. We even split the whole datasets ($p^k = 31^4$ data points) into training and test datasets. Every row represents a random attention head from the transformer. The left figure shows the final trained attention weights being an apparent 2-dim cosine shape. The right figure shows their 2-dim Fourier power spectrum. The results in the figures are consistent with Figure~\ref{fig:nn_w_k4}. See Figure~\ref{fig:s_k3} and Figure~\ref{fig:s_k5} in Appendix~\ref{app:sec:exp_transformer} for similar results when $k$ is 3 or 5.}
    \label{fig:s_k4}
\end{figure}

\subsection{One-layer Transformer}\label{sec:exp_transfomer}
We find similar results in one-layer transformers. Let $E$ be input embedding and $W^P, W^V, W^K, W^Q$ be projection, value, key and query matrix. The $m$-heads attention layer can be written as 
\begin{align*}
    & W^{P}\begin{pmatrix}
W_1^{V\top} E \cdot \softmax\left({E^\top W_1^{K} W_1^{Q\top} E}\right) \\
\dots \\
W_m^{V\top} E \cdot \softmax\left({E^\top W_m^{K} W_m^{Q\top} E}\right)
\end{pmatrix}.
\end{align*}
We denote $W^{K} W^{Q\top}$ as $W^{KQ}$ and call it attention matrix.
In Figure~\ref{fig:s_k4}, we train a one-layer transformer with $m=160$ heads attention and hidden dimension 128, i.e., above equation, on $k=4$-sum mod-$p=31$ addition dataset, i.e., Eq.~\eqref{eq:data}. Figure~\ref{fig:s_k4} shows that the SGD-trained one-layer transformer learns 2-dim cosine shape attention matrices, which is similar to the one-hidden layer neural networks in Figure~\ref{fig:nn_w_k4}. This means that the attention layer has a learning mechanism similar to neural networks in the modular arithmetic task. It prefers to learn (2-dim) Fourier-based circuits when trained by SGD. There are more similar results when $k$ is 3 or 5 in Appendix~\ref{app:sec:exp_transformer}.

\begin{figure*}[t!]
    \centering
{\includegraphics[width=0.24\linewidth]{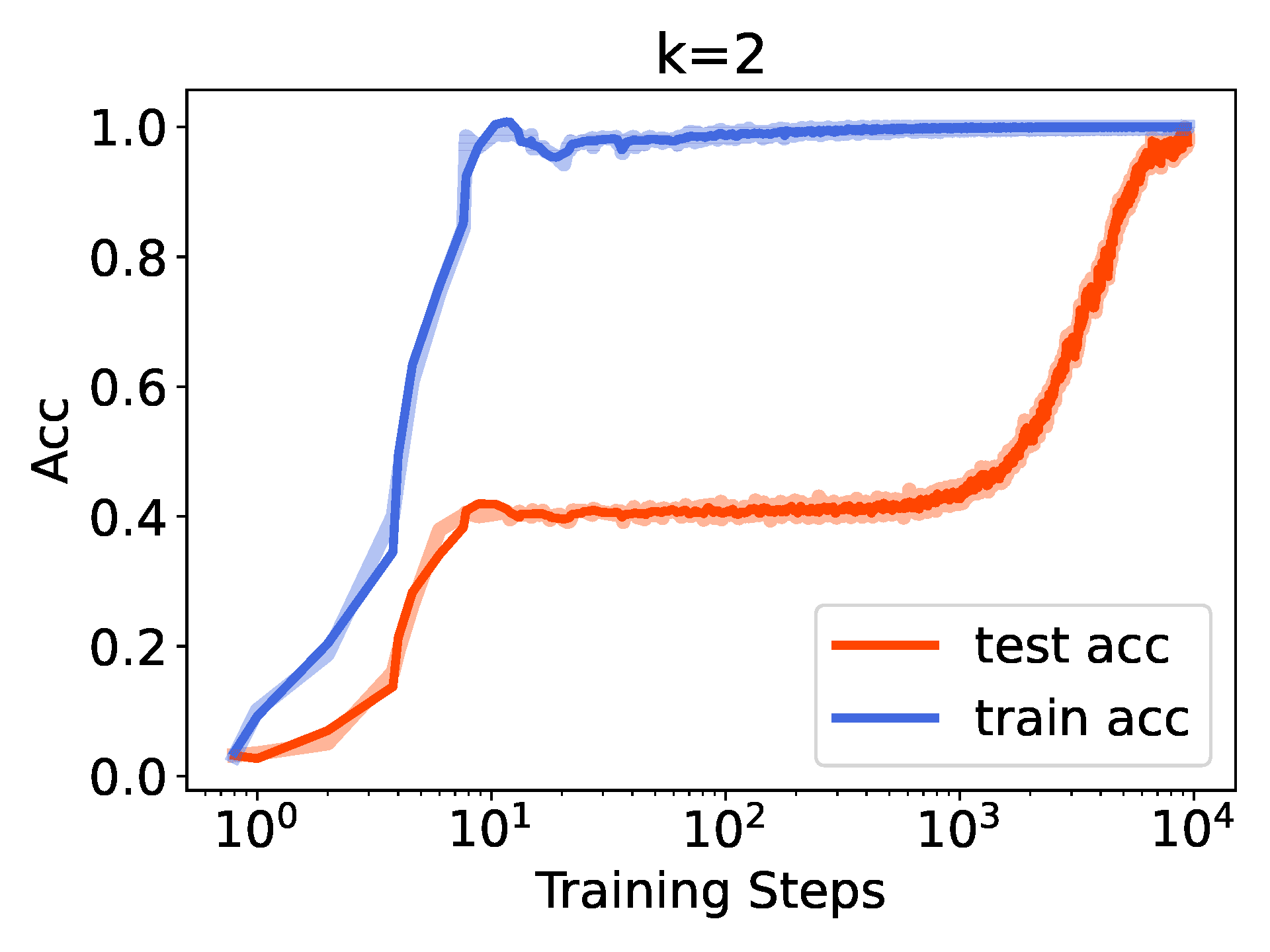}}
{\includegraphics[width=0.24\linewidth]{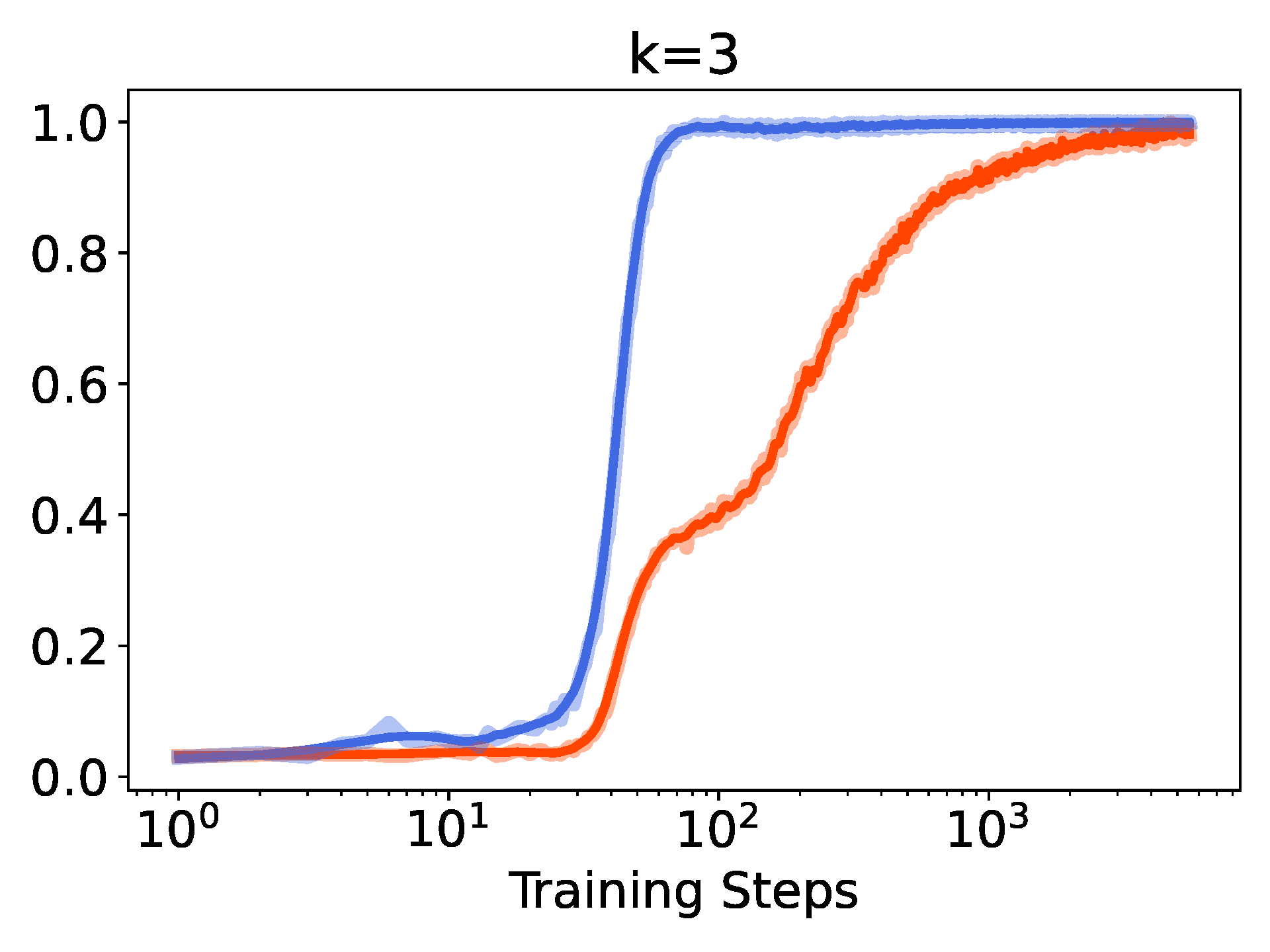}}
{\includegraphics[width=0.24\linewidth]{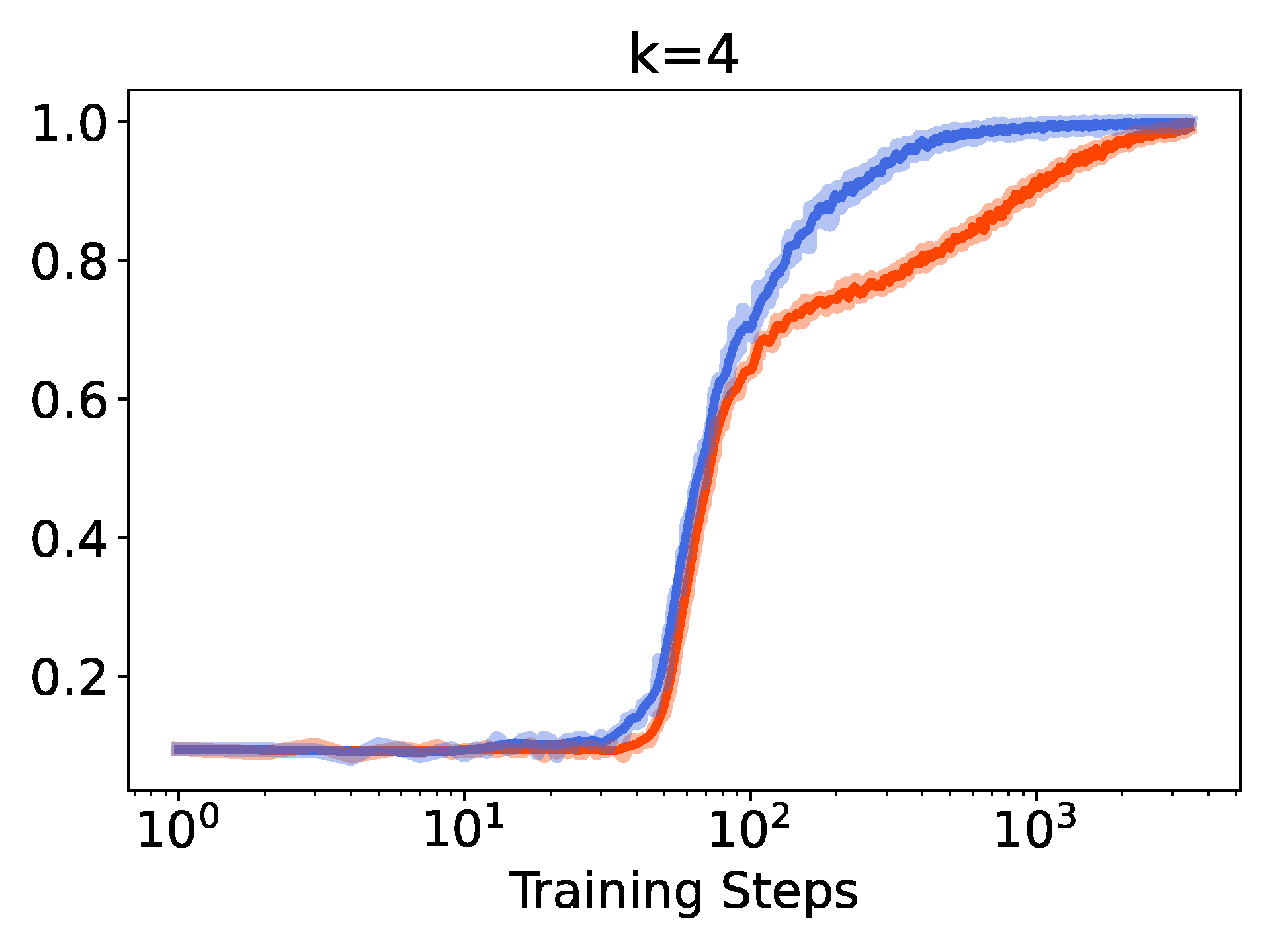}}
{\includegraphics[width=0.24\linewidth]{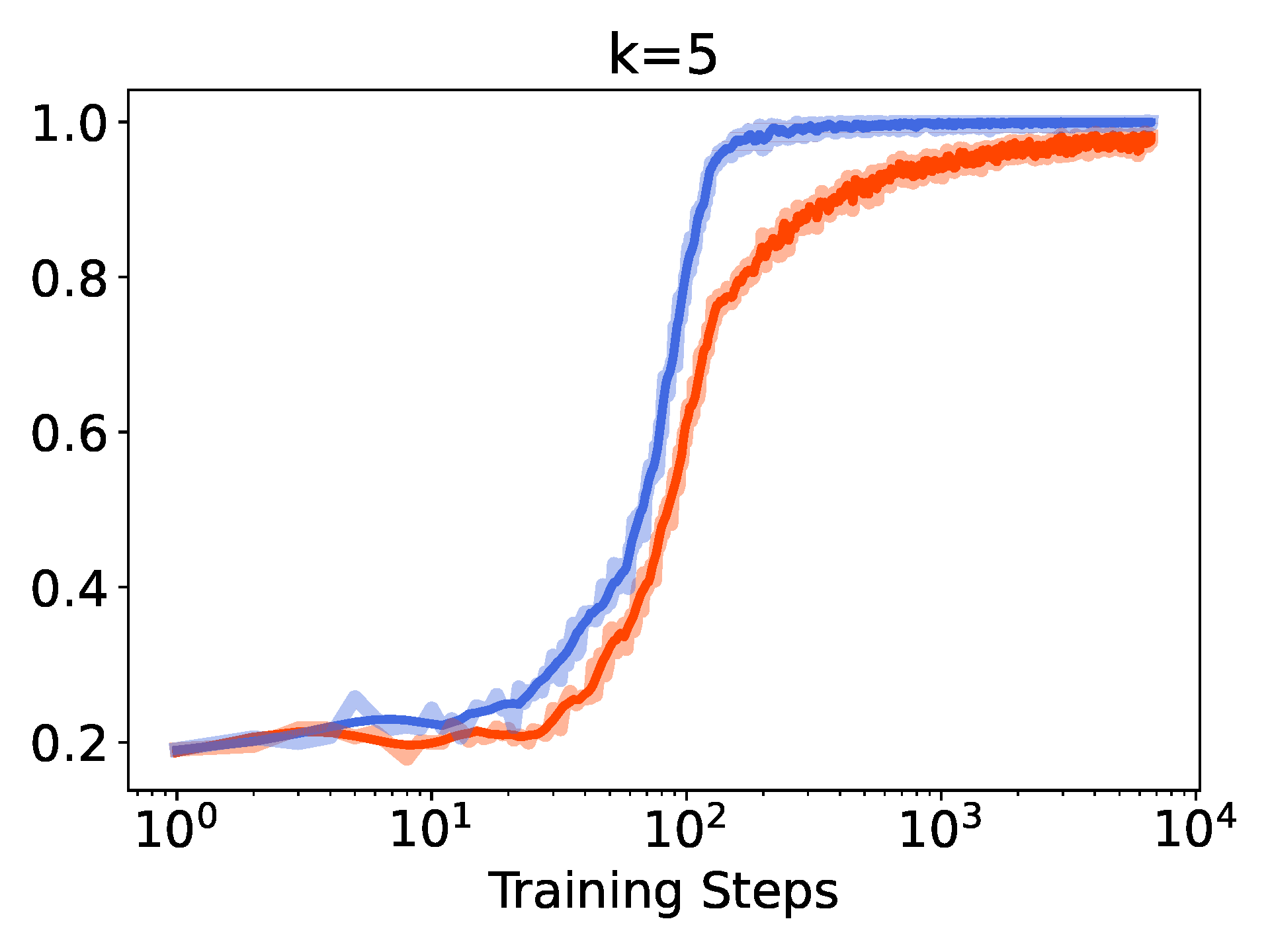}}
    \caption{ Grokking (models abruptly transition from bad generalization to perfect generalization after a large number of training steps) under learning modular addition involving $k=2,3,4,5$ inputs. We train two-layer transformers with $m=160$ attention heads on $k=2,3,4,5$-sum mod-$p = 97,31,11,5$ addition dataset with $50$\% of the data in the training set under AdamW~\cite{loshchilov2017decoupled} optimizer 1e-3 learning rate and 1e-3 weight decay. We use different $p$ to guarantee the dataset sizes are roughly equal to each other. 
    The blue curves show training accuracy, and the red ones show validation accuracy. There is a grokking phenomenon in all figures. However, as $k$ increases, the grokking phenomenon becomes weak. See explanation in Section~\ref{sec:exp}. 
    }
    \label{fig:grok}
\end{figure*}

\subsection{Grokking under Different \texorpdfstring{$k$}{} }\label{sec:exp_grok}
To support the importance of our data setting, we study the grokking phenomenon in our data distribution. 
Following the experiments' protocol in {\cite{pbe+22}}, we show there is the grokking phenomenon under different $k$. We train two-layer transformers with $m = 160$ attention heads and hidden dimension as 128 on $k = 2, 3, 4, 5$-sum mod-$p = 97, 31, 11, 5$ addition dataset with 50\% of the data in training. We use different $p$
to guarantee the dataset sizes are roughly equal to each other. 
Figure~\ref{fig:grok} shows that the grokking weakens as the number of $k$ increases, which is consistent with our analysis.
When $k$ increases, the function class will become more complicated, as we may need more neurons to achieve the max-margin solution. Thus, we use our Theorem~\ref{thm:main_k:informal} as a metric to measure the data complexity.
It implies that when the ground-truth function class becomes ``complicated'', the transformers need to train more steps to fit the training datasets, and the generalization tends to be better. 
Brilliant recent works by~\cite{lyu2023dichotomy,kbgp23} argue that, during learning, the network will be first in the lazy training/NTK regime and then transfer to the rich/feature learning regime sharply, leading to a grokking phenomenon. 
We use learning steps required for regime switch as a metric of grokking strength.

\textbf{``Underfitting'' in NTK but ``overfitting'' in Feature Learning. } 
NTK is a notorious overparameterized regime, which probably needs a much larger number of neurons than our max-margin convergence case, i.e., much larger than $\Omega(2^{2k})$ in Theorem~\ref{thm:main_k:informal}. 
Thus, under the fixed $m$ and increasing $k$, the model may easily escape the NTK regime, or there is no longer an NTK regime. Thus, we will see a weaker grokking phenomenon as the learning steps needed to transfer from the NTK regime to the feature learning regime become fewer. 
With increasing $k$, the model will have an ``underfitting'' issue in the NTK regime, meaning the model must need feature learning to fit the task but cannot only fit the task by NTK. However, the model still has an ``overfitting'' in the feature learning regime. 

\section{DISCUSSION}

\subsection{Grokking in Transformers} 
The interpretability of grokking in Transformers is explored in \cite{ncl+23}. 
By examining various intermediate states within the residual stream of the Transformer model, it is validated that the model employs Fourier features to tackle the modular addition task. 
However, fully comprehending how the Transformer model and LLMs perform modular addition remains challenging based on the current work, particularly from a theoretical standpoint. We contend that beginning with a simplistic model setup and achieving a thorough and theoretical understanding of how the network utilizes Fourier features to address the problem serves as a valuable starting point and it provides a theoretical understanding of the grokking phenomenon. We believe that further study on Transformers will be an interesting and important future direction.

\subsection{Grokking, Benign Overfitting, and Implicit Bias}  Recently, {\cite{xu2023benign}} connects the grokking phenomenon to benign overfitting~\citep{bartlett2020benign,cao2022benign,tsigler2023benign,frei2022benign,frei2023benign}. It shows how the network undergoes a grokking period from catastrophic to benign overfitting.
{\cite{lyu2023dichotomy,kbgp23}} uses implicit bias~\citep{shn18,gunasekar2018characterizing,ji2019implicit,shah2020pitfalls,moroshko2020implicit,chizat2020implicit,lyu2021gradient,jacot2022implicit,xu2023improving,xsw+24} to explain grokking, where grokking happens if the early phase bias implies an overfitting solution while late phase bias implies a generalizable solution.
The intuition from the benign overfitting and the implicit bias well align with our observation in Section~\ref{sec:exp}.
It is interesting and valuable to rigorously analyze the grokking or emergent ability under different function class complexities, e.g., Eq~\eqref{eq:data}. We leave this challenge problem as a future work.

\subsection{High Order Correlation Attention} 
{\cite{sanford2023representational,as23,alman2023capture,lssz24_tat,lls+24_tensor,zly+25}} state that, when $k=3$, $a_1+a_2+a_3 \bmod p$ is hard to be captured by traditional attention. Thus, they introduce high-order attention to capture high-order correlation from the input sequence. However, in Section~\ref{sec:exp}, we show that one-layer transformers have a strong learning ability and can successfully learn modular arithmetic tasks even when $k=5$. This implies that the traditional attention may be more powerful than we expect.

\subsection{Connection to Parity and SQ Hardness}
If we let $p=2$, then $(a_1+\dots+a_k) \bmod p$ will degenerate to parity function, i.e., $b_1, \dots, b_k \in \{\pm 1\}$ and determining $\prod_{i=1}^k b_i$. Parity functions serve as a fundamental set of learning challenges in computational learning theory, often used to demonstrate computational obstacles~\citep{shalev2017failures}.
In particular, $(n, k)$-sparse parity problem is notorious hard to learn, i.e.,  Statistical Query (SQ) hardness~\citep{blum1994weakly}. {\cite{dm20}} showed that one-hidden layer networks need an $\Omega(\exp(k))$ number of neurons or an $\Omega(\exp(k))$ number of training steps to successfully learn it by SGD. In our work, we are studying Eq.~\eqref{eq:nn}, which is a more general function than parity and indeed is a learning hardness. Our Theorem~\ref{thm:main_k:informal} states that we need $\Omega(\exp(k))$ number of neurons to represent the maximum-margin solution, which well aligns with existing works. Our experiential results in Section~\ref{sec:exp} are also consistent. Hence, our modular addition involving $k$ inputs function class is a good data model to analyze and test the model learning ability, i.e., approximation, optimization, and generalization.

\section{CONCLUSION}
We study neural networks and transformers learning on $(a_1+\dots+a_k) \bmod p$. We theoretically show that networks prefer to learn Fourier circuits. Our experiments on neural networks and transformers support our analysis. Finally, we study the grokking phenomenon under this new data setting.

\section{LIMITATIONS}\label{sec:limitation}
Our work has made progress in exploring how neural networks and Transformers can solve complex mathematical problems such as modular addition operation, but the practical application scope of their conclusions is limited.
On the other hand, we admit that our theorem can provide intuition but cannot fully explain the phenomena shown in Figure~\ref{fig:grok}. 
Thus, we would like to introduce this more general data setting to the community so that we can study and understand grokking in a more broad way.
Studying the relationship between the number of neurons and the grokking strength is interesting and important, and we will leave it as our future work.

\section*{Acknowledgement}
Research is partially supported by the National Science Foundation (NSF) Grants 2008559-IIS, 2023239-DMS, CCF-2046710, and Air Force Grant FA9550-18-1-0166.

\ifdefined\isarxiv

\bibliographystyle{alpha}
\bibliography{ref}
\else
\bibliography{ref}
\bibliographystyle{plainnat} 
\input{20_checklist}
\fi

\newpage
\onecolumn
\appendix

\ifdefined\isarxiv
\begin{center}
	\textbf{\LARGE Appendix }
\end{center}
\else
\aistatstitle{Fourier Circuits in Neural Networks and Transformers: A Case Study of Modular Arithmetic with Multiple Inputs: \\
Supplementary Materials}

{\hypersetup{linkcolor=black}
\tableofcontents
\bigbreak
\bigbreak
}
\fi

\paragraph{Roadmap.}
In Section~\ref{sec:impact}, we discuss the societal impacts of our work.
In Section~\ref{app:def}, we introduce some definitions that will be used in the proof. In Section~\ref{app:tools}, we introduce some auxiliary lemma from previous work that we need. In Section~\ref{app:single}, Section~\ref{app:construct}, Section~\ref{app:frequency}, Section~\ref{app:main}, we provide the proof of our Lemmas and our main results. In particular, we provide two versions of proof (1) $k=3$ and (2) general $k \ge 3$. We use $k=3$ version to illustrate our proof intuition and then extend our proof to the general $k$ version. Finally, in Section~\ref{app:exp}, we provide more experimental results and implementation details.

\section{Societal Impact}\label{sec:impact}
Our work aims to understand the potential of large language models in mathematical reasoning and modular arithmetic. Our paper is purely theoretical and empirical in nature (mathematics problem) and thus we foresee no immediate negative ethical impact.

We propose that neural networks and transformers prefer to learn Fourier circuits when training on modular addition involving $k$ inputs under SGD, which may have a positive impact on the machine learning community. We hope our work will inspire effective algorithm design and promote a better understanding of large language models learning mechanisms.

\section{More Notations and Definitions}\label{app:def}

We use $\i$ to denote $\sqrt{-1}$.
Let $z = a + \i b$ denote a complex number where $a$ and $b$ are real numbers. Then we have $\ov{z} = a - \i b$ and $|z| := \sqrt{a^2 + b^2}$.

For any positive integer $n$, we use $[n]$ to denote set $\{1,2,\cdots,n\}$. We use $\E[]$ to denote expectation. We use $\Pr[]$ to denote probability. We use $z^\top$ to denote the transpose of a vector $z$. 

Considering a vector $z$, we denote the $\ell_2$ norm as  $\|z\|_2:=( \sum_{i=1}^n z_i^2 )^{1/2}$. We denote the $\ell_1$ norm as $\|z\|_1:=\sum_{i=1}^n |z_i| $. The number of non-zero entries in vector $z$ is defined as $\|z\|_0$.  $\| z \|_{\infty}$ is  defined as   $\max_{i \in [n]} |z_i|$.

We define the vector norm and matrix norm as the following.
\begin{definition}[$L_b$ (vector) norm]
Given a vector $v \in \R^n$ and $b \geq 1$, we have 
$
\|v\|_b := (\sum_{i=1}^n |v_{i}|^b)^{1/b}.
$
\end{definition}
\begin{definition}[$L_{a,b}$ (matrix) norm]\label{def:norm}
The $L_{a,b}$ norm of a network with parameters $\theta = \{\theta_i \}_{i=1}^m$ is $\|\theta\|_{a,b}:= (\sum_{i=1}^m \|\theta_i \|_a^b)^{1/b}$, where $\theta_i $ denotes the vector of concatenated parameters for a single neuron. 
\end{definition}

We define our regularized training objective function.
\begin{definition}
Let $l$ be the cross-entropy loss.  Our regularized training objective function is 
\begin{align*}
\mathcal{L}_\lambda(\theta):=\frac{1}{|D_p|} \sum_{(x, y) \in D_p} l(f(\theta, x), y)+\lambda\|\theta\|_{2,k+1}.
\end{align*}
\end{definition}

\begin{definition}
We define $\Theta^*: = \arg \max _{\theta \in \Theta} h(\theta)$.
\end{definition}

Finally, let $\Omega := \mathbb{R}^{p \times (k+1)}$ denote the domain of each $\theta_i$, and let $\Omega'$ be a subset of $\Omega$. We say the parameter set $\theta = \{\theta_1, \ldots, \theta_m \}$ has directional support on $\Omega'$, if for every $i \in [m]$, either $\theta_i = 0$ or there exists $\alpha_i > 0$ such that $\alpha_i \theta_i \in \Omega'$.
\section{Tools from Previous Work}\label{app:tools}
 
Section~\ref{app:tools:implying_single_combined_neurons} states that we can use the single neuron level optimization to get the maximum-margin network. Section~\ref{app:tools:maximum_margin_for_multiclass} introduces the maximum-margin for multi-class. 

\subsection{Tools from Previous Work: Implying Single/Combined Neurons}\label{app:tools:implying_single_combined_neurons}

\begin{lemma}[Lemma 5 in page 8 in \cite{meo+23}]\label{lem:combine}
If the following conditions hold
\begin{itemize}
    \item Given $\Theta := \{\theta:\|\theta\|_{a, b} \leq 1 \}$.
    \item Given $\Theta_q^{\prime *} := \arg \max _{\theta \in \Theta} \E_{(x, y) \sim q} [g^{\prime}(\theta, x, y) ]$.
    \item Given $\Omega := \{\theta_i:\|\theta_i\|_a \leq 1 \}$. 
    \item Given $\Omega_q^{\prime *} :=\arg \max _{\theta_i \in \Omega} \E_{(x, y) \sim q} [\psi^{\prime}(\theta, x, y) ]$.
\end{itemize}
Then:

\begin{itemize}
    \item 
    Let $\theta \in \Theta_q^{\prime *}$. We have $\theta$ only has directional support on $\Omega_q^{\prime *}$.
    \item 
    Given $\theta_1^*, \ldots, \theta_m^* \in \Omega_q^{\prime *}$, we have for any set of neuron scalars where $\sum_{i=1}^m \alpha_i^\nu=1, \alpha_i \geq 0$,  the weights $\theta= \{\alpha_i \theta_i^* \}_{i=1}^m$ is in $\Theta_q^{\prime *}$.
\end{itemize}

\end{lemma}
Given $q^*$, then we can get the $\theta^*$ satisfying
\begin{align}\label{eq:definition_theta_star_prime}
     \theta^* \in \underset{\theta \in \Theta}{\arg \min } \underset{(x, y) \sim q^*}{\E}[g^{\prime}(\theta, x, y)].
\end{align}

\subsection{Tools from Previous Work: Maximum Margin for Multi-Class}\label{app:tools:maximum_margin_for_multiclass}

\begin{lemma}[Lemma 6 in page 8 in \cite{meo+23}]\label{lemma:multi-class}
If the following conditions hold
\begin{itemize}
\item Given $\Theta= \{\theta:\|\theta\|_{a, b} \leq 1 \}$ and $\Theta_q^{\prime *}=\arg \max _{\theta \in \Theta} \E_{(x, y) \sim q} [g^{\prime}(\theta, x, y) ]$. 
\item  Given $\Omega= \{\theta_i:\|\theta_i\|_a \leq 1 \}$ and $\Omega_q^{\prime *}=\arg \max _{\theta_i \in \Omega} \E_{(x, y) \sim q} [\psi^{\prime}(\theta, x, y) ]$.
\item Suppose that  $\exists \{\theta^*, q^* \}$ such that Equations~\eqref{eq:definition_q_star} and~\eqref{eq:definition_theta_star_prime}, and~\ref{lem:translation_property} holds.
\end{itemize}

Then, we can show:
\begin{itemize}
    \item $\theta^* \in \arg \max _{\theta \in \Theta} g(\theta, x, y)$
    \item $\forall \hat{\theta} \in \arg \max _{\theta \in \Theta} \min _{(x, y) \in D} g(\theta, x, y)$ the below properties hold:
    
    - $\hat{\theta}$ only has directional support on $\Omega_{q^*}^{\prime *}$.
    
    - $\forall (x, y) \in \operatorname{spt} (q^* ), f(\hat{\theta}, x, y)-\max _{y^{\prime} \in \mathcal{Y} \backslash\{y\}} f (\hat{\theta}, x, y^{\prime} )=\gamma^*$.
\end{itemize}

\end{lemma}

\begin{condition}[Condition C.1 in page 8 in \cite{meo+23}]

    We have $g^{\prime} (\theta^*, x, y )=g (\theta^*, x, y )$ for all $(x, y) \in \operatorname{spt} (q^* )$, where $\operatorname{spt}$ is the support. It means: 
    $
        \{y' \in \mathcal{Y} \backslash\{y\}: \tau(x, y)[y']>0\} \subseteq \underset{y' \in \mathcal{Y} \backslash\{y\}}{\arg \max } f (\theta^*, x )[y'].
    $
\end{condition}
\section{Class-weighted Max-margin Solution of Single Neuron}\label{app:single}

Section~\ref{app:single:definitions} introduces some definitions. Section~\ref{app:single:transfer_discrete_fourier_space} shows how we transfer the problem to discrete Fourier space. Section~\ref{app:single:get_solution_set} proposes the weighted margin of the single neuron. Section~\ref{app:single:transfer_discrete_fourier_space_k} shows how we transfer the problem to discrete Fourier space for general $k$ version. Section~\ref{app:single:get_solution_set_k} provides the solution set for general $k$ version and the maximum weighted margin for a single neuron. 

\subsection{Definitions}\label{app:single:definitions}

\begin{definition}\label{def:eta_3}
When $k=3$, let
\begin{align*}
    \eta_{u_1,u_2,u_3, w}(\delta) := \E_{a_1, a_2, a_3}[(u_1(a_1)+u_2(a_2)+u_3(a_3))^3 w(a_1+a_2+a_3-\delta)].
\end{align*}
\end{definition}

\begin{definition}\label{def:margin_max_problem}
Let $\eta$ be defined in Definition~\ref{def:eta_3}. When $k=3$, provided the following conditions are met
\begin{itemize}
    \item We denote ${\cal B}$ as  the ball that $\| u_1 \|^2  + \| u_2 \|^2 + \| u_3 \|^2 + \| w \|^2 \leq 1$.
\end{itemize}

We define 
\begin{align*}
     \Omega_q^{\prime *} = & \underset{u_1,u_2,u_3,w \in {\cal B} }{\arg \max } (\eta_{u_1,u_2,u_3, w}(0) - \E_{\delta \neq 0} [\eta_{u_1,u_2,u_3,w}(\delta)]). 
\end{align*}
\end{definition}

\subsection{Transfer to Discrete Fourier Space}\label{app:single:transfer_discrete_fourier_space}

The goal of this section is to prove the following Lemma,
\begin{lemma}\label{lemma:fourier_space}
When $k=3$, provided the following conditions are met
\begin{itemize}
    \item We denote ${\cal B}$ as the ball that $\| u_1 \|^2 + \|u_2 \|^2 + \| u_3 \|^2 + \| w \|^2 \leq 1$.
    \item We define $\Omega_q^{'*}$  in Definition~\ref{def:margin_max_problem}.
    \item We adopt the uniform class weighting: $\forall c' \neq a_1 + a_2 + a_3, ~~\tau (a_1, a_2, a_3)[c'] := 1/(p - 1)$.
\end{itemize}

We have the following 
    \begin{align*}
        \Omega_q^{\prime *} 
    = & \underset{u_1,u_2,u_3, ,w \in {\cal B} }{\arg \max } \frac{6}{(p-1) p^{3}} \sum_{j \neq 0} \hat{u}_1(j) \hat{u}_2(j) \hat{u}_3(j) \hat{w}(-j).
    \end{align*}
\end{lemma}

\begin{proof}
 
We have 
\begin{align*}
\eta_{u_1, u_2, u_3, w}(\delta)= & ~ \E_{a_1, a_2, a_3}[(u_1(a_1)+u_2(a_2)+u_3(a_3))^3 w(a_1+a_2+a_3-\delta)] \\
= & ~ \E_{a_1, a_2, a_3}[
(u_1(a_1)^3 + 3u_1(a_1)^2 u_2(a_2) + 3u_1(a_1)^2 u_3(a_3) + 3u_1(a_1) u_2(a_2)^2  \\
& ~ + 6u_1(a_1) u_2(a_2) u_3(a_3) + 3u_1(a_1) u_3(a_3)^2 + u_2(a_2)^3 + 3u_2(a_2)^2 u_3(a_3) \\
& ~ + 3u_2(a_2) u_3(a_3)^2 + u_3(a_3)^3) w(a_1 + a_2 + a_3 - \delta)
].
\end{align*}

Recall ${\cal B}$ is defined as Lemma Statement.

The goal is to solve the following mean margin maximization problem:
\begin{align}\label{eq:margin_max_problem}
    & ~ \underset{u_1,u_2,u_3, w \in {\cal B} }{\arg \max } (\eta_{u_1,u_2,u_3, w}(0) - \E_{\delta \neq 0} [\eta_{u_1,u_2,u_3, w}(\delta)]) \notag\\ 
    =& ~ \frac{p}{p-1} (\eta_{u_1,u_2,u_3, w}(0)-\E_\delta [\eta_{u_1,u_2,u_3, w}(\delta)]),
\end{align}
where the equation follows $\tau (a_1, a_2, a_3)[c'] := 1/(p - 1) ~~ \forall c' \neq a_1 + a_2 + a_3$ and $1-{1\over p-1} = \frac{p}{p-1}$.

First, note that  
\begin{align*}
& ~ \E_{a_1,a_2,a_3} [ u_1(a_1)^3 w(a_1+a_2+a_3 - \delta) ] \\
= & ~ \E_{a_1} [ u_1(a_1)^3 \E_{a_2,a_3} [ w(a_1+a_2+a_3 - \delta) ] ] \\
= & ~ 0,
\end{align*}
where the first step follows from taking out the $u_1(a_1)$ from the expectation for  $a_2,a_3$, and the last step is from the definition of $w$.

Similarly for the $u_2(a_2)^{3}$,$u_3(a_3)^{3}$ components of $\eta$, they equal to $0$.

Note that 
\begin{align*}
& ~ \E_{a_1,a_2,a_3} [ u_1(a_1)^2 u_2(a_2) w(a_1+a_2+a_3 - \delta) ] \\
= & ~ \E_{a_1} [ u_1(a_1)^2 \E_{a_2} [ u_2(a_2) \E_{a_3}[ w(a_1+a_2+a_3 - \delta) ] ] ] \\
= & ~ 0,
\end{align*}
where the first step follows from simple algebra and the last step comes from the definition of $w$. 

Similarly for the $ u_1(a_1)^2 u_3(a_3)$, $ u_2(a_2)^2 u_1(a_1)$, $ u_2(a_2)^2 u_3(a_3)$, $ u_3(a_3)^2 u_1(a_1)$, $ u_3(a_3)^2 u_2(a_2)$ components of $\eta$, they equal to $0$.

Hence, we can rewrite Eq.~\eqref{eq:margin_max_problem} as

\begin{align*}
   \underset{u_1,u_2,u_3, w \in {\cal B} }{\arg \max } \frac{6 p}{p-1} (\tilde{\eta}_{u_1,u_2,u_3, w} (0) - \E_{\delta}[\tilde{\eta}_{u_1,u_2,u_3, w} (\delta)]), 
\end{align*}

where

\begin{align*}
    \tilde{\eta}_{u_1,u_2,u_3, w}(\delta): = \E_{a_1,a_2,a_3} [ u_1(a_1) u_2(a_2) u_3(a_3) w(a_1+a_2+a_3 - \delta) ].
\end{align*}

Let $\rho:=e^{2 \pi \i / p}$, and let $\hat{u}_1, \hat{u}_2, \hat{u}_3, \hat{w}$ be the DFT of $u_1,u_2,u_3$, and $w$ respectively:
\begin{align*}
    & ~ \tilde{\eta}_{u_1,u_2,u_3, w} (\delta) \\ 
    = & ~\E_{a_1, a_2, a_3} [(\frac{1}{p} \sum_{j_1=0}^{p-1} \hat{u}_1(j_1) \rho^{j_1 a_1})(\frac{1}{p} \sum_{j_2=0}^{p-1} \hat{u}_2(j_2) \rho^{j_2 a_2})(\frac{1}{p} \sum_{j_3=0}^{p-1} \hat{u}_3(j_3) \rho^{j_3 a_3}) (\frac{1}{p} \sum_{j_4=0}^{p-1} \hat{w}(j_4) \rho^{j_4(a_1+a_2+a_3-\delta)})] \\
    = & ~ \frac{1}{p^{4}} \sum_{j_1, j_2, j_3, j_4} \hat{u}_1(j_1) \hat{u}_2(j_2) \hat{u}_3(j_3) \hat{w}(j_4) \rho^{-j_4 \delta} (\E_{a_1} [\rho^{(j_1+j_4) a_1}])(\E_{a_2}[ \rho^{(j_2+j_4) a_2}])( \E_{a_3}[\rho^{(j_3+j_4) a_3}]) \\
    = & ~ \frac{1}{p^{4}} \sum_{j} \hat{u}_1(j) \hat{u}_2(j) \hat{u}_3(j) \hat{w}(-j) \rho^{j \delta} \quad 
\end{align*}
where the first step follows from $\rho:=e^{2 \pi \i / p}$ and $\hat{u}_1, \hat{u}_2,\hat{u}_3, \hat{w}$ are the discrete Fourier transforms of $u_1,u_2,u_3,w$, the second step comes from simple algebra, the last step is from that only terms where $ j_1+j_4=j_2+j_4=j_3+j_4=0 $ survive.

Hence, we need to maximize

\begin{align}\label{eq:discrete_fourier_transforms_uvw}
& ~ \frac{6 p}{p-1}(\tilde{\eta}_{u_1,u_2,u_3, w}(0)-\E_{\delta}[\tilde{\eta}_{u_1,u_2,u_3, w}(\delta)]) \notag\\
= & ~ \frac{6 p}{p-1}(\frac{1}{p^{4}} \sum_{j} \hat{u}_1(j) \hat{u}_2(j) \hat{u}_3(j)\hat{w}(-j)-\frac{1}{p^{4}} \sum_{j} \hat{u}_1(j) \hat{u}_2(j) \hat{u}_3(j)\hat{w}(-j)(\E_{\delta} \rho^{j \delta})) \notag \\
= & ~ \frac{6}{(p-1) p^{3}} \sum_{j \neq 0} \hat{u}_1(j) \hat{u}_2(j)\hat{u}_3(j) \hat{w}(-j). \notag \\
= & ~ \frac{6}{(p-1) p^{3}} \sum_{j \in [ -(p-1)/2, +(p-1)/2 ] \backslash 0 } \hat{u}_1(j) \hat{u}_2(j)\hat{u}_3(j) \hat{w}(-j). 
\end{align}
where the first step is from $\tilde{\eta}_{u_1,u_2,u_3, w}(\delta)$ definition, the second step is from $\E_{\delta} \rho^{j \delta} = 0$ when $j\neq 0$, and the last step follows from simple algebra. 

\end{proof}

\subsection{Get Solution Set}\label{app:single:get_solution_set}

\begin{lemma}\label{lem:margin_soln}
When $k=3$, provided the following conditions are met
\begin{itemize}
    \item We denote ${\cal B}$ as the ball that $\| u_1 \|^2 + \|u_2 \|^2 + \| u_3 \|^2 + \| w \|^2 \leq 1$.
    \item We define $\Omega_q^{'*}$ in  Definition~\ref{def:margin_max_problem}.
    \item We adopt the uniform class weighting: $\forall c' \neq a_1 + a_2 + a_3, ~~\tau (a_1, a_2, a_3)[c'] := 1/(p - 1)$.
    \item For any $\zeta \in \{1, \ldots, \frac{p-1}{2}\}$, there exists a scaling constant $\beta \in \mathbb{R}$ and 
    \begin{align*}
    u_1(a_1) &= \beta \cdot \cos ( \theta_{u_1}^* + 2 \pi \zeta a_1 /p ) \\
    u_2(a_2) & =\beta \cdot \cos (\theta_{u_2}^*+2 \pi \zeta a_2 / p) \\
    u_3(a_3) & =\beta \cdot \cos (\theta_{u_3}^*+2 \pi \zeta a_3 / p) \\
    w(c) & =\beta \cdot \cos (\theta_w^*+2 \pi \zeta c / p)
\end{align*}
where $\theta_{u_1}^*, \theta_{u_2}^*, \theta_{u_3}^*, \theta_w^* \in \R$ are some phase offsets satisfying $\theta_{u_1}^*+\theta_{u_2}^*+\theta_{u_3}^*=\theta_w^*$.
\end{itemize}

Then, we have the following 
    \begin{align*}
        \Omega_q^{\prime *}  
    = & \{(u_1, u_2, u_3, w)\},
    \end{align*}
and 
\begin{align*}
    \underset{u_1,u_2,u_3,w \in {\cal B}}{\max } (\eta_{u_1,u_2,u_3, w}(0) - \E_{\delta \neq 0} [\eta_{u_1,u_2,u_3,w}(\delta)])= \frac{3}{16} \cdot \frac{1}{p(p-1)}.
\end{align*}

\end{lemma}
\begin{proof}
By Lemma~\ref{lemma:fourier_space}, we only need to maximize Equation~\eqref{eq:discrete_fourier_transforms_uvw}.

Thus, the mass of $\hat{u}_1, \hat{u}_2,  \hat{u}_3$, and $\hat{w}$ must be concentrated on the same frequencies. For all $j \in \mathbb{Z}_{p}$, we have
\begin{align}\label{eq:u_1_ov_u_2_ov_u_3_ov_w_ov}
    \hat{u}_1(-j)=\overline{\hat{u}_1(j)}, \hat{u}_2(-j)=\overline{\hat{u}_2(j)}, 
    \hat{u}_3(-j)=\overline{\hat{u}_3(j)}, \hat{w}(-j)=\overline{\hat{w}(j)}
\end{align}
 as  $u_1,u_2,u_3, w$ are real-valued.

For all $j \in \mathbb{Z}_{p}$ and for $u_1,u_2,u_3, w$, we denote $\theta_{u_1}, \theta_{u_2}, \theta_{u_3}, \theta_{w} \in[0,2 \pi)^{p}$ as  their phase, e.g.:
\begin{align*}
    \hat{u}_1(j)=|\hat{u}_1(j)| \exp (\i \theta_{u_1}(j)).
\end{align*}

Consider the odd $p$, Equation~\eqref{eq:discrete_fourier_transforms_uvw} becomes:

\begin{align*}
\eqref{eq:discrete_fourier_transforms_uvw} = & ~ \frac{6}{(p-1) p^{3}} \sum_{j \in [ -(p-1)/2, +(p-1)/2 ] \backslash 0 } \hat{u}_1(j) \hat{u}_2(j)\hat{u}_3(j) \hat{w}(-j) \\ 
= & ~ \frac{6}{(p-1) p^{3}} \sum_{j=1}^{(p-1) / 2}( \hat{u}_1(j) \hat{u}_2(j) \hat{u}_3(j) \overline{\hat{w}(j)} + \overline{\hat{u}_1(j) \hat{u}_2(j)\hat{u}_3(j)} \hat{w}(j) ) \\
= & ~ \frac{6}{(p-1) p^{3}} \sum_{j=1}^{(p-1) / 2}| \hat{u}_1(j) || \hat{u}_2(j) | | \hat{u}_3(j) || \hat{w}(j)| \cdot \\
& ~ \big( \exp (\i(\theta_{u_1}(j) + \theta_{u_2}(j) + \theta_{u_3}(j) - \theta_{w}(j))   + \exp (\i(-\theta_{u_1}(j) - \theta_{u_2}(j) -\theta_{u_3}(j) + \theta_{w}(j)) \big) \\
= & ~ \frac{12}{(p-1) p^{3}} \sum_{j=1}^{(p-1) / 2}| \hat{u}_1(j) || \hat{u}_2(j) | | \hat{u}_3(j) || \hat{w}(j)| \cos (\theta_{u_1}(j)+\theta_{u_2}(j) + \theta_{u_3}(j)-\theta_{w}(j)) . 
\end{align*}
where the first step comes from definition \eqref{eq:discrete_fourier_transforms_uvw}, the second step follows from Eq.~\eqref{eq:u_1_ov_u_2_ov_u_3_ov_w_ov}, the third step comes from $\hat{u}_1(-j)=\overline{\hat{u}_1(j)}$ and $\hat{u}_1(j)=|\hat{u}_1(j)| \exp (\i \theta_{u_1}(j))$, the last step follow from Euler's formula. 

Thus, we need to optimize: 
\begin{align}\label{eq:optimize_discrete_fourier_transforms_uvw}
    \max_{u_1,u_2,u_3,w \in {\cal B}} \frac{12}{(p-1) p^{3}} \sum_{j=1}^{(p-1) / 2}| \hat{u}_1(j) || \hat{u}_2(j) | | \hat{u}_3(j) || \hat{w}(j)| \cos (\theta_{u_1}(j)+\theta_{u_2}(j) + \theta_{u_3}(j)-\theta_{w}(j)) .
\end{align}

 The norm constraint  $\| u_1 \|^2 + \|u_2 \|^2 + \| u_3 \|^2 + \| w \|^2 \leq 1$ is equivalent to
\begin{align*}
    \|\hat{u}_1\|^{2}+\|\hat{u}_2\|^{2}+ \|\hat{u}_3\|^{2}+\|\hat{w}\|^{2} \leq p
\end{align*}
by using Plancherel's theorem.
Thus, we need to select them in such a way that
\begin{align*}
\theta_{u_1}(j)+\theta_{u_2}(j)+\theta_{u_3}(j)=\theta_{w}(j),
\end{align*}
ensuring that, for each $j$, the expression $\cos(\theta_{u_1}(j) + \theta_{u_2}(j) + \theta_{u_3}(j) - \theta_{w}(j)) = 1$ is maximized, except in cases where the scalar of the $j$-th term is $0$.

This further simplifies the problem to:
\begin{align}\label{eq:reduce_optimize_discrete_fourier_transforms_uvw}
    \max _{|\hat{u}_1| ,|\hat{u}_2| ,|\hat{u}_3| ,|\hat{w}| : \|\hat{u}_1\|^{2} + \|\hat{u}_2\|^{2} +\|\hat{u}_3\|^{2} + \|\hat{w}\|^{2} \leq p} \frac{12}{(p-1) p^{3}} \sum_{j=1}^{(p-1) / 2}|\hat{u}_1(j) || \hat{u}_2(j) || \hat{u}_3(j) || \hat{w}(j)|.
\end{align}

Then, we have
\begin{align}\label{eq:inequality_quadratic_geometric}
    |\hat{u}_1(j)||\hat{u}_2(j)||\hat{u}_3(j)||\hat{w}(j)| \leq(\frac{1}{4} \cdot ( |\hat{u}_1(j)|^{2} + |\hat{u}_2(j)|^{2}+ |\hat{u}_3(j)|^{2} + |\hat{w}(j)|^{2} ) )^{2}.
\end{align}
where the first step is from inequality of quadratic and geometric means.

We define $z:\{1, \ldots, \frac{p-1}{2}\} \rightarrow \R$ as 
\begin{align*}
z(j) := |\hat{u}_1(j)|^{2}+|\hat{u}_2(j)|^{2}+|\hat{u}_3(j)|^{2}+|\hat{w}(j)|^{2}.
\end{align*}

We need to have $\hat{u}_1(0)=\hat{u}_2(0)=\hat{u}_3(0)=\hat{w}(0)=0$. Then, the upper-bound of  Eq.~\eqref{eq:reduce_optimize_discrete_fourier_transforms_uvw} is given by
\begin{align*}
    & ~ \frac{12}{(p-1) p^{3}} \cdot \max _{\|z\|_{1} \leq \frac{p}{2}} \sum_{j=1}^{(p-1) / 2}(\frac{z(j)}{4})^{2} \\
    = & ~ \frac{3}{4(p-1) p^{3}} \cdot \max _{\|z\|_{1} \leq \frac{p}{2}} \sum_{j=1}^{(p-1) / 2} z(j)^2  \\
    = & ~ \frac{3}{4(p-1) p^{3}}\cdot \max _{\|z\|_{1} \leq \frac{p}{2}}\|z\|_{2}^2 \\
    \leq & ~ \frac{3}{4(p-1) p^{3}} \cdot \frac{p^2}{4} \\
    = & ~ \frac{3}{16} \cdot \frac{1}{p(p-1)},
\end{align*}
where the first step follows from simple algebra, the second step comes from the definition of $L_2$ norm, the third step follows from $\|z\|_{2} \leq \| z \|_1 \leq \frac{p}{2}$, the last step comes from simple algebra.  

For the inequality of quadratic and geometric means,  Eq.~\eqref{eq:inequality_quadratic_geometric} becomes equality when $|\hat{u}_1(j)|=|\hat{u}_2(j)|=|\hat{u}_3(j)|=|\hat{w}(j)|$. To achieve $\|z\|_{2}=\frac{p}{2}$,  all the mass must be placed on a single frequency. Hence, for some frequency $\zeta \in\{1, \ldots, \frac{p-1}{2}\}$, to achieve the upper bound, we have:

\begin{align}\label{eq:hat_u_same}
    |\hat{u}_1(j)| = |\hat{u}_2(j)| = |\hat{u}_3(j)| = |\hat{w}(j)| = \Big\{\begin{array}{ll}
    \sqrt{p / 8} & \text { if } j= \pm \zeta \\
    0 & \text { otherwise }
    \end{array} ,
\end{align}

In this case,  Eq.~\eqref{eq:reduce_optimize_discrete_fourier_transforms_uvw} matches the upper bound.

\begin{align*}
    \frac{12}{(p-1) p^{3}} \cdot(\frac{p}{8})^{2}=\frac{3}{16} \cdot \frac{1}{p(p-1)},
\end{align*}
where the first step is by simple algebra. Hence, the maximum-margin is $\frac{3}{16} \cdot \frac{1}{p(p-1)}$.

Let $\theta_{u_1}^{*} : =\theta_{u_1}(\zeta)$. Combining all the results, up to scaling, it is established that all neurons which maximize the expected class-weighted margin conform to the form:
\begin{align*}
    u_1(a_1) 
    = & ~ \frac{1}{p} \sum_{j=0}^{p-1} \hat{u}_1(j) \rho^{j a_1} \\
  = & ~ \frac{1}{p} \cdot (\hat{u}_1(\zeta) \rho^{\zeta a_1}+\hat{u}_1(-\zeta) \rho^{-\zeta a_1}) \\
  = & ~ \frac{1}{p} \cdot ( \sqrt{\frac{p}{8}} \exp (\i \theta_{u_1}^{*}) \rho^{\zeta a_1}+\sqrt{\frac{p}{8}} \exp (-\i \theta_{u_1}^{*}) \rho^{-\zeta a_1} ) \\
  = & ~ \sqrt{\frac{1}{2 p}} \cos (\theta_{u_1}^{*}+2 \pi \zeta a_1 / p),
\end{align*}
where the first step comes from the definition of $u_1(a)$, the second step and third step follow from Eq.~\eqref{eq:hat_u_same}, the last step follows from Euler’s formula. 

Similarly, 
\begin{align*}
    u_2(a_2) = & ~ \sqrt{\frac{1}{2 p}} \cos (\theta_{u_2}^{*} + 2 \pi \zeta a_2 / p) \\
    u_3(a_3) = & ~ \sqrt{\frac{1}{2 p}} \cos (\theta_{u_3}^{*} + 2 \pi \zeta a_3 / p) \\
    w(c) = & ~ \sqrt{\frac{1}{2 p}} \cos (\theta_{w}^{*} + 2 \pi \zeta c / p),
\end{align*}
for some phase offsets $\theta_{u_1}^{*}, \theta_{u_2}^{*}, \theta_{u_3}^{*}, \theta_{w}^{*} \in \R$ satisfying $\theta_{u_1}^{*}+\theta_{u_2}^{*}+\theta_{u_3}^{*}=\theta_{w}^{*}$ and some $\zeta \in \mathbb{Z}_{p} \backslash\{0\}$, where $u_1,u_2,u_3$, and $w$ shares the same $\zeta$.
    
\end{proof}

\subsection{Transfer to Discrete Fourier Space for General \texorpdfstring{$k$}{} Version}\label{app:single:transfer_discrete_fourier_space_k}

\begin{definition}\label{def:eta_k}
Let
\begin{align*}
    \eta_{u_1,\dots,u_k, w}(\delta) := \E_{a_1, \dots, a_k}[(u_1(a_1)+\dots+u_k(a_k))^k w(a_1+\dots+a_k-\delta)].
\end{align*}
\end{definition}

\begin{definition}\label{def:margin_max_problem-k}
Let $\eta$ be defined in Definition~\ref{def:eta_k}. Provided the following conditions are met
\begin{itemize}
    \item We denote ${\cal B}$ as the ball that $\| u_1 \|^2  + \dots + \| u_k \|^2 + \| w \|^2 \leq 1$.
\end{itemize}

We define 
\begin{align*}
     \Omega_q^{\prime *} = & \underset{u_1,\dots,u_k,w \in {\cal B} }{\arg \max } (\eta_{u_1,\dots,u_k, w}(0) - \E_{\delta \neq 0} [\eta_{u_1,\dots,u_k,w}(\delta)]). 
\end{align*}
\end{definition}

The goal of this section is to prove the following Lemma,
\begin{lemma}\label{lemma:fourier_space-k}
Provided the following conditions are met
\begin{itemize}
    \item Let ${\cal B}$ denote the ball that $\| u_1 \|^2 + \dots + \| u_k \|^2 + \| w \|^2 \leq 1$.
    \item We define $\Omega_q^{'*}$ in Definition~\ref{def:margin_max_problem-k}.
    \item We adopt the uniform class weighting: $\forall c' \neq a_1 + \dots + a_k, ~~\tau (a_1, \dots, a_k)[c'] := 1/(p - 1)$.
\end{itemize}

We have the following 
    \begin{align*}
        \Omega_q^{\prime *} 
    = & \underset{u_1,\dots,u_k,w \in {\cal B} }{\arg \max } \frac{k!}{(p-1) p^{k}} \sum_{j \neq 0} \hat{w}(-j) \prod_{i=1}^k \hat{u}_i(j).
    \end{align*}
\end{lemma}

\begin{proof}
 
We have 
\begin{align*}
\eta_{u_1, \dots, u_k, w}(\delta)= & ~ \E_{a_1, \dots, a_k}[(u_1(a_1)+\dots+u_k(a_k))^k w(a_1+\dots+a_k-\delta)].
\end{align*}

The goal is to solve the following mean margin maximization problem:
\begin{align}\label{eq:margin_max_problem-k}
    & ~ \underset{u_1,\dots,u_k, w \in {\cal B} }{\arg \max } (\eta_{u_1,\dots,u_k, w}(0) - \E_{\delta \neq 0} [\eta_{u_1,\dots,u_k, w}(\delta)]) \notag\\ 
    =& ~ \frac{p}{p-1} (\eta_{u_1,\dots,u_k, w}(0)-\E_\delta [\eta_{u_1,\dots,u_k, w}(\delta)]),
\end{align}
where the equation follows $\tau (a_1, \dots, a_k)[c'] := 1/(p - 1) ~~ \forall c' \neq a_1 + \dots + a_k$ and $1-{1\over p-1} = \frac{p}{p-1}$.
 
We note that all terms are zero rather than $w(\cdot)\cdot\prod_{i=1}^k u_i(a_i)$. 

Hence, we can rewrite Eq.~\eqref{eq:margin_max_problem-k} as

\begin{align*}
   \underset{u_1,\dots,u_k, w \in {\cal B} }{\arg \max } \frac{k! p}{p-1} (\tilde{\eta}_{u_1,\dots,u_k, w} (0) - \E_{\delta}[\tilde{\eta}_{u_1,\dots,u_k, w} (\delta)]), 
\end{align*}

where
\begin{align*}
    \tilde{\eta}_{u_1,\dots,u_k, w}(\delta): = \E_{a_1,\dots,a_k} [ w(a_1+\dots+a_k - \delta) \prod_{i=1}^k u_i(a_i)  ].
\end{align*}

Let $\rho:=e^{2 \pi \i / p}$, and $\hat{u}_1, \dots, \hat{u}_k, \hat{w}$ denote the discrete Fourier transforms of $u_1,\dots,u_k$, and $w$ respectively. We have
\begin{align*}
    & ~ \tilde{\eta}_{u_1,\dots,u_k, w} (\delta) = \frac{1}{p^{k+1}} \sum_{j=0}^{p-1} \hat{w}(-j) \rho^{j \delta} \prod_{i=1}^k \hat{u}_i(j)   
\end{align*}
which comes from $\rho:=e^{2 \pi \i / p}$ and $\hat{u}_1, \dots,\hat{u}_k, \hat{w}$ are the discrete Fourier transforms of $u_1,\dots,u_k,w$.

Hence, we need to maximize

\begin{align}\label{eq:discrete_fourier_transforms_uvw-k}
& ~ \frac{k! p}{p-1}(\tilde{\eta}_{u_1,\dots,u_k, w}(0)-\E_{\delta}[\tilde{\eta}_{u_1,\dots,u_k, w}(\delta)]) \notag\\
= & ~ \frac{k! p}{p-1} \cdot \left(\frac{1}{p^{k+1}} \sum_{j=0}^{p-1} \hat{w}(-j)\prod_{i=1}^k \hat{u}_i(j)-\frac{1}{p^{k+1}} \sum_{j=0}^{p-1} \hat{w}(-j)(\E_{\delta} [ \rho^{j \delta} ] )\prod_{i=1}^k \hat{u}_i(j)\right) \notag \\
= & ~ \frac{k!}{(p-1) p^{k}} \sum_{j \neq 0} \hat{w}(-j)\prod_{i=1}^k \hat{u}_i(j). \notag \\
= & ~ \frac{k!}{(p-1) p^{k}} \sum_{j \in [ -(p-1)/2, +(p-1)/2 ] \backslash 0 } \hat{w}(-j)\prod_{i=1}^k \hat{u}_i(j). 
\end{align}
where the first step follows from the definition of $\tilde{\eta}_{u_1,\dots,u_k, w}(\delta)$, the second step follows from $\E_{\delta} [ \rho^{j \delta} ] = 0$ when $j\neq 0$, the last step is from simple algebra. 

\end{proof}

\subsection{Get Solution Set for General \texorpdfstring{$k$}{} Version} \label{app:single:get_solution_set_k}

\begin{lemma}[Formal version of Lemma \ref{lem:margin_soln-k:informal}]\label{lem:margin_soln-k}
Provided the following conditions are met
\begin{itemize}
    \item We denote ${\cal B}$ as the ball that $\| u_1 \|^2 + \dots + \| u_k \|^2 + \| w \|^2 \leq 1$.
    \item Let $\Omega_q^{'*}$ be defined as Definition~\ref{def:margin_max_problem-k}.
    \item We adopt the uniform class weighting: $\forall c' \neq a_1 + \dots + a_k, ~~\tau (a_1, \dots, a_k)[c'] := 1/(p - 1)$.
    \item For any $\zeta \in \{1, \ldots, \frac{p-1}{2}\}$, there exists a scaling constant $\beta \in \mathbb{R}$ and 
    \begin{align*}
    u_1(a_1) &= \beta \cdot \cos ( \theta_{u_1}^* + 2 \pi \zeta a_1 /p ) \\
     u_2(a_2) &= \beta \cdot \cos ( \theta_{u_2}^* + 2 \pi \zeta a_2 /p ) \\
    & \dots \\
    u_k(a_k) & =\beta \cdot \cos (\theta_{u_k}^*+2 \pi \zeta a_k / p) \\
    w(c) & =\beta \cdot \cos (\theta_w^*+2 \pi \zeta c / p)
\end{align*}
where $\theta_{u_1}^*, \dots, \theta_{u_k}^*, \theta_w^* \in \R$ are some phase offsets satisfying  $\theta_{u_1}^*+\dots+\theta_{u_k}^*=\theta_w^*$.
\end{itemize}

Then, we have the following 
    \begin{align*}
        \Omega_q^{\prime *}  
    = & \{(u_1, \dots, u_k, w)\},
    \end{align*}
and 
\begin{align*}
    \underset{u_1,\dots,u_k,w \in {\cal B}}{\max } (\eta_{u_1,\dots,u_k, w}(0) - \E_{\delta \neq 0} [\eta_{u_1,\dots,u_k,w}(\delta)])= \frac{2(k!)}{(2k+2)^{(k+1)/2}(p-1) p^{(k-1)/2}} .
\end{align*}

\end{lemma}
\begin{proof}
By Lemma~\ref{lemma:fourier_space-k}, we only need to maximize Equation~\eqref{eq:discrete_fourier_transforms_uvw-k}.
Thus, the mass of $\hat{u}_1, \dots, \hat{u}_k$, and $\hat{w}$ must be concentrated on the same frequencies. For all $j \in \mathbb{Z}_{p}$, we have
\begin{align}\label{eq:u_1_ov_u_2_ov_u_3_ov_w_ov-k}
    \hat{u}_i(-j)=\overline{\hat{u}_i(j)},  ~~~  \hat{w}(-j)=\overline{\hat{w}(j)}
\end{align}
as $u_1,\dots,u_k, w$ are real-valued.
For all $j \in \mathbb{Z}_{p}$ and for $u_1,u_2,u_3, w$, we denote $\theta_{u_1}, \dots, \theta_{u_k}, \theta_{w} \in[0,2 \pi)^{p}$ as their phase, e.g.:
\begin{align}\label{eq:u_1_u_2_u_3}
    \hat{u}_1(j)=|\hat{u}_1(j)| \exp (\i \theta_{u_1}(j)).
\end{align}

Considering odd $p$, Equation~\eqref{eq:discrete_fourier_transforms_uvw-k} becomes:

\begin{align*}
\eqref{eq:discrete_fourier_transforms_uvw-k} = & ~ \frac{k!}{(p-1) p^{k}} \sum_{j \in [ -(p-1)/2, +(p-1)/2 ] \backslash 0 } \hat{w}(-j)\prod_{i=1}^k \hat{u}_i(j) \\ 
= & ~ \frac{k!}{(p-1) p^{k}} \sum_{j=1}^{(p-1) / 2} ( \prod_{i=1}^k \hat{u}_i(j) \overline{\hat{w}(j)} + \hat{w}(j) \prod_{i=1}^k\overline{\hat{u}_i(j)} ) \\
= & ~ \frac{2(k!)}{(p-1) p^{k}} \sum_{j=1}^{(p-1) / 2}| \hat{w}(j)| \cos (\sum_{i=1}^k \theta_{u_i}(j)-\theta_{w}(j)) \prod_{i=1}^k | \hat{u}_i(j) |. 
\end{align*}
where the first step follows from definition \eqref{eq:discrete_fourier_transforms_uvw-k}, the second step comes from Eq.~\eqref{eq:u_1_ov_u_2_ov_u_3_ov_w_ov-k}, the last step follows from Eq.~\eqref{eq:u_1_u_2_u_3}, i.e., Euler's formula. 

Thus, we need to optimize: 
\begin{align}\label{eq:optimize_discrete_fourier_transforms_uvw-k}
    \max_{u_1,\dots,u_k,w \in {\cal B}} \frac{2(k!)}{(p-1) p^{k}} \sum_{j=1}^{(p-1) / 2}| \hat{w}(j)| \cos (\sum_{i=1}^k \theta_{u_i}(j)-\theta_{w}(j)) \prod_{i=1}^k | \hat{u}_i(j) |.
\end{align}

We can transfer the norm constraint to
\begin{align*}
    \|\hat{u}_1 \|^{2}+\dots +  \|\hat{u}_k \|^{2}+\|\hat{w}\|^{2} \leq p
\end{align*}
by using Plancherel's theorem.

Therefore, we need to select them in a such way that $\theta_{u_1}(j)+\dots+\theta_{u_k}(j)=\theta_{w}(j)$, 
ensuring that, for each $j$, the expression $\cos (\theta_{u_1}(j)+\dots+\theta_{u_k}(j)-\theta_{w}(j))=1$ is maximized, except in cases where the scalar of the $j$-th term is $0$.

This further simplifies the problem to:
\begin{align}\label{eq:reduce_optimize_discrete_fourier_transforms_uvw-k}
    \max _{ \| \hat{u}_1 \|^{2}+\dots +  \|\hat{u}_k \|^{2}+\|\hat{w}\|^{2} \leq p} \frac{2(k!)}{(p-1) p^{k}} \sum_{j=1}^{(p-1) / 2}| \hat{w}(j)| \prod_{i=1}^k | \hat{u}_i(j) |.
\end{align}

Then, we have
\begin{align}\label{eq:inequality_quadratic_geometric-k}
    | \hat{w}(j)| \prod_{i=1}^k | \hat{u}_i(j) | \leq(\frac{1}{k+1} \cdot ( |\hat{u}_1(j)|^{2} + \dots + |\hat{u}_k(j)|^{2} + |\hat{w}(j)|^{2} ) )^{(k+1)/2}.
\end{align}
where the first step follows from inequality of quadratic and geometric means.

We define $z:\{1, \ldots, \frac{p-1}{2}\} \rightarrow \R$, where 
\begin{align*}
    z(j) := |\hat{u}_1 (j) |^{2}+\dots +|\hat{u}_k(j)|^{2}+|\hat{w}(j)|^{2}.
\end{align*}

We need to have $\hat{u}_1 (0)=\dots=\hat{u}_k(0)=\hat{w}(0)=0$. Then, the upper-bound of Equation~\eqref{eq:reduce_optimize_discrete_fourier_transforms_uvw-k} is given by
\begin{align*}
    & ~ \frac{2(k!)}{(p-1) p^{k}} \cdot \max _{\|z\|_{1} \leq \frac{p}{2}} \sum_{j=1}^{(p-1) / 2}(\frac{z(j)}{k+1})^{(k+1)/2} \\
    = & ~ \frac{2(k!)}{(k+1)^{(k+1)/2}(p-1) p^{k}} \cdot \max _{\|z\|_{1} \leq \frac{p}{2}} \sum_{j=1}^{(p-1) / 2} z(j)^{(k+1)/2} \\
    \leq & ~ \frac{2(k!)}{(k+1)^{(k+1)/2}(p-1) p^{k}} \cdot (p/2)^{(k+1)/2} \\
    = & ~ \frac{2(k!)}{(2k+2)^{(k+1)/2}(p-1) p^{(k-1)/2}}  ,
\end{align*}
where the first step follows from simple algebra, the second step comes from the definition of $L_2$ norm, the third step follows from $\|z\|_{2} \leq \| z \|_1 \leq \frac{p}{2}$, the last step follows from simple algebra.  

For the inequality of quadratic and geometric means,  Eq.~\eqref{eq:inequality_quadratic_geometric-k} becomes equality when $|\hat{u}_1(j)|=\dots=|\hat{u}_k (j)|=|\hat{w}(j)|$.
To achieve $\|z\|_{2}=\frac{p}{2}$,  all the mass must be placed on a single frequency. Hence, for some frequency $\zeta \in\{1, \ldots, \frac{p-1}{2}\}$, to achieve the upper bound, we have:

\begin{align}\label{eq:hat_u_same-k}
    |\hat{u}_1 (j)| = \dots = |\hat{u}_k(j)| = |\hat{w}(j)| = \Big\{ \begin{array}{ll}
    \sqrt{p \over 2(k+1)}, & \text { if } j= \pm \zeta; \\
    0, & \text { otherwise. }
    \end{array} 
\end{align}
In this case, Equation~\eqref{eq:reduce_optimize_discrete_fourier_transforms_uvw-k} matches the upper bound.
Hence, this is the maximum-margin.

Let $\theta_{u_1}^{*} : =\theta_{u_1}(\zeta)$. Combining all the results, up to scaling, it is established that all neurons which maximize the expected class-weighted margin conform to the form:
\begin{align*}
    u_1(a_1) 
    = & ~ \frac{1}{p} \sum_{j=0}^{p-1} \hat{u}_1(j) \rho^{j a_1} \\
= & ~ \frac{1}{p} \cdot (\hat{u}_1(\zeta) \rho^{\zeta a_1}+\hat{u}_1(-\zeta) \rho^{-\zeta a_1} ) \\
= & ~ \frac{1}{p} \cdot ( \sqrt{\frac{p}{2(k+1)}} \exp (\i \theta_{u_1}^{*}) \rho^{\zeta a_1}+\sqrt{\frac{p}{2(k+1)}} \exp (-\i \theta_{u_1}^{*}) \rho^{-\zeta a_1} ) \\
= & ~ \sqrt{\frac{2}{(k+1) p}} \cos (\theta_{u_1}^{*}+2 \pi \zeta a_1 / p),
\end{align*}
where the first step comes from the definition of $u_1(a)$, the second step and third step follow from Eq.~\eqref{eq:hat_u_same-k}, the last step follows from Eq.~\eqref{eq:u_1_u_2_u_3} i.e., Euler's formula. 

We have similar results for other neurons
where $\theta_{u_1}^{*}, \dots, \theta_{u_k}^{*}, \theta_{w}^{*} \in \R$ satisfying $\theta_{u_1}^{*}+\dots+\theta_{u_k}^{*}=\theta_{w}^{*}$ and some $\zeta \in \mathbb{Z}_{p} \backslash\{0\}$, where $u_1,\dots,u_k$, and $w$ shares the same $\zeta$.
    
\end{proof}
\section{Construct Max Margin Solution}\label{app:construct}

Section~\ref{app:construct:sum_to_product_identities} proposed the sum-to-product identities for $k$ inputs. Section~\ref{app:construct:constructions_for_theta} shows how we construct $\theta^*$ when $k=3$. Section~\ref{app:construct:constructions_for_theta_k} gives the constructions for $\theta^*$ for general $k$ version.

\subsection{Sum-to-product Identities}\label{app:construct:sum_to_product_identities}

\begin{lemma}[Sum-to-product Identities]\label{lem:sum2product}
If the following conditions hold
\begin{itemize}
    \item Let $a_1,\dots,a_k$ denote any $k$ real numbers
\end{itemize} 
We have
\begin{itemize}
\item {\bf Part 1.}
\begin{align*}
    2^2\cdot 2! \cdot a_1 a_2  & = (a_1+ a_2)^2 - (a_1- a_2)^2 - (-a_1+ a_2)^2+(-a_1-a_2)^2 
\end{align*}
\item {\bf Part 2.}
\begin{align*}
    2^3\cdot 3! \cdot a_1 a_2 a_3 & = (a_1+ a_2+ a_3)^3 - (a_1+ a_2- a_3)^3 - (a_1- a_2+ a_3)^3-(-a_1+ a_2+ a_3)^3 \\
    & + (a_1- a_2- a_3)^3 + (-a_1+ a_2- a_3)^3 + (-a_1- a_2+ a_3)^3-(-a_1- a_2- a_3)^3  
\end{align*}
\item {\bf Part 3.}
\begin{align*}
2^k \cdot k! \cdot  \prod_{i=1}^k a_i = \sum_{c \in \{-1,+1\}^k } (-1)^{(k-\sum_{i=1}^k c_i)/2} ( \sum_{j=1}^k c_j a_j)^k.
\end{align*}
\end{itemize}
\end{lemma}
\begin{proof}

{\bf Proof of Part 1.}
 
We define $A_1, A_2, A_3, A_4$ as follows
\begin{align*}
    A_1 :=  ~ (a_1+ a_2)^2, A_2 :=  ~ (a_1- a_2)^2, A_3 :=  ~ (-a_1+ a_2)^2,  A_4 :=  ~ (-a_1-a_2)^2, 
\end{align*}

For the first term, we have
\begin{align*}
    A_1 = & ~  a_1^2 + a_2^2  +  2a_1a_2.
\end{align*}

For the second term, we have
\begin{align*}
    A_2 = & ~  a_1^2 + a_2^2  -  2a_1a_2.
\end{align*}

For the third term, we have
\begin{align*}
    A_3 = & ~  a_1^2 + a_2^2  -  2a_1a_2.
\end{align*}

For the fourth term, we have
\begin{align*}
    A_4 = & ~  a_1^2 + a_2^2  +  2a_1a_2.
\end{align*}

Putting things together, we have
\begin{align*}
(a_1+ a_2)^2 - (a_1- a_2)^2 - (-a_1+ a_2)^2+(-a_1-a_2)^2 
= & ~A_1 - A_2 -A_3 +A_4\\
= & ~ 8 a_1 a_2  \\
= & ~  2^3 a_1 a_2  
\end{align*}
{\bf Proof of Part 2.}

We define $B_1, B_2, B_3, B_4, B_5, B_6, B_7, B_8$ as follows
\begin{align*}
    B_1 := & ~ (a_1+a_2+a_3)^3, \\
    B_2 := & ~ (a_1+a_2-a_3)^3, \\
    B_3 := & ~ (a_1-a_2+a_3)^3, \\
    B_4 := & ~ (-a_1+a_2+a_3)^3, \\
    B_5 := & ~ (a_1-a_2-a_3)^3, \\
    B_6 := & ~ (-a_1+a_2-a_3)^3, \\
    B_7 := & ~ (-a_1-a_2+a_3)^3, \\
    B_8 := & ~ (-a_1-a_2-a_3)^3, 
\end{align*}
For the first term, we have
\begin{align*}
    B_1 = & ~  a_1^3 + a_2^3  + a_3^3 + 3a_2a_3^2 + 3a_2a_1^2 + 3a_1a_3^2 + 3a_1a_2^2 + 3a_3a_1^2  + 3a_3a_2^2   + 6a_1a_2a_3.
\end{align*}

For the second term, we have
\begin{align*}
    B_2 = & ~  a_1^3 + a_2^3 - a_3^3 + 3a_2a_3^2 + 3a_2a_1^2 + 3a_1a_3^2 + 3a_1a_2^2 - 3a_3a_1^2  - 3a_3a_2^2   - 6a_1a_2a_3.
\end{align*}

For the third term, we have
\begin{align*}
    B_3 = & ~  a_1^3 - a_2^3  + a_3^3 - 3a_2a_3^2 - 3a_2a_1^2 + 3a_1a_3^2 + 3a_1a_2^2 + 3a_3a_1^2  + 3a_3a_2^2   - 6a_1a_2a_3.
\end{align*}

For the fourth term, we have
\begin{align*}
    B_4 = & ~ -a_1^3 + a_2^3  + a_3^3 + 3a_2a_3^2 + 3a_2a_1^2 - 3a_1a_3^2 - 3a_1a_2^2 + 3a_3a_1^2  + 3a_3a_2^2   - 6a_1a_2a_3.
\end{align*}

For the fifth term, we have
\begin{align*}
    B_5 = & ~ a_1^3 - a_2^3  - a_3^3 - 3a_2a_3^2 - 3a_2a_1^2 + 3a_1a_3^2 + 3a_1a_2^2 - 3a_3a_1^2  - 3a_3a_2^2   + 6a_1a_2a_3.
\end{align*}

For the sixth term, we have
\begin{align*}
    B_6 = & ~ -a_1^3 + a_2^3  - a_3^3 + 3a_2a_3^2 + 3a_2a_1^2 - 3a_1a_3^2 - 3a_1a_2^2 - 3a_3a_1^2  - 3a_3a_2^2   + 6a_1a_2a_3.
\end{align*}

For the seventh term, we have
\begin{align*}
    B_7 = & ~ -a_1^3 - a_2^3  + a_3^3 - 3a_2a_3^2 - 3a_2a_1^2 - 3a_1a_3^2 - 3a_1a_2^2 + 3a_3a_1^2  + 3a_3a_2^2   + 6a_1a_2a_3.
\end{align*}

For the eighth term, we have
\begin{align*}
    B_8 = & ~ -a_1^3 - a_2^3  - a_3^3 - 3a_2a_3^2 - 3a_2a_1^2 - 3a_1a_3^2 - 3a_1a_2^2 - 3a_3a_1^2  - 3a_3a_2^2   - 6a_1a_2a_3.
\end{align*}

Putting things together, we have
\begin{align*}
& (a_1+ a_2+ a_3)^3 - (a_1+ a_2- a_3)^3 - (a_1- a_2+ a_3)^3-(-a_1+ a_2+ a_3)^3 \\
    & + (a_1- a_2- a_3)^3 + (-a_1+ a_2- a_3)^3 + (-a_1- a_2+ a_3)^3-(-a_1- a_2- a_3)^3  \\
= & ~B_1 - B_2 -B_3 -B_4+ B_5 + B_6 +B_7 -B_8\\
= & ~ 48 a_1 a_2 a_3 \\
= & ~ 3 \cdot 2^4 a_1 a_2 a_3 
\end{align*}

{\bf Proof of Part 3.} 

\begin{align*}
2^k \cdot k! \cdot  \prod_{i=1}^k a_i = \sum_{c \in \{-1,+1\}^k } (-1)^{(k-\sum_{i=1}^k c_i)/2} ( \sum_{j=1}^k c_j a_j)^k.
\end{align*}
We first let $a_1 = 0$. Then each term on RHS can find a corresponding negative copy of this term.  In detail, let $c_1$ change sign and we have, $(-1)^{(k-c_1-\sum_{i=2}^k c_i)/2} ( c_1 \cdot 0+\sum_{j=2}^k c_j a_j)^k = - (-1)^{(k+c_1-\sum_{i=2}^k c_i)/2} ( -c_1 \cdot 0+\sum_{j=2}^k c_j a_j)^k$. We can find this mapping is always one-to-one and onto mapping with each other. Thus, we have RHS
is constant $0$ regardless of $a_2, \dots, a_k$. Thus, $a_1$ is a factor of RHS. By symmetry, $a_2, \dots, a_k$ also are factors of RHS. Since RHS is $k$-th order, we have RHS$=\alpha \prod_{i=1}^k a_i$ where $\alpha$ is a constant. Take $a_1=\dots=a_k=1$, we have $\alpha=2^k \cdot k! =$RHS. Thus, we finish the proof. 
\end{proof}

\subsection{Constructions for \texorpdfstring{$\theta^*$}{}}\label{app:construct:constructions_for_theta}

\begin{lemma}\label{lem:construct}
When $k=3$, provided the following conditions are met
\begin{itemize}
    \item We denote ${\cal B}$ as the ball that $\| u_1 \|^2 + \|u_2 \|^2 + \| u_3 \|^2 + \| w \|^2 \leq 1$.
    \item We define $\Omega_q^{'*}$ in Definition~\ref{def:margin_max_problem}.
    \item We adopt the uniform class weighting: $\forall c' \neq a_1 + a_2 + a_3, ~~\tau (a_1, a_2, a_3)[c'] := 1/(p - 1)$.
    \item Let $\cos_\zeta(x)$ denote $\cos (2 \pi \zeta x / p)$
    \item Let $\sin _\zeta(x)$ denote $\sin (2 \pi \zeta x / p)$
\end{itemize}
Then, we have 
\begin{itemize}
\item 
The maximum $L_{2,4}$-margin solution $\theta^*$ will consist of $16(p-1)$ neurons $\theta^*_i \in \Omega_q^{'*}$ to simulate $\frac{p-1}{2}$ type of cosine computation, each cosine computation is uniquely determined a $\zeta \in \{1, \ldots, \frac{p-1}{2}\}$. In particular, for each $\zeta$ the cosine computation is   $\cos_\zeta(a_1+a_2+a_3-c), \forall a_1,a_2,a_3,c \in \mathbb{Z}_p$.
\end{itemize}
\end{lemma}

\begin{proof}

Referencing Lemma~\ref{lem:margin_soln}, we can identify elements within $\Omega_q^{'}$. Our set $\theta^*$ will be composed of $16(p-1)$ neurons, including $32$ neurons dedicated to each frequency in the range $1, \ldots, \frac{p-1}{2}$. Focusing on a specific frequency $\zeta$, for the sake of simplicity, let us use $\cos_\zeta(x)$ to represent $\cos (2 \pi \zeta x / p)$ and $\sin _\zeta(x)$ likewise. We note:

\begin{align}\label{eq:cos_a_c}
    \cos_\zeta(a_1+a_2+a_3-c)= & \cos_\zeta(a_1+a_2+a_3) \cos_\zeta(c) + \sin_\zeta(a_1+a_2+a_3) \sin_\zeta(c)  \notag\\
    = & ~ \cos_\zeta(a_1+a_2) \cos_\zeta(a_3) \cos_\zeta(c) - \sin_\zeta(a_1+a_2) \sin_\zeta(a_3) \cos_\zeta(c) \notag\\
    & +\sin_\zeta(a_1+a_2) \cos_\zeta(a_3) \sin_\zeta(c) + \cos_\zeta(a_1+a_2) \sin_\zeta(a_3) \sin_\zeta(c)  \notag\\ 
    = &  ~ (\cos_\zeta(a_1) \cos_\zeta(a_2) - \sin_\zeta(a_1) \sin_\zeta(a_2)
)\cos_\zeta(a_3) \cos_\zeta(c) \notag\\
 & ~ - (\sin_\zeta(a_1) \cos_\zeta(a_2) + \cos_\zeta(a_1) \sin_\zeta(a_2)
)\sin_\zeta(a_3) \cos_\zeta(c)  \notag\\
    &  ~ + (\sin_\zeta(a_1) \cos_\zeta(a_2) + \cos_\zeta(a_1) \sin_\zeta(a_2)
) \cos_\zeta(a_3) \sin_\zeta(c) \notag\\
   & ~ +  ((\cos_\zeta(a_1) \cos_\zeta(a_2) - \sin_\zeta(a_1) \sin_\zeta(a_2)
))\sin_\zeta(a_3) \sin_\zeta(c)  \notag\\
    = &  ~ \cos_\zeta(a_1) \cos_\zeta(a_2) \cos_\zeta(a_3) \cos_\zeta(c) - \sin_\zeta(a_1) \sin_\zeta(a_2)
\cos_\zeta(a_3) \cos_\zeta(c)   \notag\\
 & ~ - \sin_\zeta(a_1) \cos_\zeta(a_2)\sin_\zeta(a_3) \cos_\zeta(c)  - \cos_\zeta(a_1) \sin_\zeta(a_2)\sin_\zeta(a_3) \cos_\zeta(c) 
 \notag\\
    &  ~ + \sin_\zeta(a_1) \cos_\zeta(a_2) \cos_\zeta(a_3) \sin_\zeta(c)+ \cos_\zeta(a_1) \sin_\zeta(a_2)\cos_\zeta(a_3) \sin_\zeta(c)   \notag\\
   & ~ +  \cos_\zeta(a_1) \cos_\zeta(a_2)\sin_\zeta(a_3) \sin_\zeta(c) - \sin_\zeta(a_1) \sin_\zeta(a_2)\sin_\zeta(a_3) \sin_\zeta(c) 
\end{align}
where all steps comes from trigonometric function.

Each of these $8$ terms can be implemented by $4$ neurons $\phi_1, \phi_2, \cdots,  \phi_{4}$.
Consider the first term, $\cos_\zeta(a_1) \cos_\zeta(a_2) \cos_\zeta(a_3) \cos_\zeta(c) $. 

For the $i$-th neuron, we have 
\begin{align*}
\phi_i  =  (u_{i,1} (a_1)+u_{i,2} (a_2)+u_{i,3} (a_3))^3 \cdot w_i(c).
\end{align*}

By changing $(\theta_{i,j})^{*}$, we can change the constant factor of $\cos_\zeta(\cdot)$ to be $+\beta$ or $-\beta$. Hence, we can view $u_{i,j}(\cdot),w_{i}(\cdot)$ as the following:
\begin{align*}
    u_{i,1} (\cdot):= & ~ p_{i,1}    \cdot \cos_\zeta(\cdot), \\
    u_{i,2}(\cdot):=  & ~p_{i,2}   \cdot \cos_\zeta(\cdot), \\
    u_{i,3}(\cdot):=  & ~p_{i,3}    \cdot  \cos_\zeta(\cdot), \\
    w_{i}(\cdot):=  & ~p_{i,4}   \cdot \cos_\zeta(\cdot)
\end{align*}
where $p_{i,j} \in \{ -1,1\}$.

For simplicity, let $d_i$ denote $\cos_\zeta(a_i)$.

We set $(\theta_{u_1}^*, \theta_{u_2}^*, \theta_{u_3}^*, \theta_w^*) = (0,0,0,0)$, then
\begin{align*}
p_{1,1}, p_{1,2}, p_{1,3}, p_{1,4}  = 1,
\end{align*}
then we have 
\begin{align*}
    \phi_1 =   (d_1 + d_2 + d_3 )^3 \cos_\zeta(c).
\end{align*}

We set $(\theta_{u_1}^*, \theta_{u_2}^*, \theta_{u_3}^*, \theta_w^*) = (0,0,\pi,\pi)$, then $p_{2,1}, p_{2,2}  = 1$ and $p_{2,3},p_{2,4} = -1$, then we have 
\begin{align*}
\phi_2 = -  (d_1 + d_2 - d_3 )^3 \cos_\zeta(c).
\end{align*}

We set $(\theta_{u_1}^*, \theta_{u_2}^*, \theta_{u_3}^*, \theta_w^*) = (0,\pi,0,\pi)$, then  $p_{3,1}, p_{3,3}  = 1$ and $p_{3,2},p_{3,4} = -1$, then we have 
\begin{align*}
\phi_3 = -  (d_1 - d_2 + d_3 )^3 \cos_\zeta(c).
\end{align*}

We set $(\theta_{u_1}^*, \theta_{u_2}^*, \theta_{u_3}^*, \theta_w^*) = (\pi,0,0,\pi)$, then  $p_{4,1}, p_{4,4}  = -1$ and $p_{2,2},p_{2,3} = 1$, then we have 
\begin{align*}
\phi_4 = -  (-d_1 + d_2 + d_3 )^3 \cos_\zeta(c).
\end{align*}

Putting them together, we have
\begin{align}
& ~ \sum_{i=1}^{4} \phi_i(a_1,a_2,a_3) \notag \\ 
= & ~ \sum_{i=1}^{4}(u_{i,1} (a_1)+u_{i,2} (a_2)+u_{i,3} (a_3))^3 w_i (c)  \notag\\
= & ~ \sum_{i=1}^{4}(p_{i,1} \cos_\zeta(a_1)+p_{i,2} \cos_\zeta(a_2)+p_{i,3} \cos_\zeta(a_3))^3 w_i(c)  \notag\\
 = & ~   [(d_1+d_2+d_3)^3 - (d_1+d_2-d_3)^3 - (d_1-d_2+d_3)^3-(-d_1+d_2+d_3)^3]\cos_\zeta(c) \notag\\
 = & ~ 24 d_1 d_2 d_3 \cos_\zeta(c) \notag \\
 = & ~ 24 \cos_\zeta(a_1) \cos_\zeta(a_2) \cos_\zeta(a_3) \cos_\zeta(c)
\end{align}
where the first step comes from the definition of $\phi_i$, the second step comes from the definition of $u_{i,j}$, the third step comes from $d_i = \cos_\zeta (a_i)$, the fourth step comes from simple algebra, the last step comes from $d_i = \cos_\zeta (a_i)$. 

Similarly, consider $- \sin_\zeta(a_1) \sin_\zeta(a_2)
\cos_\zeta(a_3) \cos_\zeta(c) $.

We set $(\theta_{u_1}^*, \theta_{u_2}^*, \theta_{u_3}^*, \theta_w^*) = (\pi/2,\pi/2,0,\pi)$, then we have 
\begin{align*}
\phi_1 = -  (\sin_\zeta(a_1) +\sin_\zeta(a_2)+\cos_\zeta(a_3))^3 \cos_\zeta(c).
\end{align*}
We set $(\theta_{u_1}^*, \theta_{u_2}^*, \theta_{u_3}^*, \theta_w^*) = (\pi/2,\pi/2,-\pi,0)$, then we have 
\begin{align*}
\phi_2 =  (\sin_\zeta(a_1) +\sin_\zeta(a_2)-\cos_\zeta(a_3))^3 \cos_\zeta(c).
\end{align*}
We set $(\theta_{u_1}^*, \theta_{u_2}^*, \theta_{u_3}^*, \theta_w^*) = (\pi/2,-\pi/2,0,0)$, then we have 
\begin{align*}
\phi_3 =  (\sin_\zeta(a_1) -\sin_\zeta(a_2)+\cos_\zeta(a_3))^3 \cos_\zeta(c).
\end{align*}
We set $(\theta_{u_1}^*, \theta_{u_2}^*, \theta_{u_3}^*, \theta_w^*) = (-\pi/2,\pi/2,0,0)$, then we have 
\begin{align*}
\phi_4 =  (-\sin_\zeta(a_1) +\sin_\zeta(a_2)+\cos_\zeta(a_3))^3 \cos_\zeta(c).
\end{align*}
Putting them together, we have
\begin{align}
& ~ \sum_{i=1}^{4} \phi_i(a_1,a_2,a_3) \notag \\ 
= & ~ -24 \sin_\zeta(a_1) \sin_\zeta(a_2)
\cos_\zeta(a_3) \cos_\zeta(c) 
\end{align}
Similarly, all other six terms in Eq.~\eqref{eq:cos_a_c} can be composed by four neurons with different $(\theta_{u_1}^*, \theta_{u_2}^*, \theta_{u_3}^*, \theta_w^*)$.

When we include such $56$ neurons for all frequencies $\zeta \in \{1, \ldots, \frac{p-1}{2} \}$, we have that the network will calculate the following function

\begin{align*}
f(a_1, a_2, a_3, c  ) & =\sum_{\zeta=1}^{(p-1) / 2} \cos _\zeta(a_1+a_2+a_3-c) \\
& =\sum_{\zeta=1}^{p-1} \frac{1}{2} \cdot \exp (2 \pi \i \zeta(a_1+a_2+a_3-c) / p) \\
& = \begin{cases}\frac{p-1}{2} & \text { if } a_1+a_2+a_3=c \\
0 & \text { otherwise }\end{cases}
\end{align*}

where the first step comes from the definition of $f(a_1, a_2, a_3, c )$, the second step comes from Euler’s formula, the last step comes from the properties of discrete Fourier transform.

The scaling factor $\beta$ for each neuron can be selected such that the entire network maintains an $L_{2,4}$-norm of 1. In this setup, every data point lies exactly on the margin, meaning $q=\operatorname{unif} (\mathbb{Z}_p )$ uniformly covers points on the margin, thus meeting the criteria for $q^*$ as outlined in Definition \ref{def:q_star_theta_star}. Furthermore, for any input $(a_1, a_2, a_3)$, the function $f$ yields an identical result across all incorrect labels $c^{\prime}$, adhering to Condition \ref{lem:translation_property}.
\end{proof}

\subsection{Constructions for \texorpdfstring{$\theta^*$}{} for General \texorpdfstring{$k$}{} Version}\label{app:construct:constructions_for_theta_k}

\begin{lemma}[Formal version of Lemma \ref{lem:construct-k:informal}]\label{lem:construct-k}
Provided the following conditions are met
\begin{itemize}
    \item We denote ${\cal B}$ as the ball that $\| u_1 \|^2 + \dots +  \| u_k \|^2 + \| w \|^2 \leq 1$.
    \item We define $\Omega_q^{'*}$ in Definition~\ref{def:margin_max_problem}.
    \item We adopt the uniform class weighting: $\forall c' \neq a_1 + \dots + a_k, ~~\tau (a_1, \dots, a_k)[c'] := 1/(p - 1)$.
    \item Let $\cos_\zeta(x)$ denote $\cos (2 \pi \zeta x / p)$
    \item Let $\sin _\zeta(x)$ denote $\sin (2 \pi \zeta x / p)$
\end{itemize}
Then, we have 
\begin{itemize}
\item 
The maximum $L_{2,k+1}$-margin solution $\theta^*$ will consist of $2^{2k-1} \cdot {p-1 \over 2}$ neurons $\theta^*_i \in \Omega_q^{'*}$ to simulate $\frac{p-1}{2}$ type of cosine computation, each cosine computation is uniquely determined a $\zeta \in \{1, \ldots, \frac{p-1}{2}\}$. In particular, for each $\zeta$ the cosine computation is   $\cos_\zeta(a_1+\dots+a_k-c), \forall a_1,\dots,a_k,c \in \mathbb{Z}_p$.
\end{itemize}
\end{lemma}

\begin{proof}

By Lemma~\ref{lem:margin_soln}, we can get elements of $\Omega_q^{'*}$. 
Our set $\theta^*$ will be composed of $2^{2k-1} \cdot {p-1 \over 2}$ neurons, including $2^{2k-1} $ neurons dedicated to each frequency in the range $1, \ldots, \frac{p-1}{2}$. Focusing on a specific frequency $\zeta$, for the sake of simplicity, let us use $\cos_\zeta(x)$ to represent $\cos (2 \pi \zeta x / p)$ and $\sin _\zeta(x)$ likewise.

We define
\begin{align*}
    a_{[k]} := \sum_{i=1}^k a_k
\end{align*}
and we also define
\begin{align*}
    a_{k+1} := - c.
\end{align*}

For easy of writing, we will write $\cos_{\zeta}$ as $\cos$ and $\sin_{\zeta}$ as $\sin$. We have the following. 
\begin{align}
    & ~ \cos_\zeta(\sum_{i = 1}^k a_i - c ) \notag \\
    = & ~ \cos(\sum_{i = 1}^k a_i - c ) \notag \\
    = & ~ \cos ( a_{[k+1]} ) \notag \\
    = & ~ \cos( a_{[k]} ) \cos(a_{k+1}) - \sin( a_{[k]} ) \sin( a_{k+1} ) \notag \\  
    = & ~ \cos( a_{[k-1]} + a_{k} ) \cos (a_{k+1}) - \sin( a_{[k-1]} + a_k ) \sin(a_{k+1}) \notag \\
    = & ~  \cos ( a_{[k-1]} ) \cos(a_k) \cos(a_{k+1}) - \sin (a_{[k-1]}) \sin(a_k) \cos(a_{k+1}) \notag \\
    & ~ -  \sin (a_{[k-1]}) \cos (a_k) \sin (a_{k+1}) - \cos(a_{[k-1]}) \sin(a_k) \sin( a_{k+1}) \notag \\
    = & ~ \sum_{b \in \{0,1\}^{k+1}} \prod_{i=1}^{k+1} \cos^{1-b_i}(a_i) \cdot \sin^{b_i}(a_i) \cdot {\bf 1}[ \sum_{i=1}^{k+1} b_i \%2 = 0 ] \cdot (-1)^{ {\bf 1 }[ \sum_{i=1}^{k+1} b_i \% 4 = 2 ] }, \label{eq:k_cos_sum}
\end{align} 
where the first step comes from the simplicity of writing, the second step comes from the definition of $a_{[k+1]}$ and $a_{k+1}$, the third step comes from the trigonometric function, the fourth step also follows trigonometric function, and the last step comes from the below two observations:

\begin{itemize}
    \item First, we observe  that 
$\cos(a+b) = \cos(a)\cos(b) - \sin(a)\sin(b)$ and $\sin(a+b) = \sin(a)\cos(b) + \cos(a)\sin(b)$. When we split $\cos$ once, we will remove one $\cos$ product and we may add zero or two $\sin$ products. When we split $\sin$ once, we may remove one $\sin$ product and we will add one $\sin$ product as well. Thus, we can observe that the number of $\sin$ products in each term is always even. 
\item 
Second, we observe only when we split $\cos$ and add two $\sin$ products will introduce a $-1$ is this term. Thus, when the number of $\sin$ products $\%4 = 2$, the sign of this term will be $-1$. Otherwise, it will be $+1$. 
\end{itemize}

Note that we have $2^k$ non-zero term in Eq.~\eqref{eq:k_cos_sum}. Each of these $2^k$ terms can be implemented by $2^{k-1}$ neurons $\phi_1,  \cdots,  \phi_{2^{k-1}}$.

For the $i$-th neuron, we have  
\begin{align*}
\phi_i  = ( \sum_{j=1}^k u_{i,j} (a_j) )^k \cdot w_i(c).
\end{align*}

By changing $(\theta_{i,j})^{*}$, we can change the $u_{i,j} (a_j)$ from $\cos_\zeta(\cdot)$ to be $-\cos_\zeta(\cdot)$ or $\sin_\zeta(\cdot)$ or $-\sin_\zeta(\cdot)$. Denote $\theta_{u_i}^*$ as $(\theta_{i,a_i})^{*}$.

For simplicity, let $d_i$ denote the $i$-th product in one term of Eq.~\eqref{eq:k_cos_sum}. By fact that
\begin{align*}
2^k \cdot k! \cdot  \prod_{i=1}^k d_i = \sum_{c \in \{-1,+1\}^k } (-1)^{(k-\sum_{i=1}^k c_i)/2} ( \sum_{j=1}^k c_j d_j)^k,
\end{align*}
each term can be constructed by $2^{k-1}$ 
neurons (note that there is a symmetric effect so we only need half terms). Based on Eq.~\eqref{eq:k_cos_sum} and the above fact with carefully check, we can see that $\theta_{u_1}^* + \dots + \theta_{u_k}^* = \theta_w^*$. Thus, we need $2^k \cdot 2^{k-1} \cdot {p-1 \over 2}$ neurons in total. 

When we include such $2^k \cdot 2^{k-1}$ neurons for all frequencies $\zeta \in \{1, \ldots, \frac{p-1}{2} \}$, we have the network will calculate the following function
\begin{align*}
f(a_1, \dots, a_k, c  ) & =\sum_{\zeta=1}^{(p-1) / 2} \cos _\zeta(\sum_{i=1}^{k}a_i-c) \\
& =\sum_{\zeta=1}^{p-1} \frac{1}{2} \cdot \exp (2 \pi \i \zeta(\sum_{i=1}^{k}a_i-c) / p) \\
& = \begin{cases}\frac{p-1}{2} & \text { if } \sum_{i=1}^{k}a_i=c \\
0 & \text { otherwise }\end{cases}
\end{align*}

where the first step comes from the definition of $f(a_1, \dots, a_k, c  ) $, the second step comes from Euler’s formula, the last step comes from the properties of discrete Fourier transform.

The scaling parameter $\beta$ for each neuron can be adjusted to ensure that the network possesses an $L_{2,k+1}$-norm of 1. For this network, all data points are positioned on the margin, which implies that $q=\operatorname{unif} (\mathbb{Z}_p )$ naturally supports points along the margin, aligning with the requirements for $q^*$ presented in Definition \ref{def:q_star_theta_star}. Additionally, for every input $(a_1, \dots, a_k)$, the function $f$ assigns the same outcome to all incorrect labels $c^{\prime}$, thereby fulfilling Condition \ref{lem:translation_property}.
\end{proof}
\section{Check Fourier Frequencies}\label{app:frequency}

Section~\ref{app:frequency:all_frequencies} proves all frequencies are used. Section~\ref{app:frequency:all_frequencies_k} proves all frequencies are used for general $k$ version.

\subsection{All Frequencies are Used}\label{app:frequency:all_frequencies}
Let $f: \mathbb{Z}_p^4 \rightarrow \mathbb{C}$. Its multi-dimensional discrete Fourier transform is defined as:

\begin{align*}
    & ~ \hat{f}(j_1, j_2, j_3,j_4) \\
    := & ~ \sum_{a_1 \in \mathbb{Z}_p} e^{-2 \pi \i \cdot j_1 a_1 / p} ( \sum_{a_2 \in \mathbb{Z}_p} e^{-2 \pi \i \cdot j_2 a_2 / p} (\sum_{a_3 \in \mathbb{Z}_p} e^{-2 \pi \i \cdot j_3 a_3 / p} (\sum_{c \in \mathbb{Z}_p} e^{-2 \pi \i \cdot j_4 c / p} f(a_1, a_2, a_3, c))).
\end{align*}

\begin{lemma}\label{lem:frequency}
When $k=3$, if the following conditions hold
\begin{itemize}
    \item We adopt the uniform class weighting: $\forall c' \neq a_1 + a_2 + a_3, ~~\tau (a_1, a_2, a_3)[c'] := 1/(p - 1)$.
    \item $f$ is the maximum $L_{2,4}$-margin solution.
\end{itemize}
Then, for any $j_1=j_2=j_3=-j_4 \neq 0$, we have $\hat{f}(j_1, j_2, j_3,j_4)>0$.
\end{lemma}

\begin{proof}
In this proof, 
let $j_1,j_2,j_3,j_4 \in \mathbb{Z}$, and $\theta_u = \theta_u^* \cdot \frac{p}{2 \pi}$ to simplify the notation. By Lemma~\ref{lem:margin_soln}, 

\begin{align}\label{eq:u_1_theta}
    u_1(a_1)=\sqrt{\frac{1}{2 p}} \cos _p (\theta_{u_1} + \zeta a_1) .
\end{align}

Let
\begin{align*}
& f(a_1, a_2, a_3,c) \\
= & ~ \sum_{h=1}^H \phi_h(a_1, a_2, a_3,c) \\
= & ~ \sum_{h=1}^H (u_{h,1}(a_1) + u_{h,2}(a_2)+u_{h,3}(a_3))^3 w_h(c) \\
= & ~ (\frac{1}{2 p})^{2} \sum_{h=1}^H (\cos _p (\theta_{u_{h,1}}+\zeta_h a_1) + \cos _p (\theta_{u_{h,2}} + \zeta_h a_2) + \cos _p (\theta_{u_{h,3}} + \zeta_h a_3))^3 \cos_p (\theta_{w_h}+\zeta_h c)
\end{align*}
where each neuron conforms to the previously established form, and the width $H$ function is an arbitrary margin-maximizing network. The first step is from the definition of $f(a_1, a_2, a_3, c)$, the subsequent step on the definition of $\phi_h(a_1, a_2, a_3, c)$, and the final step is justified by Eq.~\eqref{eq:u_1_theta}.

We can divide each $\phi$ into ten terms:
\begin{align*}
    & ~ \phi(a_1, a_2, a_3, c) \\
    = & ~ \phi^{(1)}(a_1, a_2, a_3, c) + \dots + \phi^{(10)}(a_1, a_2, a_3, c) \\
    = & ~ \big(u_1(a_1)^3 +u_2(a_2)^3 + u_3(a_3)^3  + 3 u_1(a_1)^2 u_2(a_2) +  3 u_1(a_1)^2 u_3(a_3) + 3 u_2(a_2)^2 u_1(a_1) \\
    &~ + 3 u_2(a_2)^2 u_3(a_3) + 3 u_3(a_3)^2 u_1(a_1) + 3 u_3(a_3)^2 u_2(a_2) +  6 u_1(a_1) u_2(a_2)u_3(a_3)\big) w(c).
\end{align*}

Note, $\rho = e^{2\pi \i/p}$. 
$\widehat{\phi}_1(j_1, j_2, j_3,j_4)$ is nonzero only for $j_1=0$, and $\widehat{\phi}_4(j_1, j_2, j_3,j_4)$ is nonzero only for $j_1=j_2=0$. Similar to other terms. For the tenth term, we have

\begin{align*}
   \widehat{\phi}_{10}(j_1, j_2, j_3,j_4) = & ~ 6 \sum_{a_1, a_2, a_3, c \in \mathbb{Z}_p} u_1(a_1) u_2(a_2)u_3(a_3) w(c) \rho^{-(j_1 a_1+j_2 a_2+j_3 a_3+j_4 c)} \\
   = & ~ 6 \hat{u}_1(j_1) \hat{u}_2(j_2)\hat{u}_3(j_3) \hat{w}(j_4). 
\end{align*}

In particular,

\begin{align*}
\hat{u}_1(j_1) & = \sum_{a_1 \in \mathbb{Z}_p} \sqrt{\frac{1}{2 p}} \cos_p (\theta_{u_1}+\zeta a_1) \rho^{-j_1 a_1} \\
& =(8 p)^{-1 / 2} \sum_{a_1 \in \mathbb{Z}_p} (\rho^{\theta_{u_1}+\zeta a_1} + \rho^{-(\theta_{u_1}+\zeta a_1)}) \rho^{-j_1 a_1} \\
& =(8 p)^{-1 / 2}(\rho^{\theta_{u_1}} \sum_{a_1 \in \mathbb{Z}_p} \rho^{(\zeta-j_1) a_1} + \rho^{-\theta_{u_1}} \sum_{a_1 \in \mathbb{Z}_p} \rho^{-(\zeta + j_1) a_1}) \\
& = \begin{cases}\sqrt{p / 8} \cdot \rho^{\theta_{u_1}} & \text { if } j_1 = + \zeta \\
\sqrt{p / 8} \cdot \rho^{-\theta_{u_1}} & \text { if } j_1 = -\zeta \\
0 & \text { otherwise }\end{cases}
\end{align*}
where the first step comes from $\hat{u}_1(j_1)$ definition, the second step comes from Euler’s formula, the third step comes from simple algebra, the last step comes from the properties of discrete Fourier transform.
Similarly for $\hat{u}_2, \hat{u}_3$ and $\hat{w}$. As we consider $\zeta$ to be nonzero, we ignore the $\zeta=0$ case. Hence, $\widehat{\phi}_{10}(j_1, j_2, j_3,j_4)$ is nonzero only when $j_1, j_2, j_3,j_4$ are all $\pm \zeta$. We can summarize that $\hat{\phi}(j_1, j_2,j_3,j_4)$ can only be nonzero if one of the following satisfies:
\begin{itemize}
\item $j_1\cdot j_2 \cdot j_3 = 0$
\item $j_1, j_2, j_3, j_4 = \pm \zeta$.
\end{itemize}

Setting aside the previously discussed points, it's established in Lemma~\ref{lemma:multi-class} that the function $f$ maintains a consistent margin for various inputs as well as over different classes, i.e., $f$ can be broken down as

\begin{align*}
    f(a_1, a_2, a_3, c) = f_1(a_1, a_2, a_3, c) + f_2(a_1, a_2, a_3, c)
\end{align*}
where

\begin{align*}
    f_1(a_1, a_2, a_3, c) = F(a_1, a_2, a_3)
\end{align*}
for some $F: \mathbb{Z}_p \times \mathbb{Z}_p \times \mathbb{Z}_p \rightarrow \mathbb{R}$, and

\begin{align*}
    f_2(a_1, a_2, a_3, c)=\lambda \cdot \mathbf{1}_{a_1+a_2+a_3=c}
\end{align*}
where $\lambda>0$ is the margin of $f$.
Then, we have the DFT of $f_1$ and $f_2$ are

\begin{align*}
    \hat{f}_1(j_1, j_2, j_3, j_4)= \begin{cases}\hat{F}(j_1, j_2, j_3) & \text { if } j_4=0 \\ 0 & \text { otherwise }\end{cases}
\end{align*}
and

\begin{align*}
\hat{f}_2(j_1, j_2, j_3, j_4)= \begin{cases}
\lambda p^3 & \text { if } j_1=j_2=j_3=-j_4 \\
0 & \text { otherwise }
\end{cases}.
\end{align*}

Hence, when $j_1=j_2=j_3=-j_4 \neq 0$, we must have $\hat{f}(j_1, j_2, j_3,j_4)>0$. 
\end{proof}

\subsection{All Frequencies are Used for General \texorpdfstring{$k$}{} Version}\label{app:frequency:all_frequencies_k}

Let $f: \mathbb{Z}_p^{k+1} \rightarrow \mathbb{C}$. Its multi-dimensional discrete Fourier transform is defined as:
\begin{align*}
    & ~ \hat{f}(j_1, \dots,j_{k+1}) \\
    := & ~ \sum_{a_1 \in \mathbb{Z}_p} e^{-2 \pi \i \cdot j_1 a_1 / p} ( \dots (\sum_{a_k \in \mathbb{Z}_p} e^{-2 \pi \i \cdot j_k a_k / p} (\sum_{c \in \mathbb{Z}_p} e^{-2 \pi \i \cdot j_{k+1} c / p} f(a_1, \dots, a_k, c))).
\end{align*}

\begin{lemma}\label{lem:frequency-k}
If the following conditions hold
\begin{itemize}
    \item We adopt the uniform class weighting: $\forall c' \neq a_1 + \dots + a_k, ~~\tau (a_1, \dots, a_k)[c'] := 1/(p - 1)$.
    \item $f$ is the maximum $L_{2,k+1}$-margin solution.
\end{itemize}
Then, for any $j_1=\dots=j_k=-j_{k+1} \neq 0$, we have $\hat{f}(j_1, \dots ,j_{k+1})>0$.
\end{lemma}

\begin{proof}
For this proof, 
for all $j_1,\dots,j_{k+1} \in \mathbb{Z}$, to simplify the notation, let $\theta_u = \theta_u^* \cdot \frac{p}{2 \pi}$, by Lemma~\ref{lem:margin_soln-k}, so

\begin{align}\label{eq:u_1_theta_zeta}
    u_1(a_1)=\sqrt{\frac{2}{(k+1) p}} \cos _p (\theta_{u_1} + \zeta a_1) .
\end{align}

Let
\begin{align*}
f(a_1, \dots, a_k,c) & = \sum_{h=1}^H \phi_h(a_1, \dots, a_k,c) \\
& =\sum_{h=1}^H (u_{h,1}(a_1) + \dots +u_{h,k}(a_k))^k w_h(c) \\
& = ({\frac{2}{(k+1) p}})^{(k+1)/2} \sum_{h=1}^H (\cos _p (\theta_{u_{h,1}}+\zeta_h a_1) + \dots+ \cos _p (\theta_{u_{h,k}} + \zeta_h a_k))^k \cos_p (\theta_{w_h}+\zeta_h c)
\end{align*}
where each neuron conforms to the previously established form, and the width $H$ function is an arbitrary margin-maximizing network.
The first step is based on the definition of $f(a_1, \dots, a_k, c)$,  the subsequent step on the definition of $\phi_h(a_1, \dots, a_k, c)$, and the final step is justified by Eq.~\eqref{eq:u_1_theta_zeta}.

Each neuron $\phi$ we have

\begin{align*}
   \widehat{\phi}(j_1, \dots, j_k,j_{k+1}) = & ~ k! \sum_{a_1, \dots, a_k, c \in \mathbb{Z}_p} w(c) \rho^{-(j_1 a_1+\dots+j_k a_k+j_{k+1} c)} \prod_{i=1}^k u_i(a_i)   \\
   = & ~ k!  \hat{w}(j_{k+1})\prod_{i=1}^k \hat{u}_i(j_i) . 
\end{align*}

In particular,

\begin{align*}
\hat{u}_1(j_1) & = \sum_{a_1 \in \mathbb{Z}_p} \sqrt{\frac{2}{(k+1) p}} \cos_p (\theta_{u_1}+\zeta a_1) \rho^{-j_1 a_1} \\
& =\sqrt{\frac{1}{2(k+1) p}} \sum_{a_1 \in \mathbb{Z}_p} (\rho^{\theta_{u_1}+\zeta a_1} + \rho^{-(\theta_{u_1}+\zeta a_1)}) \rho^{-j_1 a_1} \\
& =\sqrt{\frac{1}{2(k+1) p}}(\rho^{\theta_{u_1}} \sum_{a_1 \in \mathbb{Z}_p} \rho^{(\zeta-j_1) a_1} + \rho^{-\theta_{u_1}} \sum_{a_1 \in \mathbb{Z}_p} \rho^{-(\zeta + j_1) a_1}) \\
& = \begin{cases}\sqrt{\frac{p}{2(k+1) }} \cdot \rho^{\theta_{u_1}} & \text { if } j_1 = + \zeta \\
\sqrt{\frac{p}{2(k+1) }} \cdot \rho^{-\theta_{u_1}} & \text { if } j_1 = -\zeta \\
0 & \text { otherwise, }\end{cases}
\end{align*}
where the first step comes from $\hat{u}_1(j_1)$ definition, the second step comes from Euler’s formula, the third step comes from simple algebra, the last step comes from the properties of discrete Fourier transform.
Similarly for $\hat{u}_i$ and $\hat{w}$. We consider $\zeta$ to be nonzero, so we ignore the $\zeta=0$ case. Hence, $\widehat{\phi}(j_1, \dots, j_k,j_{k+1})$ is nonzero only when $j_1, \dots, j_k,j_{k+1}$ are all $\pm \zeta$. We can summarize that $\hat{\phi}(j_1, \dots,j_k,j_{k+1})$ can only be nonzero if one of the below conditions satisfies:
\begin{itemize}
\item $ \prod_{i=1}^k j_i = 0$
\item $j_1, \dots, j_k, j_{k+1} = \pm \zeta$.
\end{itemize}
Setting aside the previously discussed points, it's established in Lemma~\ref{lemma:multi-class} that the function $f$ maintains a consistent margin for various inputs as well as over different classes, i.e., $f$ can be broken down as

\begin{align*}
    f(a_1, \dots, a_k, c) = f_1(a_1, \dots, a_k, c) + f_2(a_1, \dots, a_k, c)
\end{align*}
where

\begin{align*}
    f_1(a_1, \dots, a_k, c) = F(a_1, \dots, a_k)
\end{align*}
for some $F: \mathbb{Z}_p^k \rightarrow \mathbb{R}$, and

\begin{align*}
    f_2(a_1, \dots, a_k, c)=\lambda \cdot \mathbf{1}_{a_1+\dots+a_k=c}
\end{align*}
where $\lambda>0$ is the margin of $f$.
Then, we have the DFT of $f_1$ and $f_2$ are

\begin{align*}
    \hat{f}_1(j_1, \dots, j_k, j_{k+1})= \begin{cases}\hat{F}(j_1, \dots, j_k) & \text { if } j_{k+1}=0 \\ 0 & \text { otherwise }\end{cases}
\end{align*}
and

\begin{align*}
\hat{f}_2(j_1, \dots, j_k, j_{k+1})= \begin{cases}
\lambda p^k & \text { if } j_1=\dots=j_k=-j_{k+1} \\
0 & \text { otherwise }
\end{cases}.
\end{align*}

Hence, when $j_1=\dots=j_k=-j_{k+1} \neq 0$, we must have $\hat{f}(j_1, \dots, j_k,j_{k+1})>0$. 
\end{proof}

\section{Proof of Main Result}\label{app:main}

Section~\ref{app:main:main_result_k_3} proves the main result for $k=3$. Section~\ref{app:main:main_result_k} proves the general $k$ version of our main result.

\subsection{Main result for \texorpdfstring{$k=3$}{}} \label{app:main:main_result_k_3}
\begin{theorem}\label{thm:main}
When $k=3$, let $f(\theta, x)$ be the one-hidden layer networks defined in Section~\ref{sec:problem}.
If the following conditions hold
\begin{itemize}
    \item We adopt the uniform class weighting: $\forall c' \neq a_1 + a_2 + a_3, ~~\tau (a_1, a_2, a_3)[c'] := 1/(p - 1)$.
    \item  $m \geq 16(p-1)$ neurons. 
\end{itemize}

Then we have the maximum $L_{2,4}$-margin network satisfying:
\begin{itemize}
    \item  The maximum $L_{2,4}$-margin for a given dataset $D_p$ is:
    \begin{align*}
        \gamma^*=\frac{3}{16} \cdot \frac{1}{p(p-1)}.
    \end{align*}
    \item For each neuron $\phi(\{u_1,u_2, u_3, w\} ; a_1, a_2, a_3)$, there is a constant scalar $\beta \in \R$ and a frequency $\zeta \in\{1, \ldots, \frac{p-1}{2}\}$ satisfying
    \begin{align*}
    u_1(a_1) &= \beta \cdot \cos ( \theta_{u_1}^* + 2 \pi \zeta a_1 /p ) \\
    u_2(a_2) & =\beta \cdot \cos (\theta_{u_2}^*+2 \pi \zeta a_2 / p) \\
    u_3(a_3) & =\beta \cdot \cos (\theta_{u_3}^*+2 \pi \zeta a_3 / p) \\
    w(c) & =\beta \cdot \cos (\theta_w^*+2 \pi \zeta c / p)
    \end{align*}
    where $\theta_{u_1}^*, \theta_{u_2}^*, \theta_{u_3}^*, \theta_w^* \in \R$ are some phase offsets satisfying $\theta_{u_1}^*+\theta_{u_2}^*+\theta_{u_3}^*=\theta_w^*$.
    \item For each frequency $\zeta \in\{1, \ldots, \frac{p-1}{2}\}$, there exists one neuron using this frequency only.
\end{itemize}

\end{theorem}
\begin{proof}
By Lemma~\ref{lem:margin_soln}, we get the single neuron class-weighted margin solution set $\Omega_q^{'*}$ satisfying Condition \ref{lem:translation_property} and $\gamma^*$.  

By Lemma~\ref{lem:construct} and Lemma~\ref{lem:combine}, we can construct network $\theta^*$ which uses neurons in $\Omega_q^{'*}$ and satisfies Condition \ref{lem:translation_property} and Definition \ref{def:q_star_theta_star} with respect to $q = \mathrm{unif}(\mathbb{Z}_{p})$. By Lemma~\ref{lemma:multi-class}, we know it is the maximum-margin solution.  

By Lemma~\ref{lem:frequency}, when $j_1=j_2=j_3=-j_4 \neq 0$, we must have $\hat{f}(j_1, j_2, j_3,j_4)>0$. However, as discrete Fourier transform $\hat{\phi}$ of each neuron is nonzero, for each frequency, we must have that there exists one neuron using it.

\end{proof}

\subsection{Main Result for General \texorpdfstring{$k$}{} Version}\label{app:main:main_result_k}

\begin{theorem}[Formal version of Theorem~\ref{thm:main_k:informal}]\label{thm:main_k:formal}
Let $f(\theta, x)$ be the one-hidden layer networks defined in Section~\ref{sec:problem}.
If the following conditions hold
\begin{itemize}
    \item We adopt the uniform class weighting: $\forall c' \neq a_1 + \dots + a_k, ~~\tau (a_1, \dots, a_k)[c'] := 1/(p - 1)$.
    \item  $m \geq 2^{2k-1} \cdot {p-1 \over 2}$ neurons. 
\end{itemize}

Then we have the maximum $L_{2,k+1}$-margin network satisfying:
\begin{itemize}
    \item  The maximum $L_{2,k+1}$-margin for a given dataset $D_p$ is:
    \begin{align*}
        \gamma^*=\frac{2(k!)}{(2k+2)^{(k+1)/2}(p-1) p^{(k-1)/2}}.
    \end{align*}
    \item For each neuron $\phi(\{u_1,\dots, u_k, w\} ; a_1, \dots, a_k)$ there is a constant scalar $\beta \in \R$ and a frequency $\zeta \in\{1, \ldots, \frac{p-1}{2}\}$ satisfying
    \begin{align*}
    u_1(a_1) &= \beta \cdot \cos ( \theta_{u_1}^* + 2 \pi \zeta a_1 /p ) \\
    & \dots \\
    u_k(a_k) & =\beta \cdot \cos (\theta_{u_k}^*+2 \pi \zeta a_k / p) \\
    w(c) & =\beta \cdot \cos (\theta_w^*+2 \pi \zeta c / p)
    \end{align*}
    where $\theta_{u_1}^*, \dots, \theta_{u_k}^*, \theta_w^* \in \R$ are some phase offsets satisfying  $\theta_{u_1}^*+\dots+\theta_{u_k}^*=\theta_w^*$.
    \item For every frequency $\zeta \in\{1, \ldots, \frac{p-1}{2}\}$, there exists one neuron using this frequency only.
\end{itemize}

\end{theorem}
\begin{proof}
Follow the same proof sketch as Theorem~\ref{thm:main} by Lemma~\ref{lem:margin_soln-k}, Condition \ref{lem:translation_property}, Lemma~\ref{lem:construct-k}, Lemma~\ref{lem:combine}, Definition \ref{def:q_star_theta_star}, Lemma~\ref{lemma:multi-class}, Lemma~\ref{lem:frequency-k}.

\end{proof}

\section{More Empirical Details and Results}\label{app:exp}

\subsection{Implement Details}\label{app:sec:implement}

{\bf Licenses for Existing Assets \& Open Access to Data and Code.} Our code is based on a brilliant open source repository, \url{https://github.com/Sea-Snell/grokking}, which requires MIT License. We provide all of our codes in the supplemental material, including dataset generation code. We do not require open data access as we run experiments on synthetic datasets, i.e., modular addition.   

{\bf Experimental Result Reproducibility.} 
We provide all of our codes in
the supplemental material with a clear README file and clear configuration files for our experiments reproducibility. 

{\bf Experimental Setting/Details \& Experiment Statistical Significance.}
The detailed configuration can be found in supplemental material. We make a copy version here for convenience. 

For two-layer neural network training, we have the following details:
\begin{itemize}
    \item number of data loader workers: $4$
    \item batch size: $1024$
    \item learning rate: $5 \times 10^{-3}$
    \item regularization strength $\lambda$: $0.005$
    \item AdamW hyper-parameter $(\beta_1, \beta_2)$: $(0.9, 0.98)$
    \item warm-up steps: $10$
\end{itemize}

For one-layer Transformer training, we have the following details:
\begin{itemize}
    \item number of data loader workers: $4$
    \item batch size: $1024$
    \item learning rate: $1 \times 10^{-3}$
    \item regularization strength $\lambda$: $0.001$
    \item AdamW hyper-parameter $(\beta_1, \beta_2)$: $(0.9, 0.98)$
    \item warm-up steps: $10$
\end{itemize}

All results we ran 3 times with different random seeds. In Figure~\ref{fig:grok}, we reported the mean and variance range.

{\bf Experiments Compute Resources.} 
All experiments is conducted on single A100 40G NVIDIA GPU. All experiments can be finished in at most three days.  

\subsection{One-hidden Layer Neural Network}\label{app:sec:exp_nn}
In Figure~\ref{fig:nn_w_k3} and Figure~\ref{fig:nn_freq_k3}, we use SGD to train a two-layer network with $m=1536=2^{2k-2} \cdot (p-1)$ neurons, i.e., Eq.~\eqref{eq:nn}, on $k=3$-sum mod-$p=97$ addition dataset, i.e., Eq.~\eqref{eq:data}. In Figure~\ref{fig:nn_w_k5} and Figure~\ref{fig:nn_freq_k5}, we use SGD to train a two-layer network with $m=5632=2^{2k-2} \cdot (p-1)$ neurons, i.e., Eq.~\eqref{eq:nn}, on $k=5$-sum mod-$p=23$ addition dataset, i.e., Eq.~\eqref{eq:data}. 

Figure~\ref{fig:nn_w_k3} and Figure~\ref{fig:nn_w_k5} show that the networks trained with stochastic gradient descent have single-frequency hidden neurons, which support our analysis in Lemma~\ref{lem:margin_soln-k:informal}. Furthermore, Figure~\ref{fig:nn_freq_k3} and Figure~\ref{fig:nn_freq_k5} demonstrate that the network will learn all frequencies in the Fourier spectrum which is consistent with our analysis in Lemma~\ref{lem:construct-k:informal}. Together, they verify our main results in Theorem~\ref{thm:main_k:informal} and show that the network trained by SGD prefers to learn Fourier-based circuits.

\subsection{One-layer Transformer}\label{app:sec:exp_transformer}

In Figure~\ref{fig:s_k3} , we train a one-layer transformer with $m=160$ heads attention, on $k=3$-sum mod-$p=61$ addition dataset, i.e., Eq.~\eqref{eq:data}. In Figure~\ref{fig:s_k5} , we train a one-layer transformer with $m=160$ heads attention, on $k=5$-sum mod-$p=17$ addition dataset, i.e., Eq.~\eqref{eq:data}. 

Figure~\ref{fig:s_k3} and Figure~\ref{fig:s_k5} show that the one-layer transformer trained with stochastic gradient descent learns 2-dim cosine shape attention matrices, which is similar to one-hidden layer neural networks in Figure~\ref{fig:nn_w_k3} and Figure~\ref{fig:nn_w_k5}. This means that the attention layer has a similar learning mechanism to neural networks in the modular arithmetic task, where it prefers to learn Fourier-based circuits when trained by SGD. 

\begin{figure}
    \centering
\includegraphics[width=0.8\linewidth]{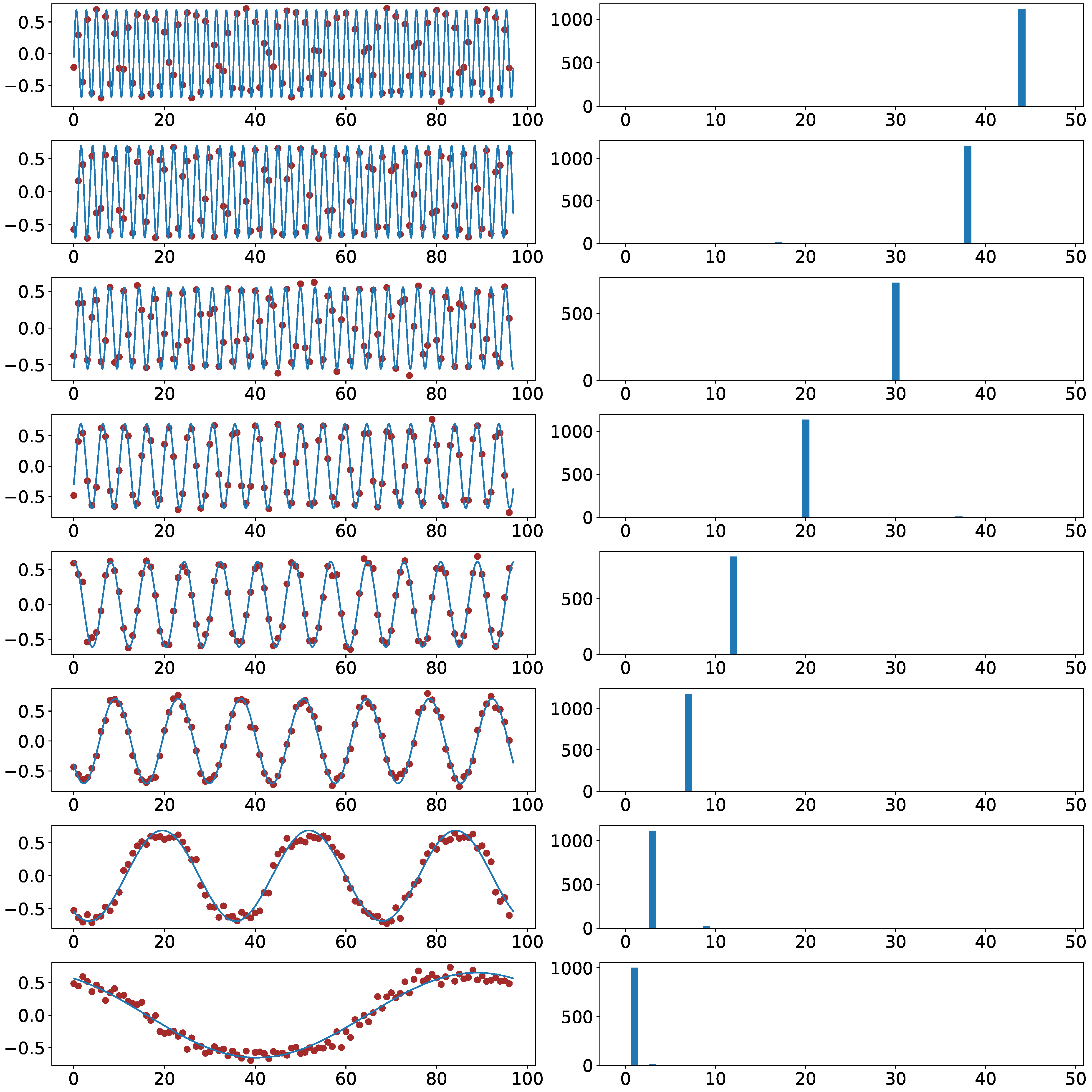}
    \caption{Cosine shape of the trained embeddings (hidden layer weights) and corresponding power of  Fourier spectrum. The two-layer network with $m=1536$ neurons is trained on $k=3$-sum mod-$p=97$ addition dataset. We even split the whole datasets ($p^k = 97^3$ data points) into the training and test datasets. Every row represents a random neuron from the network. The left figure shows the final trained embeddings, with red dots indicating the true weight values, and the pale blue interpolation is achieved by identifying the function that shares the same Fourier spectrum. The right figure shows their Fourier power spectrum. The results in these figures are consistent with our analysis statements in Lemma~\ref{lem:margin_soln-k:informal}.}
    \label{fig:nn_w_k3}
\end{figure}

\begin{figure*}
    \centering
\includegraphics[width=0.50\linewidth]{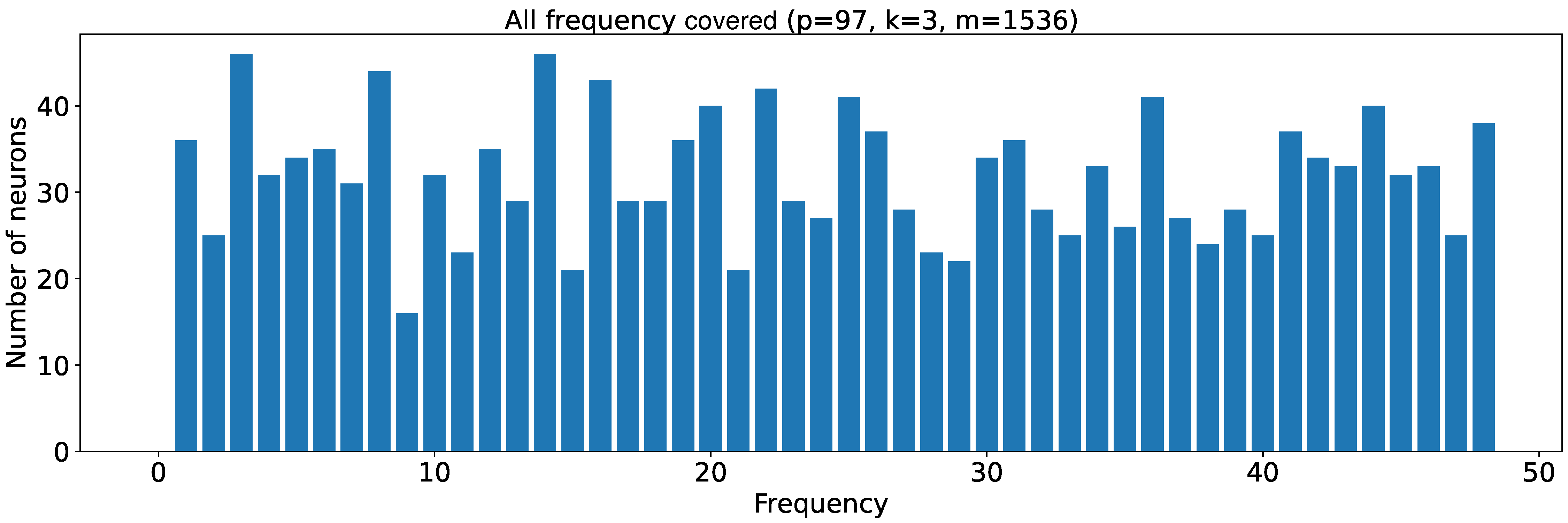}
\includegraphics[width=0.24\linewidth]{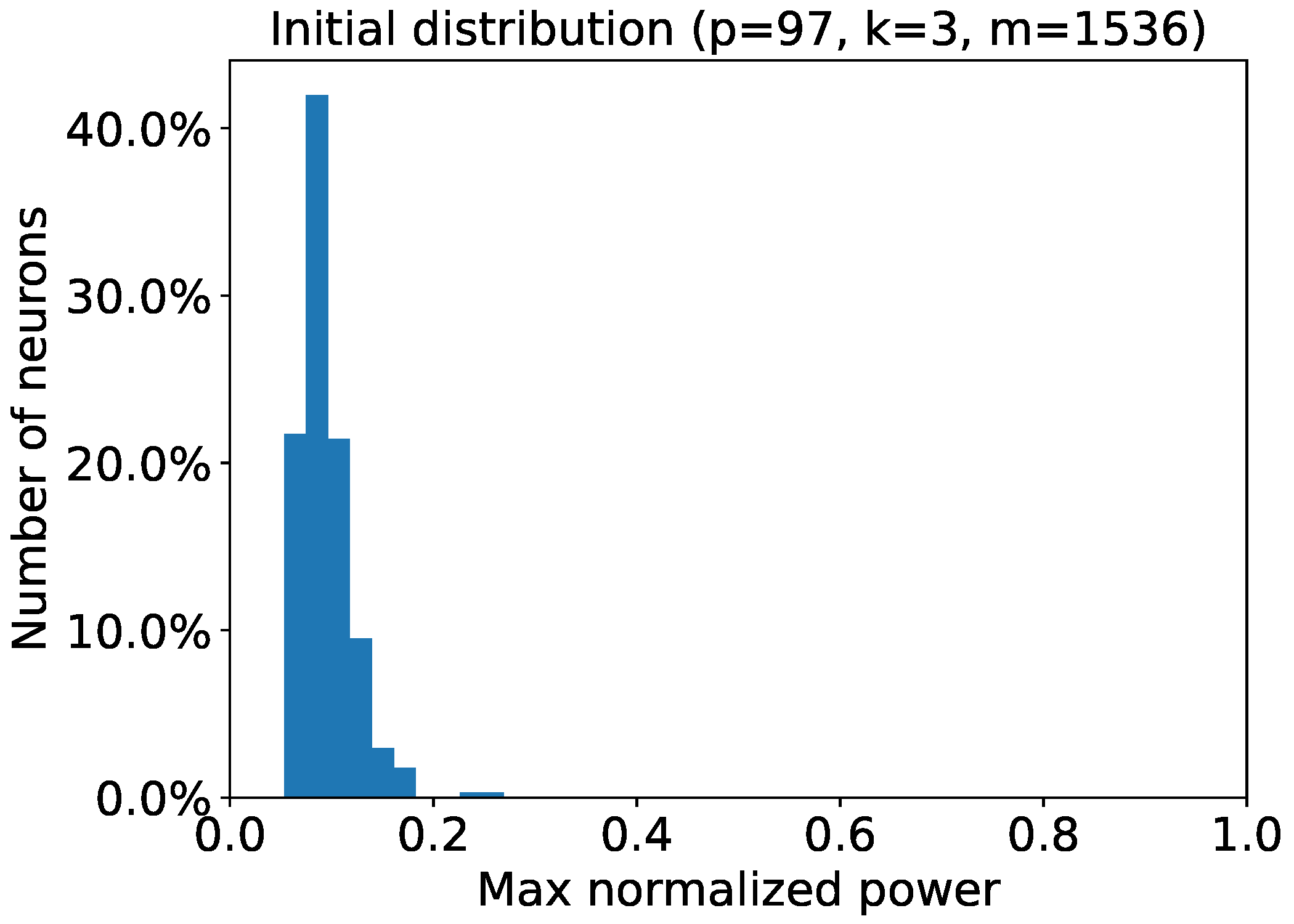}
\includegraphics[width=0.24\linewidth]{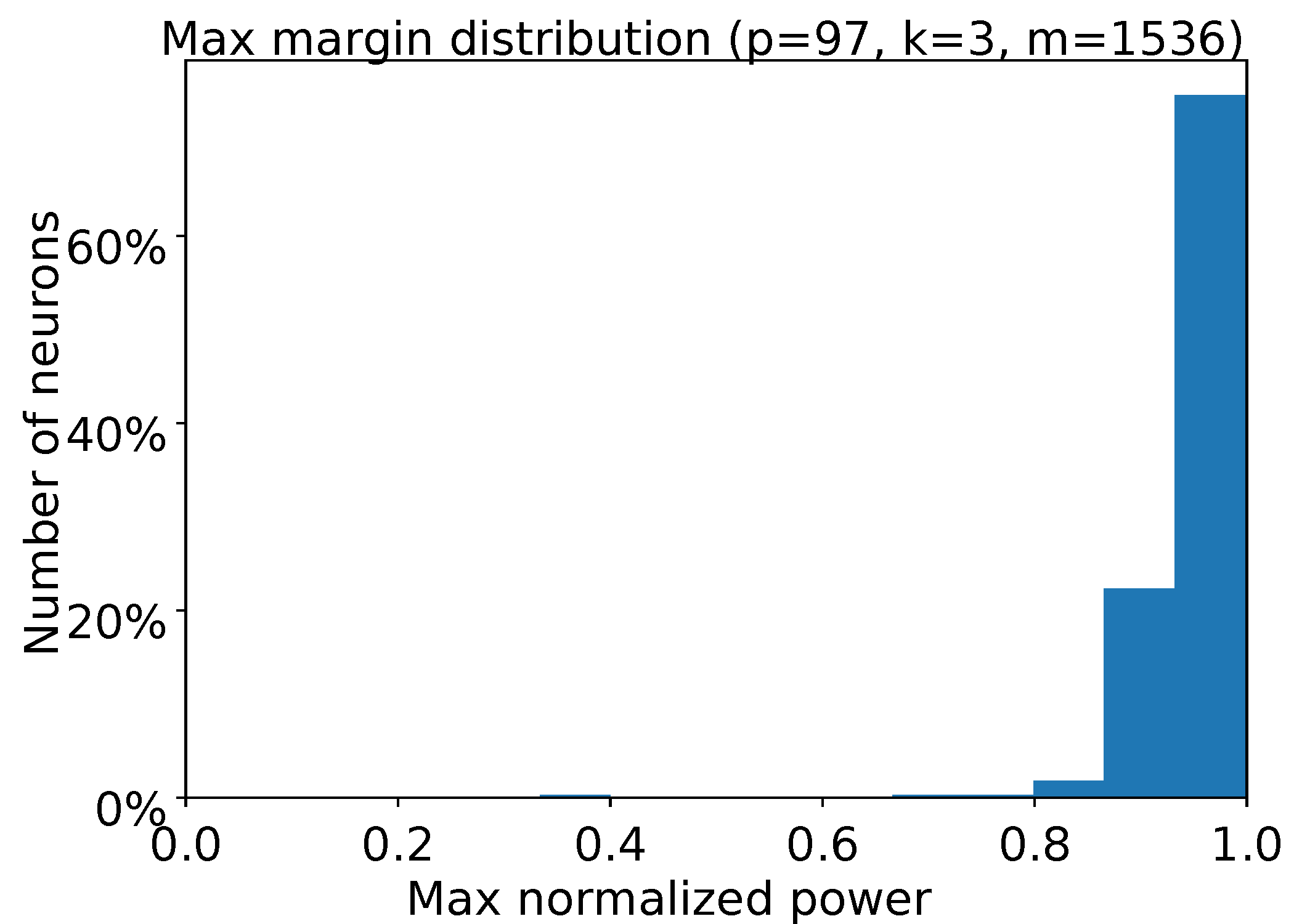}
    \caption{All Fourier spectrum frequencies being covered and the maximum normalized power of the embeddings (hidden layer weights). The one-hidden layer network with $m=1536$ neurons is trained on $k=3$-sum mod-$p=97$ addition dataset. We denote $\hat{u}[i]$ as the Fourier transform of $u[i]$. Let $\max_i |\hat{u}[i]|^2 /( \sum|\hat{u}[j]|^2 )$ be the maximum normalized power. 
    Mapping each neuron to its maximum normalized power frequency, (a) shows the final frequency distribution of the embeddings. 
    Similar to our construction analysis in Lemma~\ref{lem:construct-k:informal}, we have an almost uniform distribution over all frequencies. 
    (b) shows the maximum normalized power of the neural network with random initialization. (c) shows, in frequency space, the embeddings of the final trained network are one-sparse, i.e., maximum normalized power being almost 1 for all neurons. This is consistent with our maximum-margin analysis results in Lemma~\ref{lem:construct-k:informal}.}
    \label{fig:nn_freq_k3}
\end{figure*}

\begin{figure}
    \centering
\includegraphics[width=0.8\linewidth]{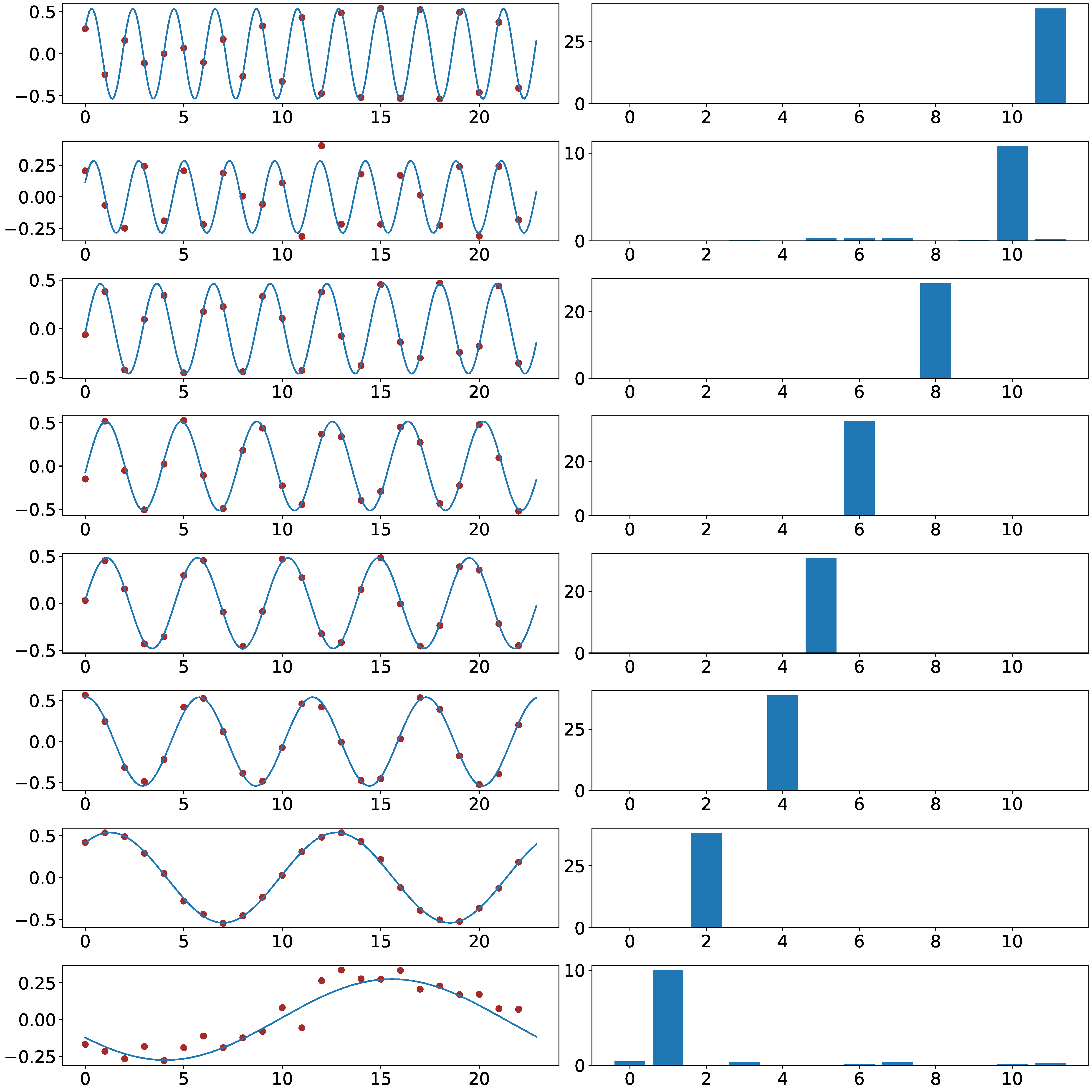}
    \caption{Cosine shape of the trained embeddings (hidden layer weights) and corresponding power of  Fourier spectrum. The two-layer network with $m=5632$ neurons is trained on $k=5$-sum mod-$p=23$ addition dataset. We even split the whole datasets ($p^k = 23^5$ data points) into the training and test datasets. Every row represents a random neuron from the network. The left figure shows the final trained embeddings, with red dots indicating the true weight values, and the pale blue interpolation is achieved by identifying the function that shares the same Fourier spectrum. The right figure shows their Fourier power spectrum. The results in these figures are consistent with our analysis statements in Lemma~\ref{lem:margin_soln-k:informal}.}
    \label{fig:nn_w_k5}
\end{figure}

\begin{figure*}
    \centering
\includegraphics[width=0.50\linewidth]{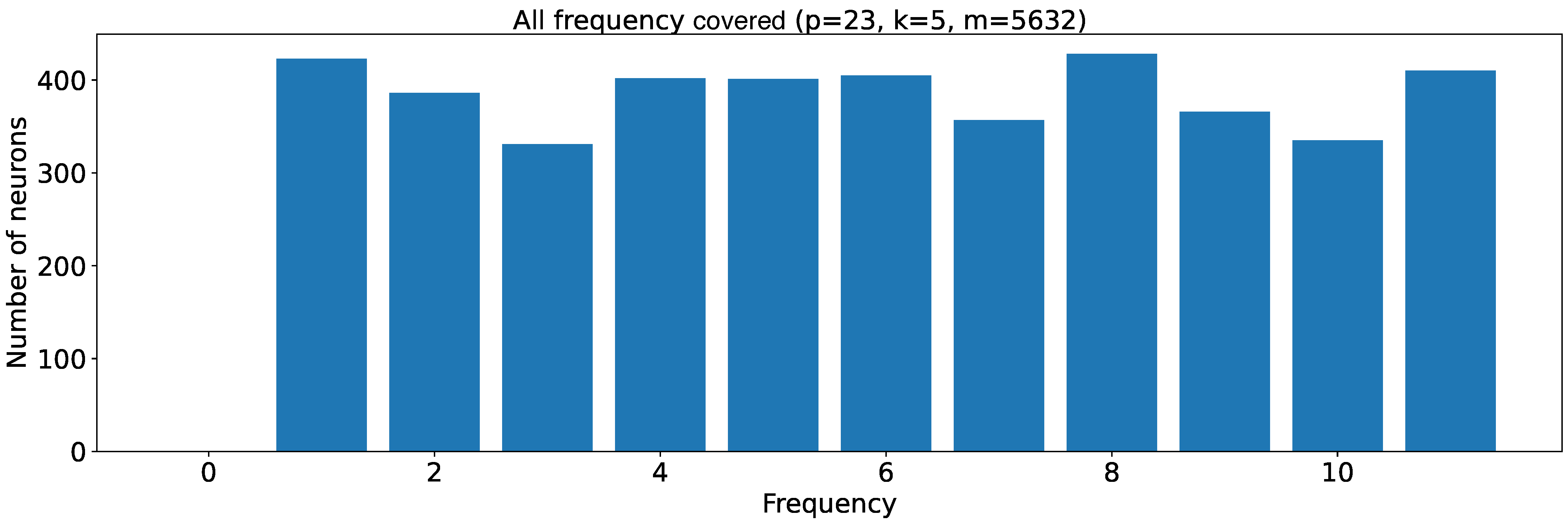}
\includegraphics[width=0.24\linewidth]{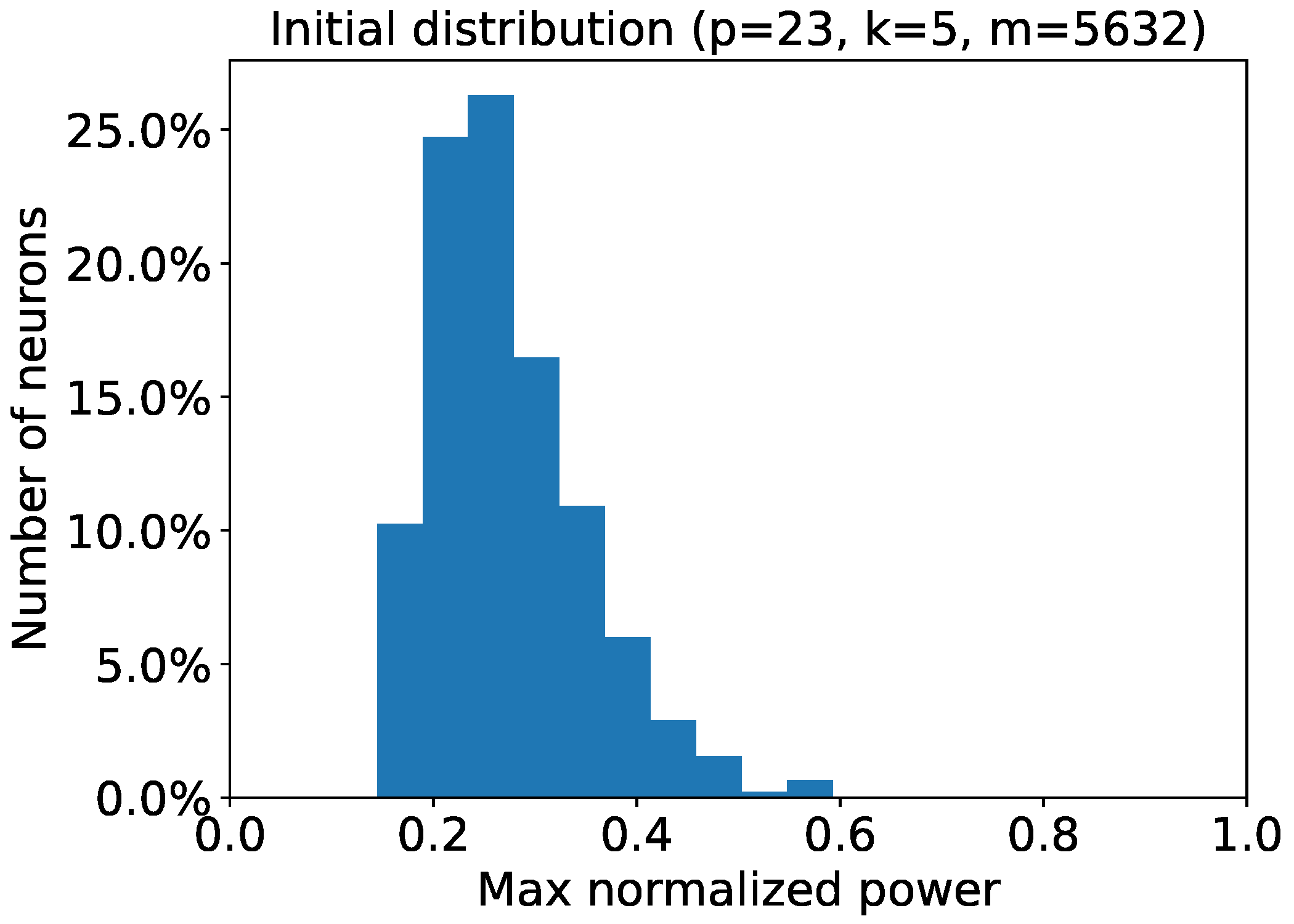}
\includegraphics[width=0.24\linewidth]{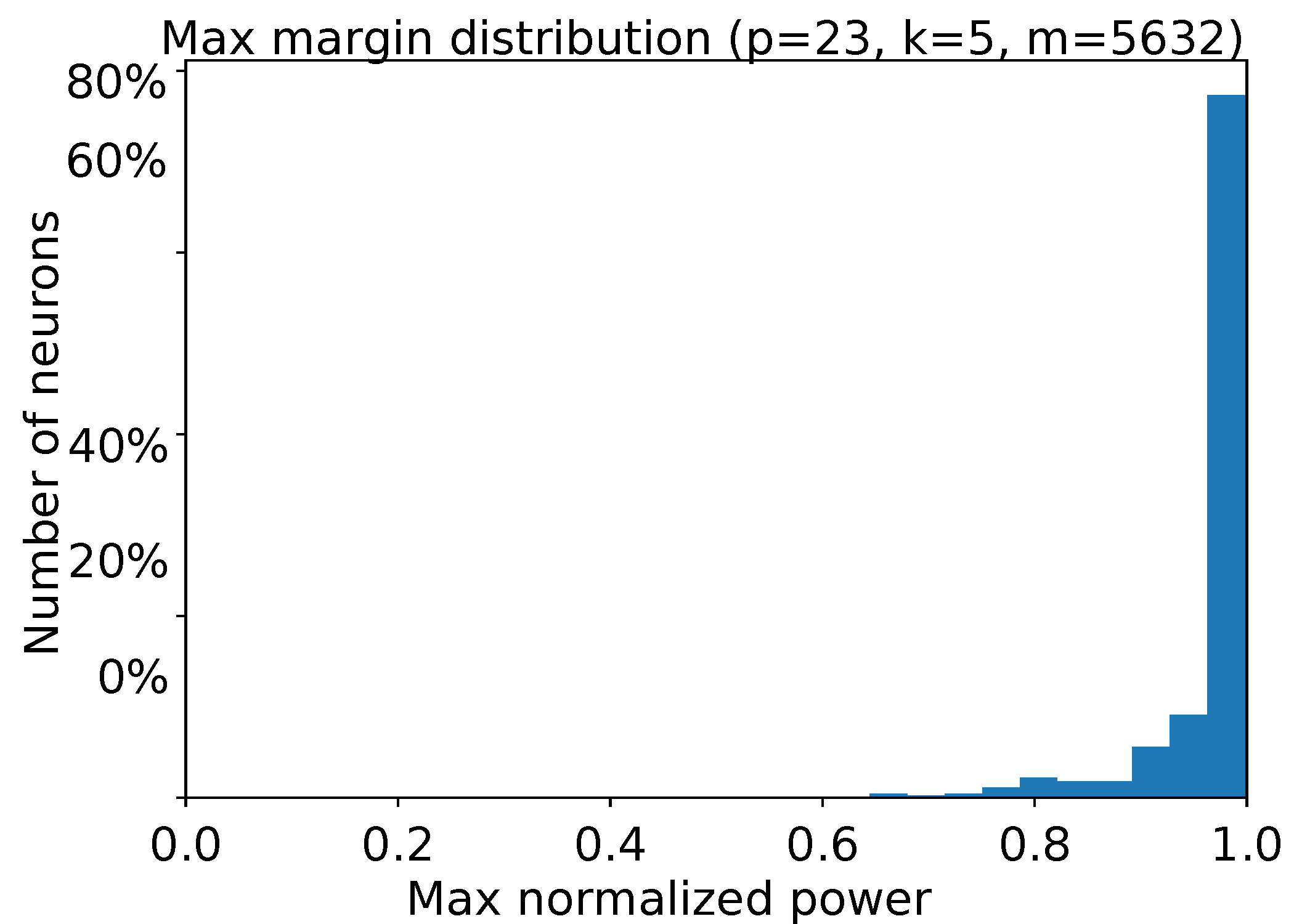}
    \caption{All Fourier spectrum frequencies being covered and the maximum normalized power of the embeddings (hidden layer weights). The one-hidden layer network with $m=5632$ neurons is trained on $k=5$-sum mod-$p=23$ addition dataset. We denote $\hat{u}[i]$ as the Fourier transform of $u[i]$. Let $\max_i |\hat{u}[i]|^2 /( \sum|\hat{u}[j]|^2 )$ be the maximum normalized power. 
    Mapping each neuron to its maximum normalized power frequency, (a) shows the final frequency distribution of the embeddings. 
    Similar to our construction analysis in Lemma~\ref{lem:construct-k:informal}, we have an almost uniform distribution over all frequencies. 
    (b) shows the maximum normalized power of the neural network with random initialization. (c) shows, in frequency space, the embeddings of the final trained network are one-sparse, i.e., maximum normalized power being almost 1 for all neurons. This is consistent with our maximum-margin analysis results in Lemma~\ref{lem:construct-k:informal}.}
    \label{fig:nn_freq_k5}
\end{figure*}

\begin{figure}
    \centering
\includegraphics[width=0.38\linewidth]{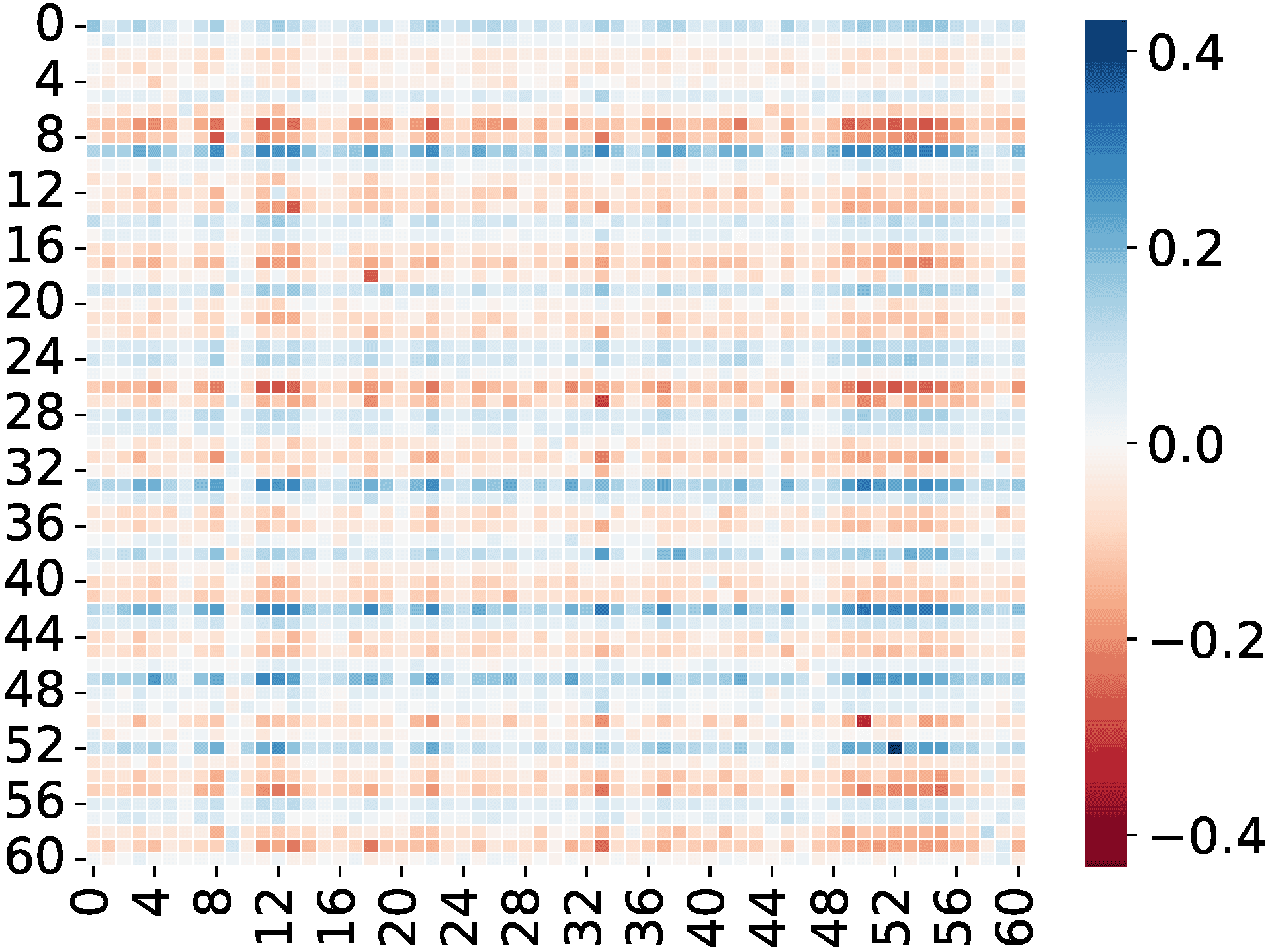}
\includegraphics[width=0.38\linewidth]{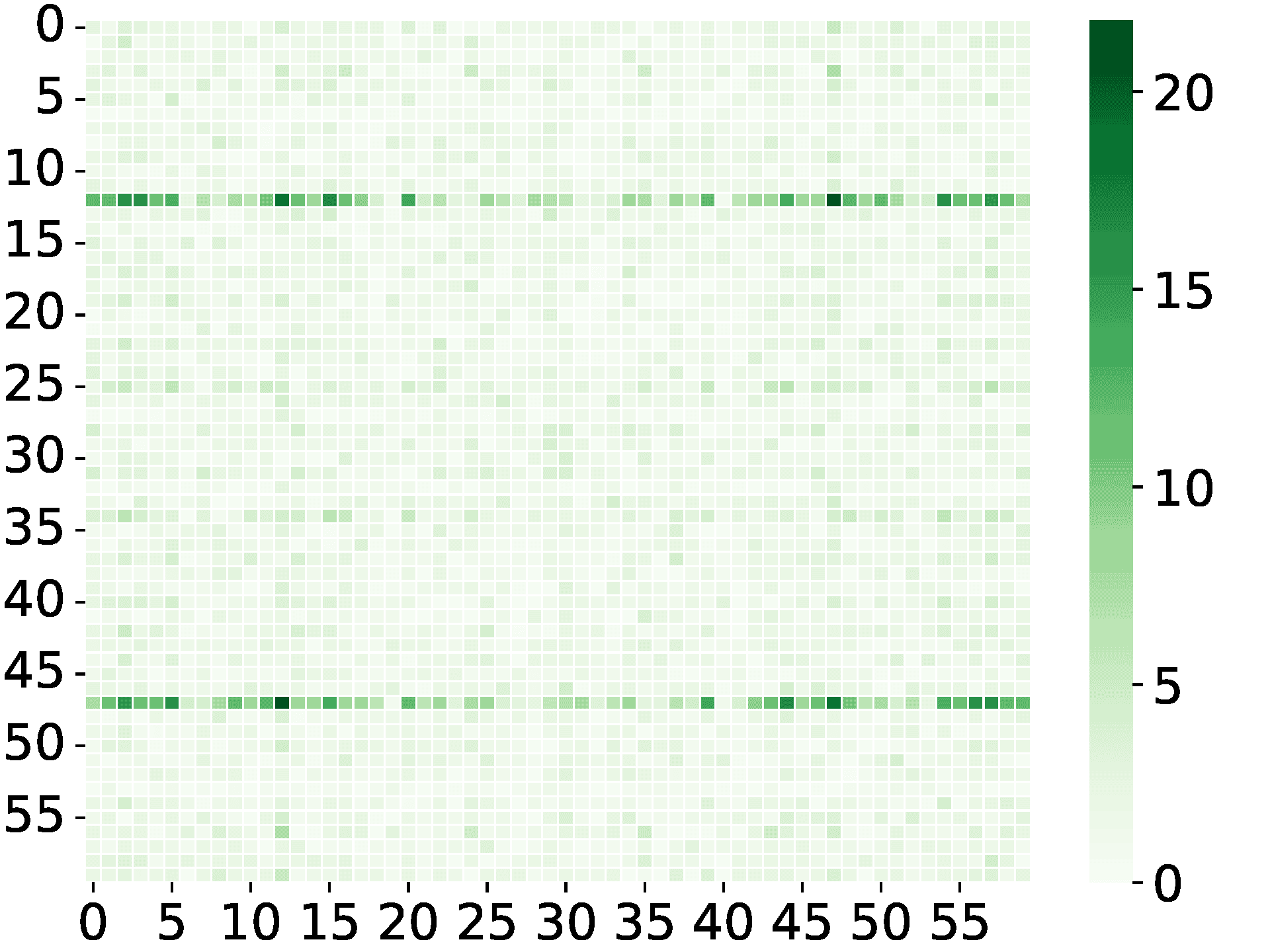}
\includegraphics[width=0.38\linewidth]{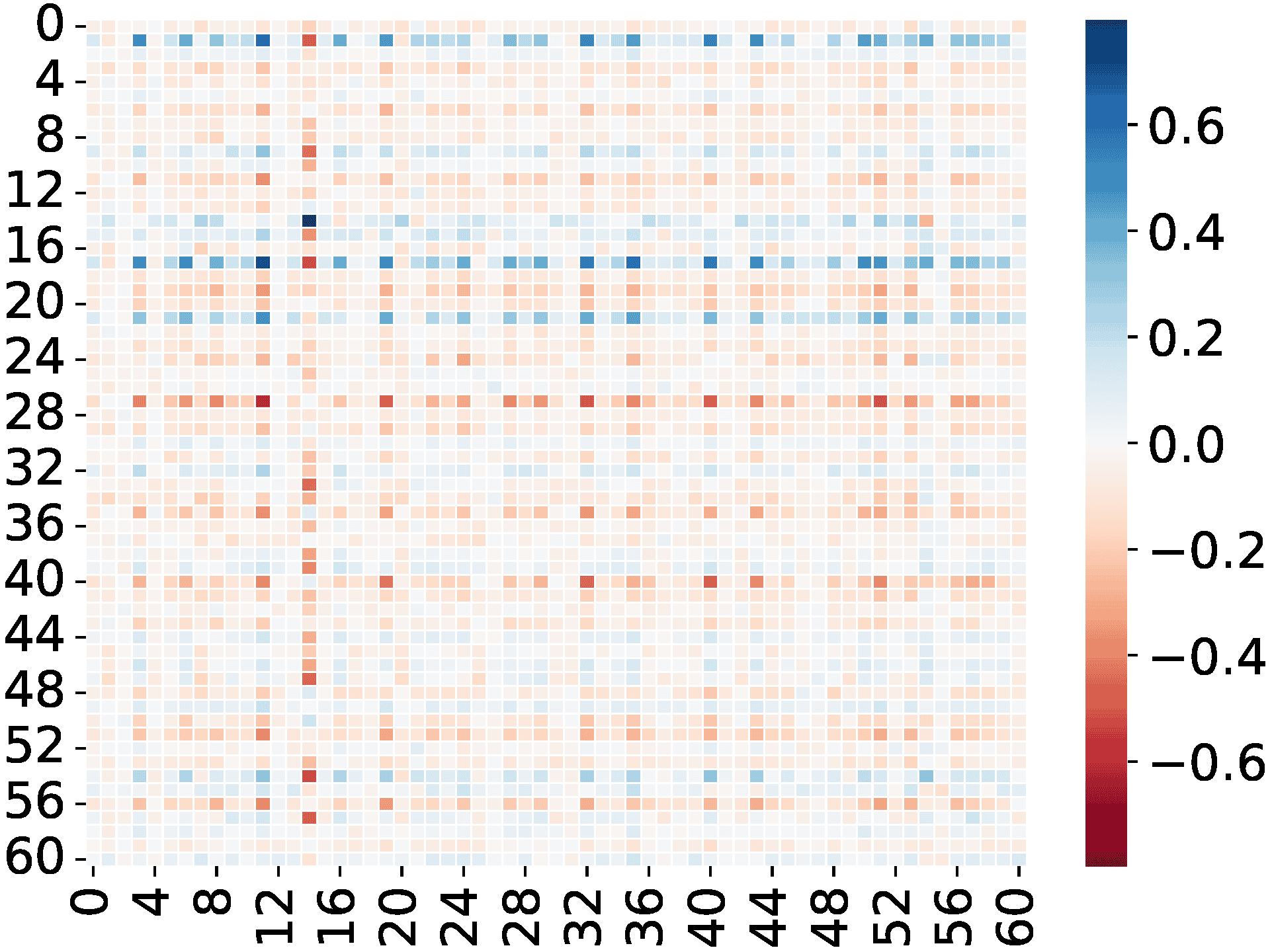}
\includegraphics[width=0.38\linewidth]{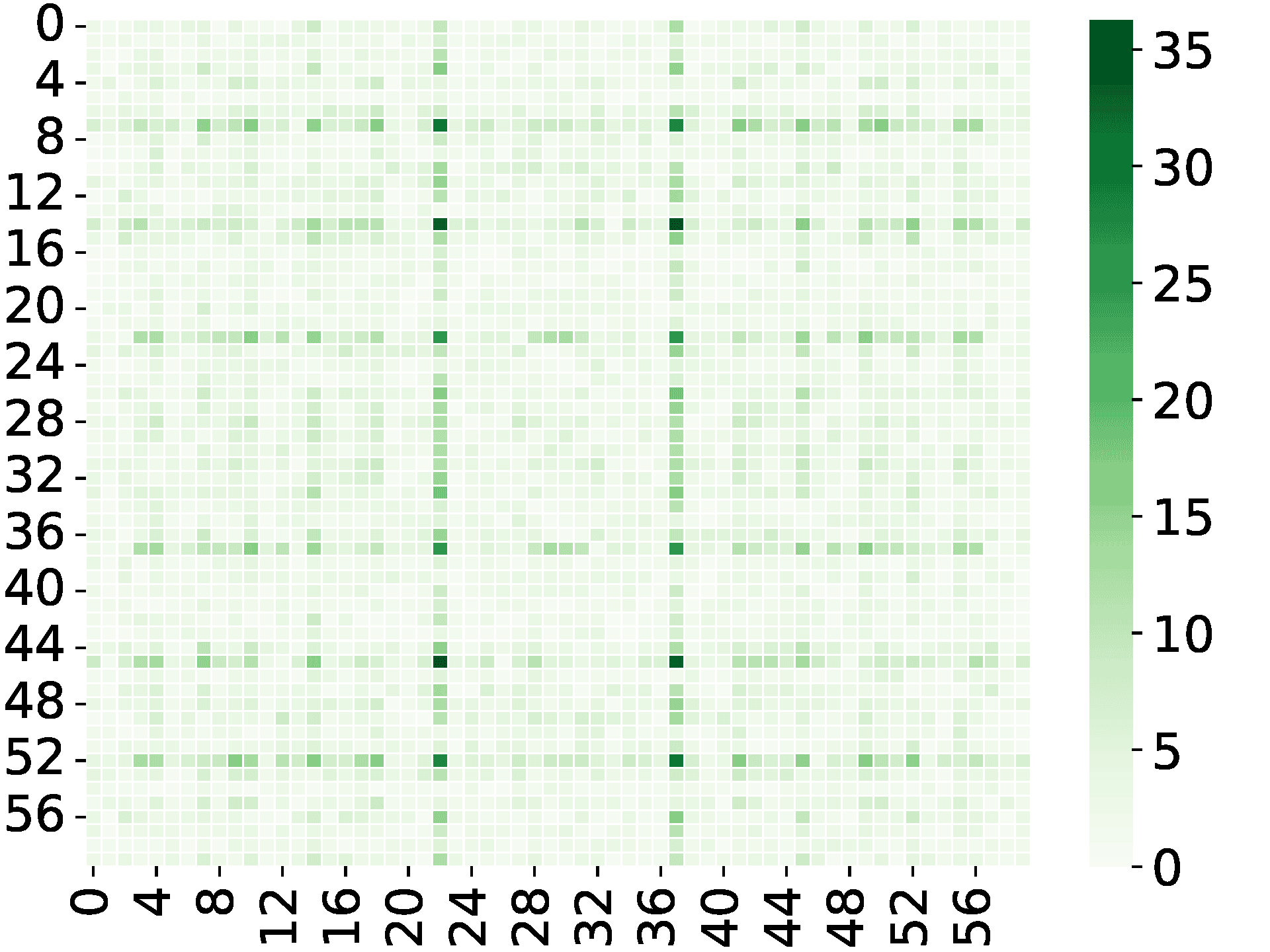}
\includegraphics[width=0.38\linewidth]{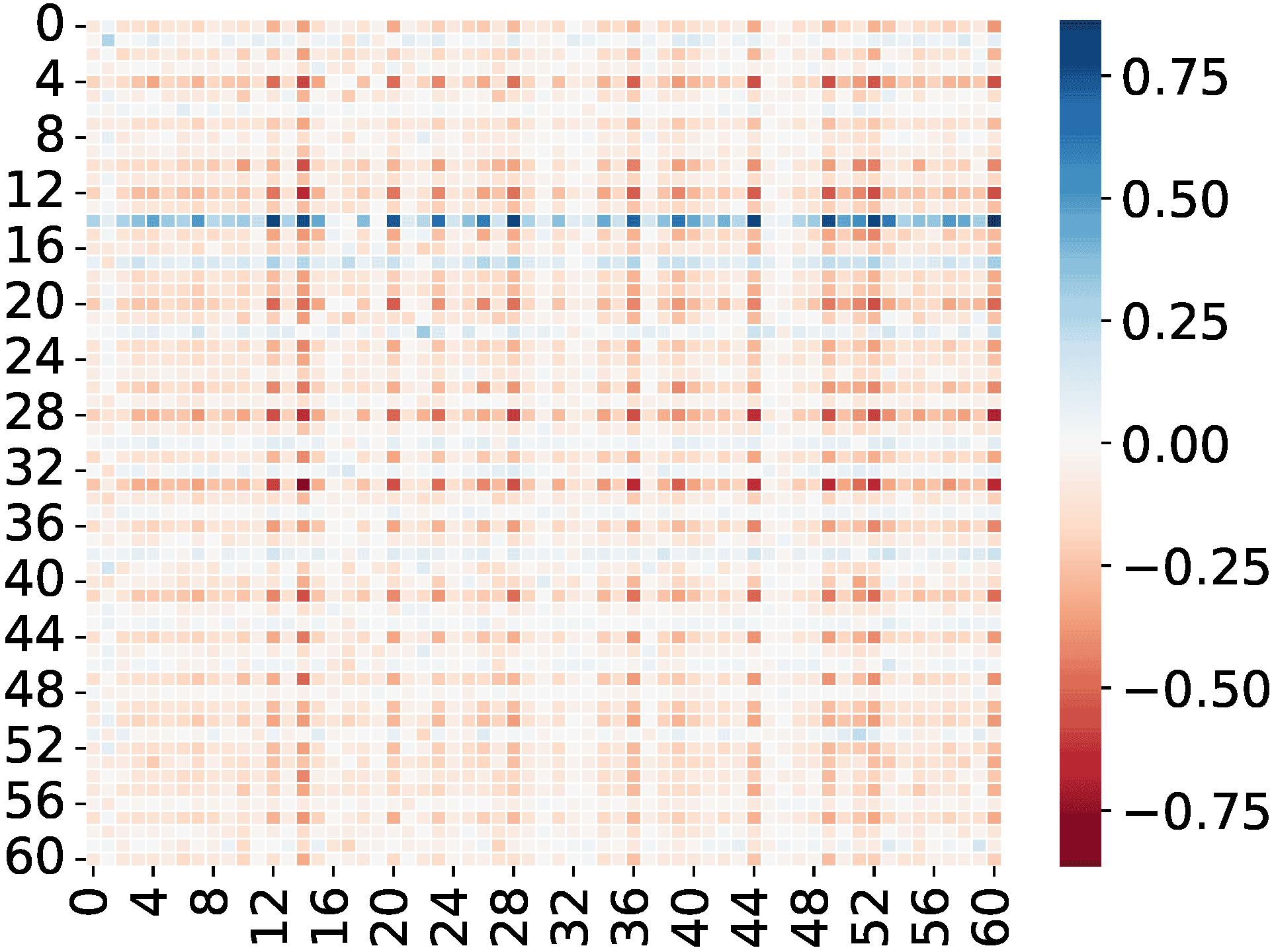}
\includegraphics[width=0.38\linewidth]{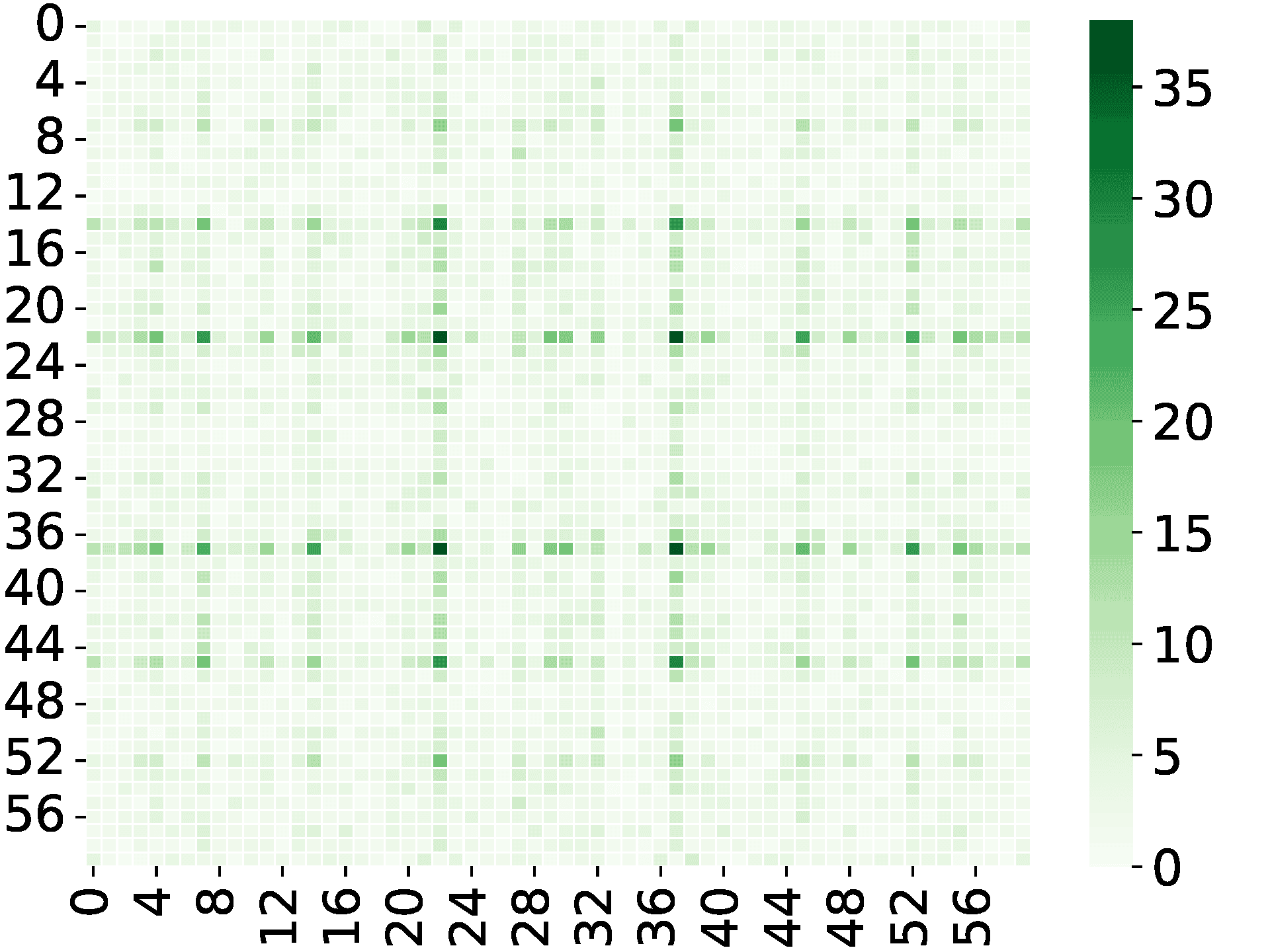}
\includegraphics[width=0.38\linewidth]{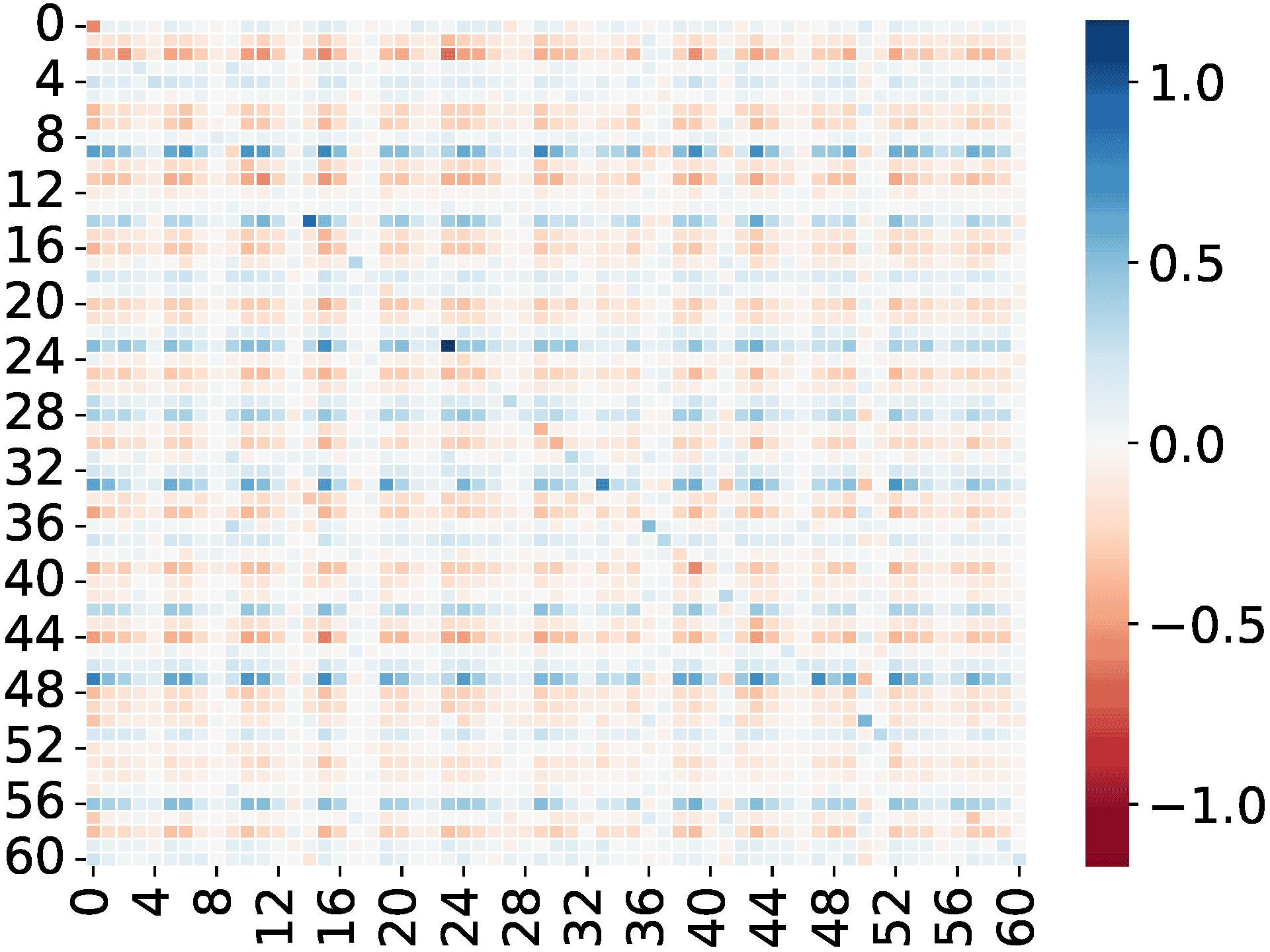}
\includegraphics[width=0.38\linewidth]{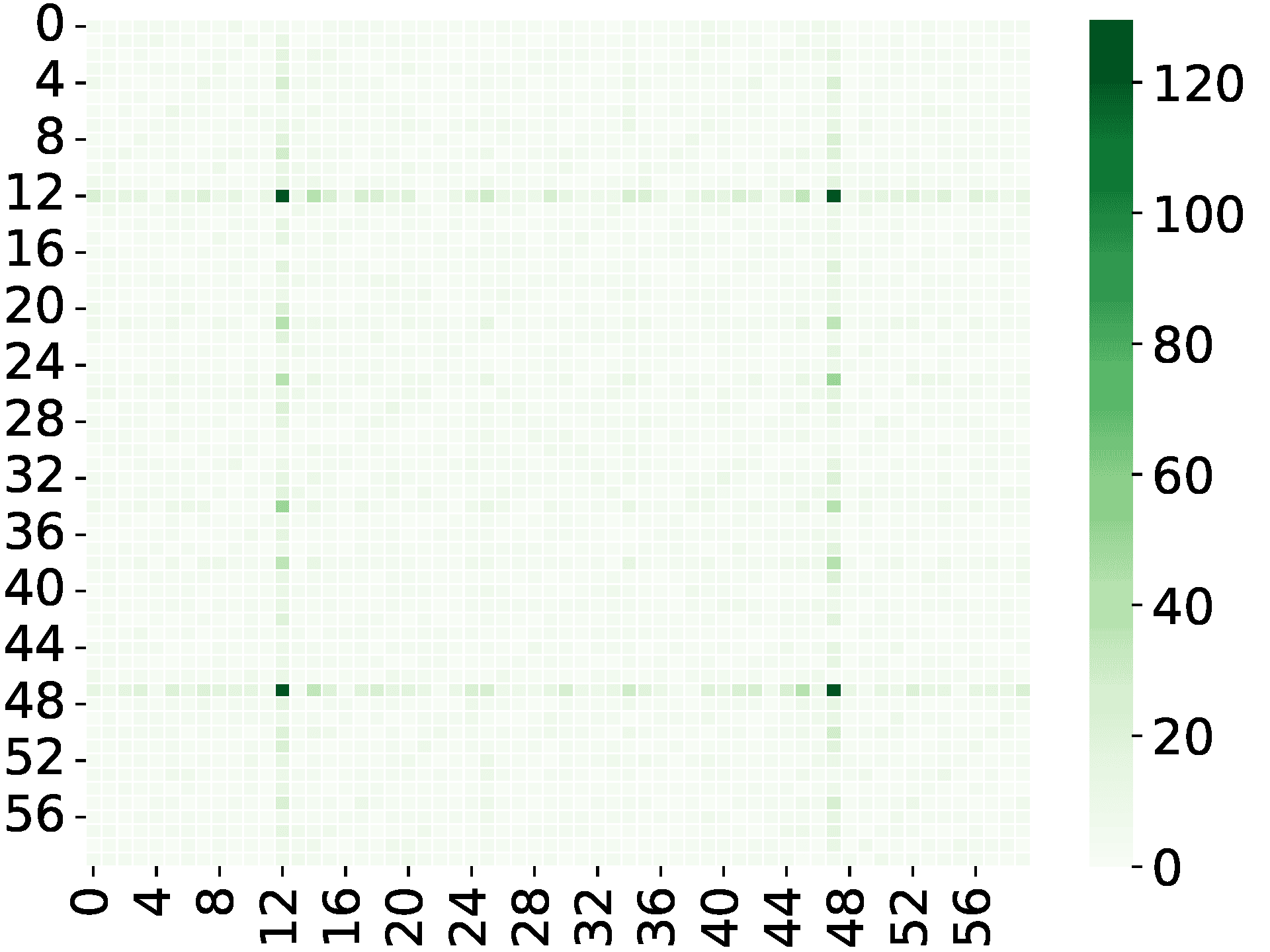}
    \caption{2-dimension cosine shape of the trained $W^{KQ}$ (attention weights) and their Fourier power spectrum. The one-layer transformer with attention heads $m=160$ is trained on $k=3$-sum mod-$p=61$ addition dataset. We even split the whole datasets ($p^k = 61^3$ data points) into training and test datasets. Every row represents a random attention head from the transformer. The left figure shows the final trained attention weights being an apparent 2-dim cosine shape. The right figure shows their 2-dim Fourier power spectrum. The results in these figures are consistent with Figure~\ref{fig:nn_w_k3}.}
    \label{fig:s_k3}
\end{figure}

\begin{figure}
    \centering
\includegraphics[width=0.38\linewidth]{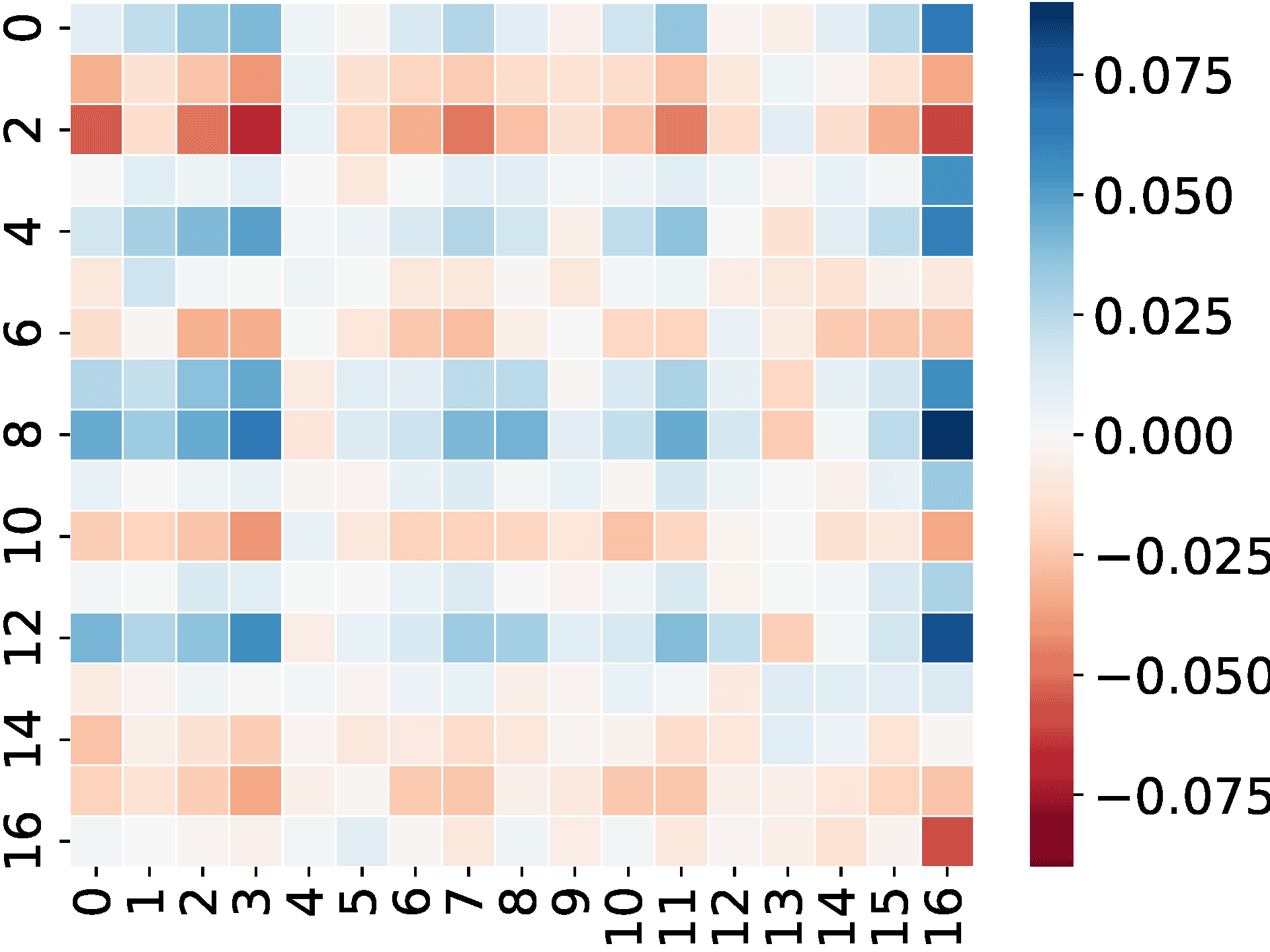}
\includegraphics[width=0.38\linewidth]{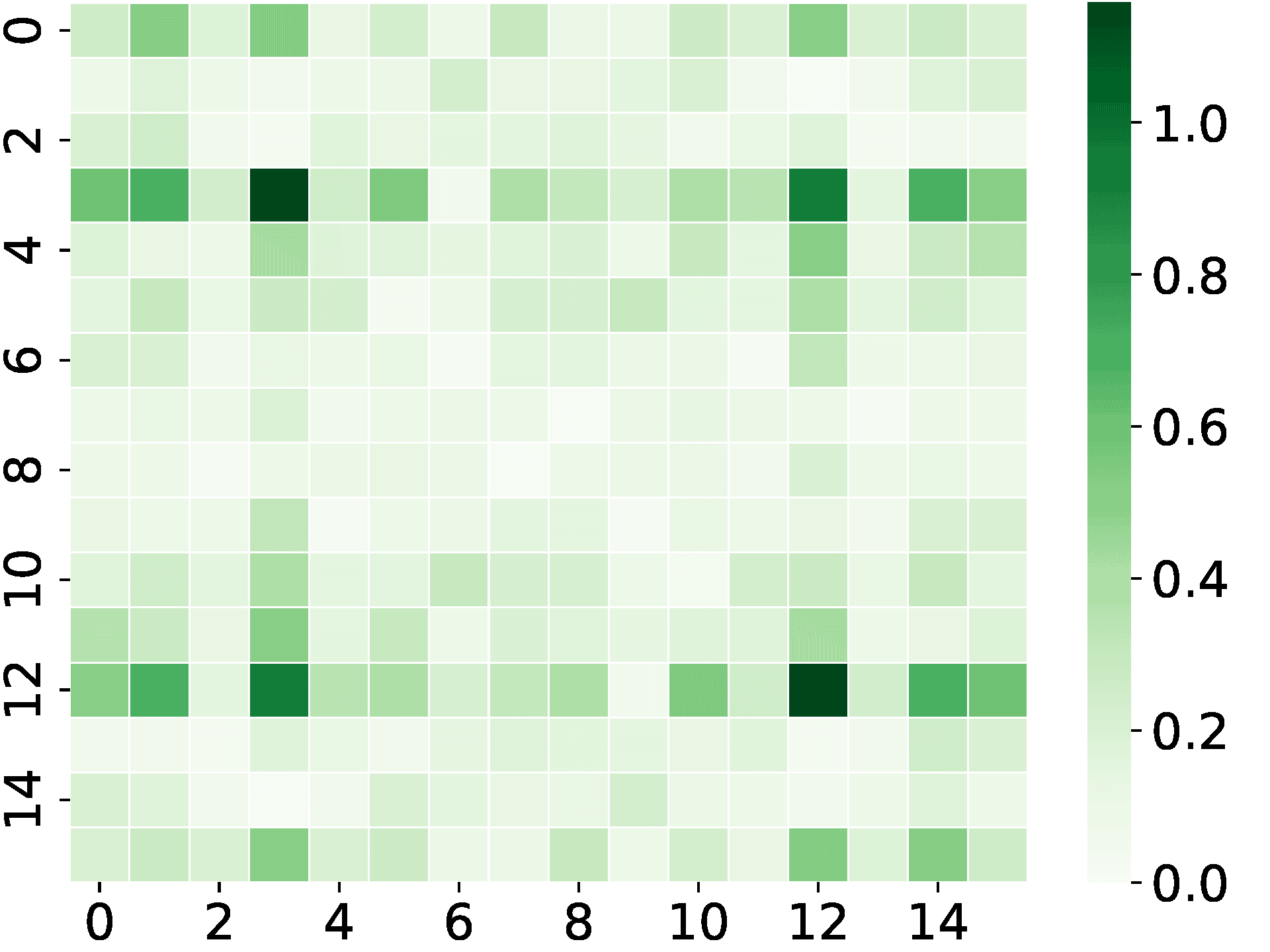}
\includegraphics[width=0.38\linewidth]{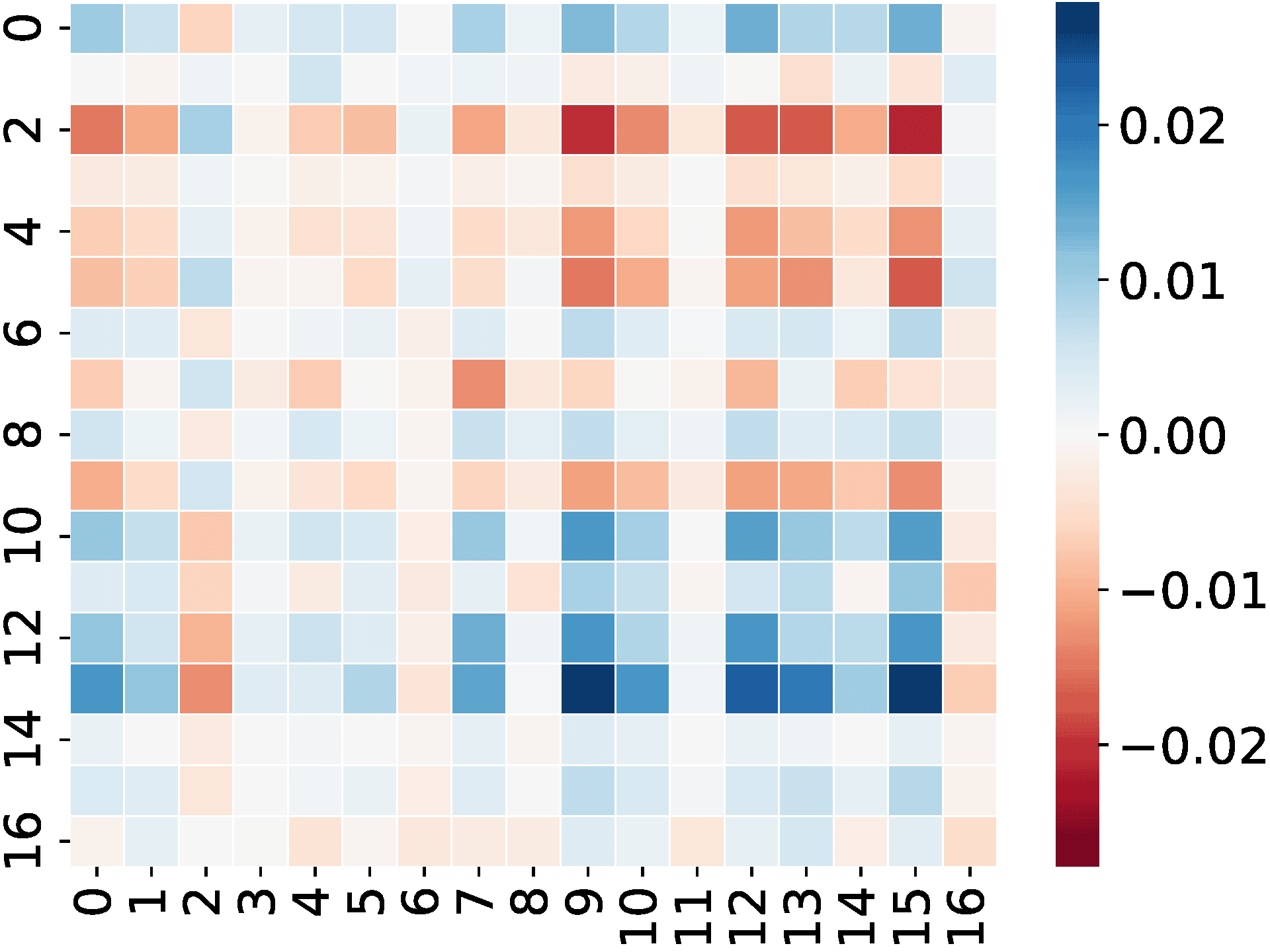}
\includegraphics[width=0.38\linewidth]{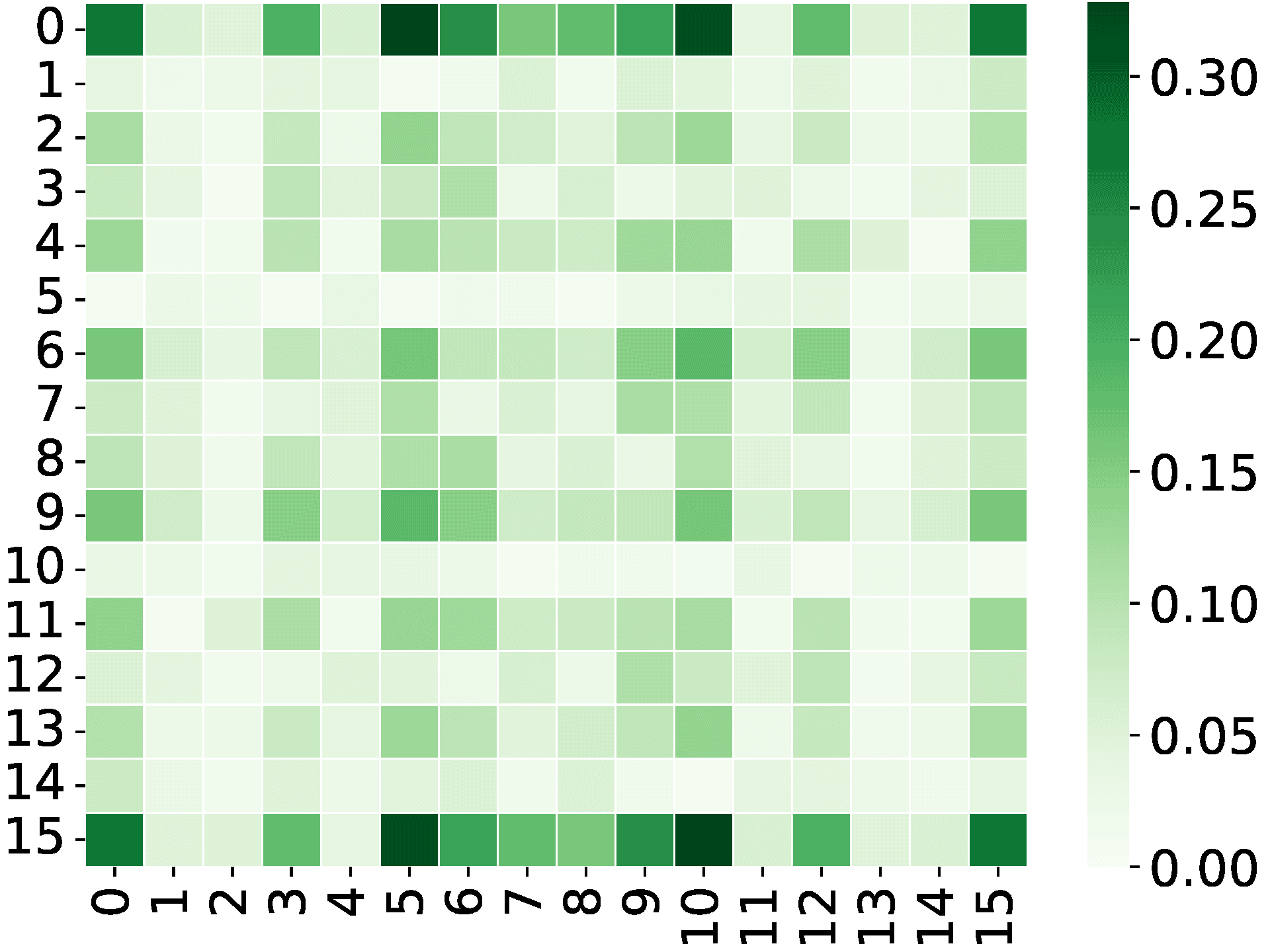}
\includegraphics[width=0.38\linewidth]{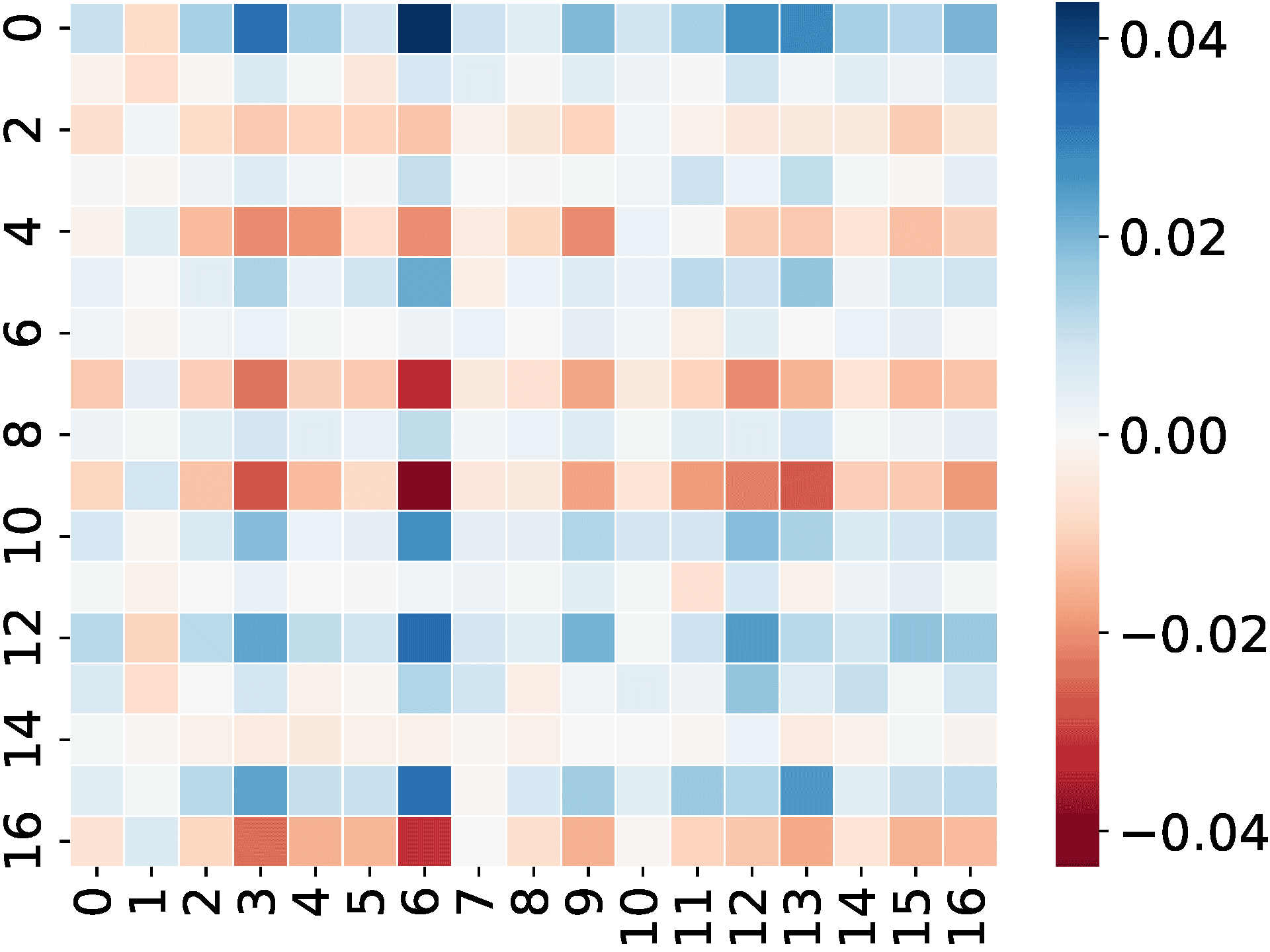}
\includegraphics[width=0.38\linewidth]{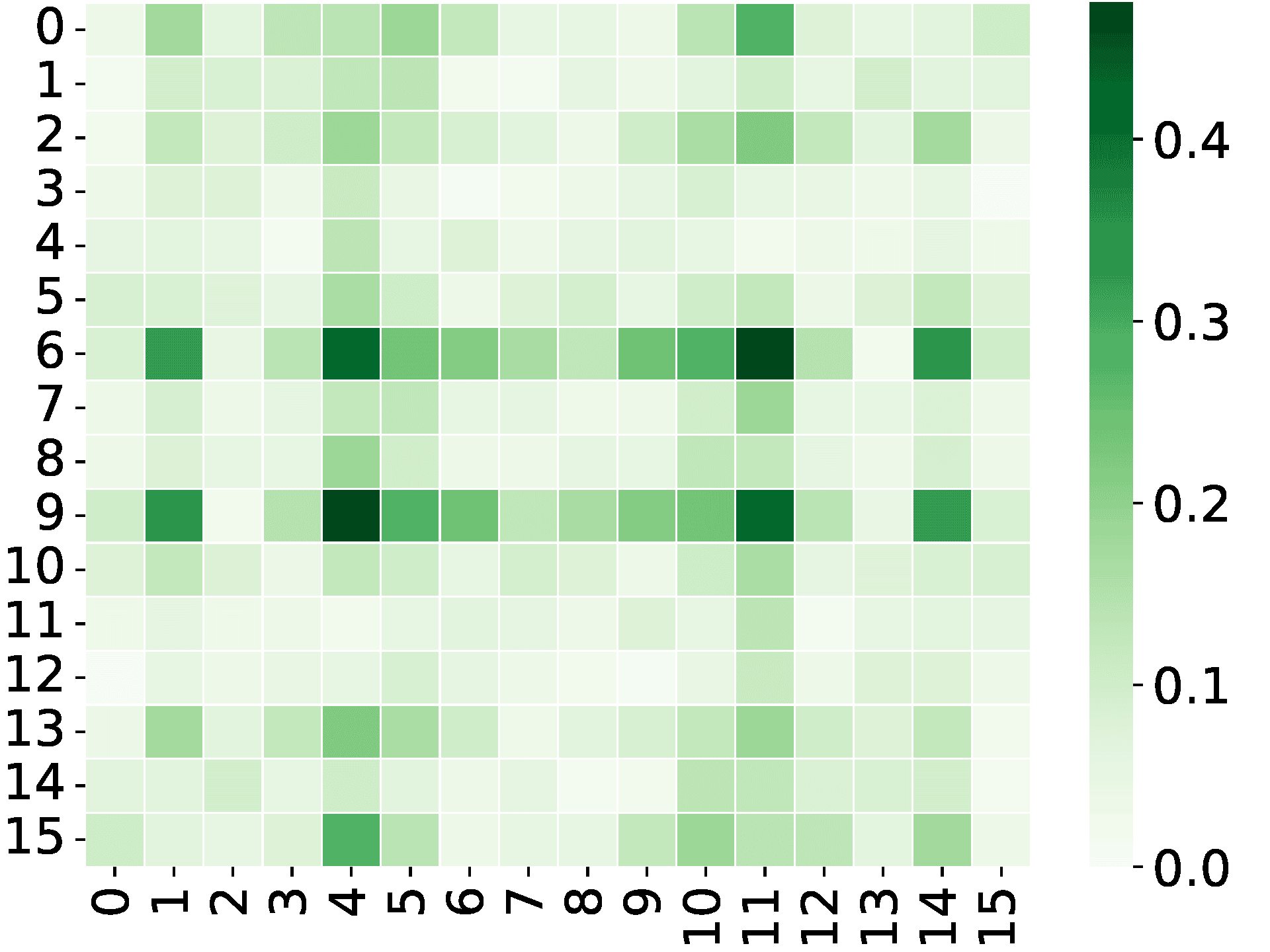}
\includegraphics[width=0.38\linewidth]{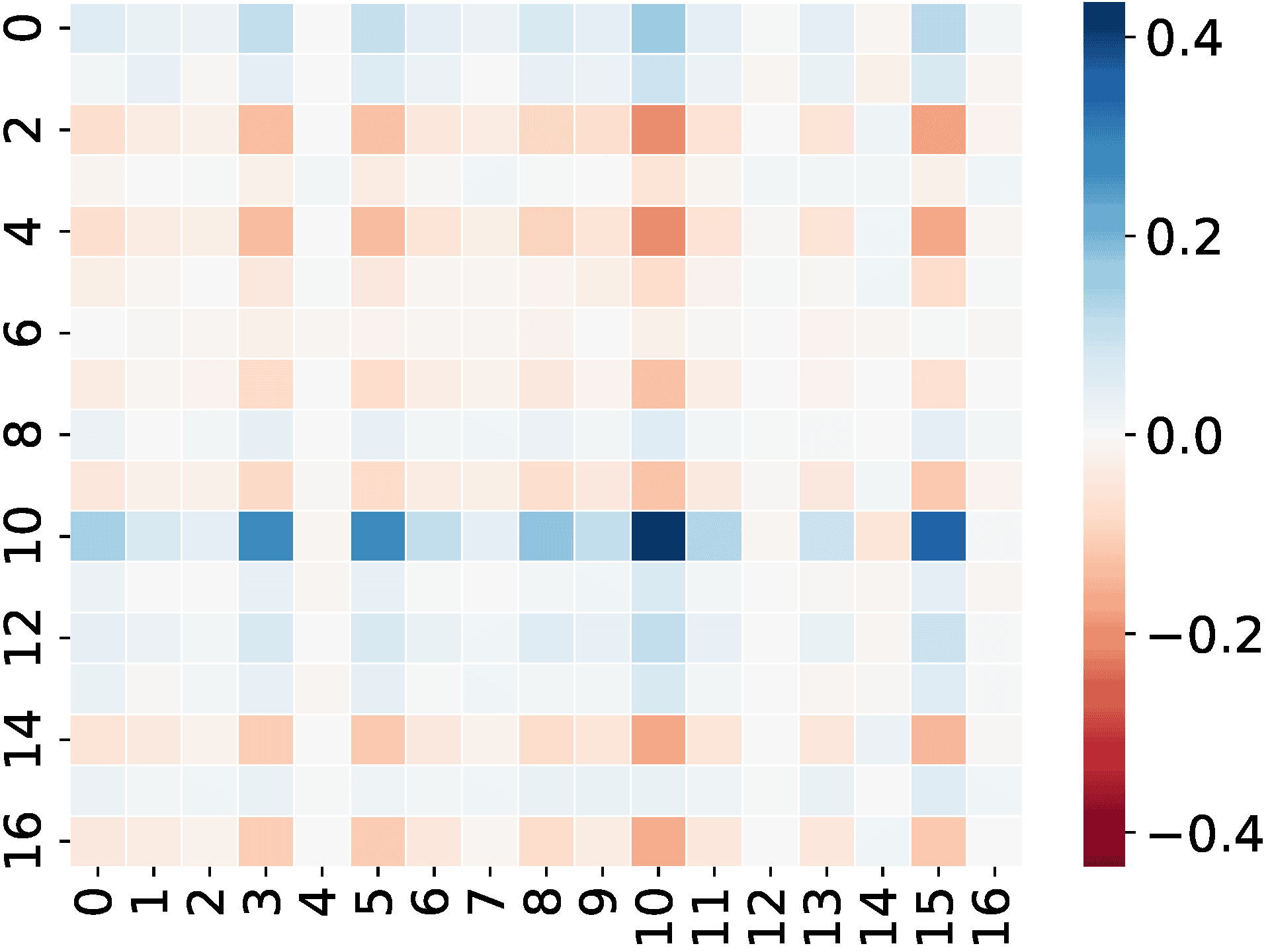}
\includegraphics[width=0.38\linewidth]{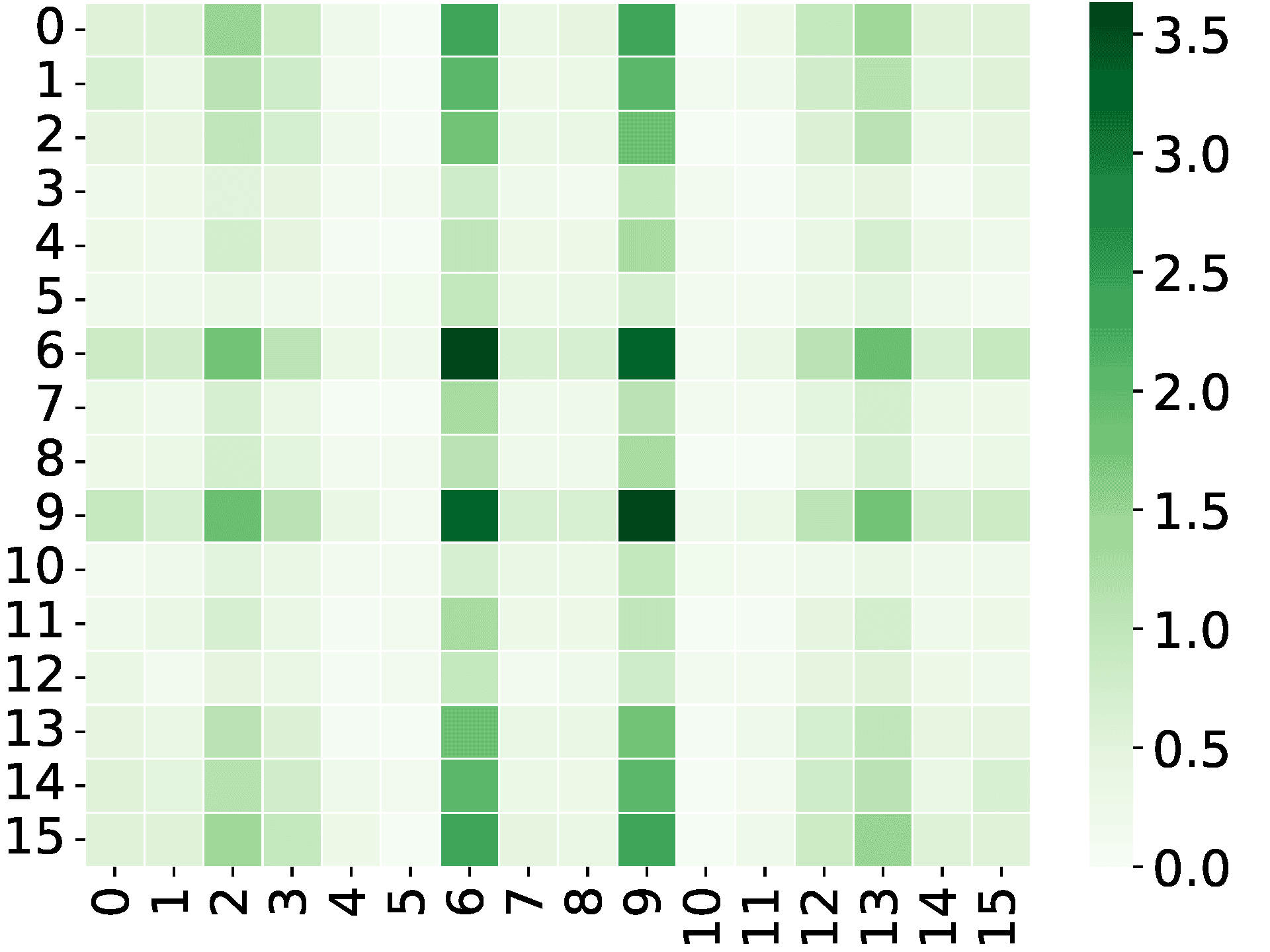}
    \caption{2-dimension cosine shape of the trained $W^{KQ}$ (attention weights) and their Fourier power spectrum. The one-layer transformer with attention heads $m=160$ is trained on $k=5$-sum mod-$p=17$ addition dataset. We even split the whole datasets ($p^k = 17^5$ data points) into training and test datasets. Every row represents a random attention head from the transformer. The left figure shows the final trained attention weights being an apparent 2-dim cosine shape. The right figure shows their 2-dim Fourier power spectrum. The results in these figures are consistent with Figure~\ref{fig:nn_w_k5}. 
    }
    \label{fig:s_k5}
\end{figure}




\end{document}